\documentclass[twoside,11pt]{article}

\usepackage{blindtext}

\usepackage{jmlr2e} 

\usepackage{tikz-cd}
\usepackage{booktabs}
\usepackage{amsmath,amscd,amssymb,amsbsy,amstext,mathtools,algorithm,graphicx,multirow}
\usepackage{algpseudocode}
\algrenewcommand\algorithmicrequire{\textbf{Input:}}
\algrenewcommand\algorithmicensure{\textbf{Output:}}
\usepackage{pdfpages}
\usepackage{enumerate}
\usepackage{mathrsfs}
\usepackage[utf8]{inputenc}
\usepackage{subcaption}

\numberwithin{equation}{section}
\usepackage{amsfonts}
\usepackage[T1]{fontenc}

\usepackage{xr-hyper}
\usepackage{hyperref}

\usepackage{url}

\def\m{\mathcal}
\def\mb{\mathbb}

\def\ms{\mathscr}

\def\wt{\widetilde}
\def\wh{\widehat}

\def\dd{{\rm d}}

\def\lY{\overleftarrow{Y}}

\def\opn{\operatorname}

\newtheorem{assumption}{Assumption}

\begin{document}

\title{A Deep Generative Approach to Stratified Learning}
\author{\name Randy Martinez \email randym@umd.edu \\
       \addr Department of Mathematics\\
       University of Maryland\\
       College Park, MD 20742, USA
       \AND
       \name Rong Tang \email martang@ust.hk \\
       \addr Department of Mathematics \\
       Hong Kong University of Science and Technology\\
       Clear Water Bay, Kowloon, Hong Kong
       \AND
       \name Lizhen Lin \email lizhen01@umd.edu \\
       \addr Department of Mathematics\\
       University of Maryland\\
       College Park, MD 20742, USA}

\editor{}

\maketitle

\begin{abstract}%


While the manifold hypothesis is widely adopted in modern machine learning, complex data is often better modeled as \textit{stratified spaces}---unions of manifolds (strata) of varying dimensions. Stratified learning is challenging due to varying dimensionality, intersection singularities, and lack of efficient models in learning the underlying distributions. We provide a deep generative approach to stratified learning by developing two generative frameworks for learning distributions on stratified spaces. The first is a sieve maximum likelihood approach realized via a dimension-aware mixture of variational autoencoders. The second is a diffusion-based framework that explores the score field structure of a mixture. We establish the convergence rates for learning both the ambient and  intrinsic distributions, which are shown to be dependent on the intrinsic dimensions and smoothness of the underlying strata. Utilizing the geometry of the score field, we also establish consistency for estimating the intrinsic dimension of each stratum and propose an algorithm that consistently estimates both the number of strata and their dimensions. Theoretical results for both frameworks provide fundamental insights into the interplay of the underlying geometry, the ambient noise level, and deep generative models. Extensive simulations and real dataset applications, such as molecular dynamics, demonstrate the effectiveness of our methods.


\end{abstract}


\section{Introduction}

Deep neural networks have achieved remarkable success across a wide range of domains, including computer vision, natural language processing, and recommendation systems \citep{lecun2015deep, krizhevsky2012imagenet,vaswani2017attention, cheng2016wide}. Deep generative models, in particular,  lie at the heart of many of these applications, where neural networks are trained to implicitly learn the underlying distribution of the data from samples alone.  Real-world datasets, such as images, molecular, and medical data  tend to be very high-dimensional, where classical statistical models often suffer from the \textit{curse of dimensionality}. Nevertheless, a growing body of empirical evidence and theory suggests that deep generative models can perform  very  well in such high-dimensional settings and and seem to  be able to circumvent the curse of dimensionality \citep{barron2002universal,pope2021intrinsic,chen2023score,potaptchik2024linear}.  


A widely adopted assumption in modern machine learning is the \textit{manifold hypothesis}, which assumes that high-dimensional data is frequently supported on or near a low-dimensional manifold. This principle influenced research in many areas such as dimensionality reduction, manifold regularization, and representation learning \citep{bengio2013representation, belkin2003laplacian, belkin2006manifold}. Recent works have illustrated that neural networks can adapt to lower-dimensional structures and approximate functions on manifolds, which partially explains their effectiveness in high-dimensional learning problems  \citep{nakada2020adaptive,schmidt2019deep}. Convergence rates of learning distributions with generative models have also been shown to adapt to the geometric properties of the manifold. For instance, \cite{Chae} develop a likelihood-based approach to understanding singular distributions and demonstrate that convergence rates depend on intrinsic dimension and smoothness. Similarly, an increasing body of work has demonstrated the adaptivity of diffusion models to distributions supported on manifolds \citep{TangYang1, azangulov2024convergence, Tang2024ConditionalDM}. These developments suggest that generative models are not only effective in modeling low-dimensional structures embedded in high-dimensional ambient spaces, but are also capable of adapting to the geometry of the data.

Although the manifold assumption has been used extensively for studying distributions on low-dimensional structures, this has often shown to be a restrictive assumption for complex data.  
For example, \cite{robinson2025token} show that token embeddings in large language models violate the manifold hypothesis, and \cite{li2025unraveling} identify  a stratified manifold structure, that is, a union of manifolds, in the embedding space. Many real-world datasets, including natural images \citep{brown2022verifying} and molecular dynamics data \citep{sule2025learning}, exhibit more complicated geometric organization, with multiple components of different dimensions that may intersect. Such spaces form the broad class of \textit{stratified spaces}. While classical manifold learning techniques and generative modeling results on manifolds provide powerful tools for studying stratified spaces,  the regularity conditions typically imposed in the manifold setting, such as smoothness and positive reach, are commonly violated in stratified spaces.  This motivates the central goal of this paper: developing a generative approach to \textit{stratified learning}, which combines a richer set of geometric objects with the adaptive power of generative models to learn distributions supported on or near a stratified space.

Previous work in learning stratified spaces have mostly followed either classical geometric or  statistical approaches. Building upon the principles of the Levina-Bickel MLE \citep{levina2004maximum}, \cite{haro} propose a Poisson mixture model that utilizes Expectation-Maximization to simultaneously estimate local intrinsic dimensions, mixture weights, and densities. From a topological perspective, \cite{bendich2012local} leverage persistent homology to cluster points into strata, providing probabilistic inference theorems for their framework. Similarly, \cite{aamari} introduce a geometric "Slabeling" algorithm to cluster by dimension via data-derived slabs, yielding convergence rates for strata, dimension, and tangent space estimators. \cite{wu2022generalized} use a discriminative deep learning model to perform clustering of a stratified space by preserving geometric structure in the latent space with knowledge of the number of strata. While fundamental, these approaches encounter numerous challenges: they can be sensitive to noise and do not provide a mechanism for learning the underlying distribution of the  data centering or lying on the stratified spaces. 
In addition, the existing approaches  are certainly not generative, i.e. they do not provide a mechanism to synthesize new samples.

Stratified learning via generative modeling introduces novel challenges that elude the standard manifold learning theory: specifically, it requires jointly learning a support consisting of varying intrinsic dimensions and smoothness while approximating hierarchical mixtures. Mixture-of-experts models are commonly used to model (hierarchical) mixture distributions \citep{jacobs1991adaptive, jordan1994hierarchical,ye2020mixtures}, and methods such as partitioning data to resemble charts of a manifold have been employed in autoencoder-based models to learn distributions on manifold structures \citep{TangYang2, alberti2024manifold}. These architectures, however, struggle to naturally cluster data according to distinct geometries in an unsupervised setting. Conversely, diffusion models have illustrated the ability to locally capture the local geometry of the data. The score function at small time scales is approximately normal to the data manifold, and approaches involving Monte Carlo approximations of the normal space or the Fokker-Planck equation have been successfully used for local intrinsic dimension estimation \citep{stanczuk2024diffusion, horvat2024gauge, kamkari2024geometric, leung2025convolutions}. While these diffusion-based estimators naturally extend to disjoint unions of manifolds given sufficient separation, they fundamentally struggle with stratified spaces, where the presence of intersections violates standard local manifold assumptions.

To our knowledge, this is the first attempt to introduce a theoretical generative modeling framework designed to learn distributions and geometric structures on stratified spaces. We first develop a likelihood approach to learning a distribution via a sieve MLE. The key difficulty in this extension from the manifold case is the singularity around intersections, where smoothness and reach conditions that are commonly assumed in the manifold setting are violated. This is addressed by partitioning the space into regions that exhibit positive reach and an intersection portion that can be controlled. We then give a diffusion-based framework and derive rates of convergence of the score approximation and Wasserstein distance between the estimator and the target distribution. Furthermore, we prove consistency of the proposed diffusion-based local intrinsic dimension estimator as well as the number of strata. The theory developed for these two generative frameworks makes it possible to dissect the ability of deep generative models to perform stratified learning and to understand the roles of singularity, geometry, and ambient noise, as well as the distinctive features of the two generative frameworks. For example, they are suitable for different regimes.  The sieve-MLE approach is most appropriate when the noise level is moderate and providing a direct estimator for both the ambient and intrinsic distribution, while likelihood-based methods become unstable in near-singular settings. By contrast, the diffusion-based framework remains well posed even without noise due to its built-in Gaussian smoothing in the forward process. Thus, diffusion models are especially well suited for learning on singular or nearly singular stratified spaces. In addition,  the local geometric information of the score field at small diffusion time can be utilized for learning the structural information of the stratified space.


The paper is organized as follows: 
Section 2 provides the setup, definitions and assumptions that will be used throughout the paper. In Section 3, we develop a likelihood approach  for stratified learning. Section 4, is devoted to a diffusion-based framework  and convergences of the score and target distribution.  Section 5 focus on the consistency of the local intrinsic dimension (LID) and number of strata estimator. In Section 6, we provide extensive numerical examples that illustrate our methods. Lastly, Section 7 contains concluding remarks. For presentation purposes, we provide proofs to all theorems in the supplementary material.

\section{Stratified Spaces and Deep Generative Models} 

In this section, we formalize the geometric and probabilistic frameworks for learning distributions on stratified spaces.

\subsection{Geometric Setup}
We assume that our data is generated from a distribution that centers around  a stratified space $\mathcal{S} \subset \mathbb{R}^D$, which is defined as a union of compact embedded manifolds (without boundary) $M_1,...,M_K \subset \mathbb{R}^D$ of intrinsic dimensions $1\leq d_k<D$. Each component is also called a \textit{stratum} of $\mathcal{S}$. The subset of singularities is denoted by $\m S_{\text{sing}}=\bigcup_{k \neq \ell} (M_k \cap M_\ell)$, and points $x \in S_{\text{sing}}$ are referred to as \textit{singular}, and \textit{regular} otherwise.

For a closed set $C \subset \mb{R}^D$ and $r>0$, the \textit{reach} of $C$ is defined as the supremum of $r$ such that $C \oplus B_r(0_D)$ has unique Euclidean projection onto $C$, where $\oplus$ denotes Minkowski summation. Stratified spaces allow for intersections between strata, so unique projection may not hold in general for any $r>0$. In this case, we say $\operatorname{reach}(\m S)=0$   and the positive-reach assumption  is  violated. We therefore impose the following regularity conditions on each stratum and subset of  singularities of $\mathcal{S}$:

\begin{assumption}[Regularity of Strata] Each stratum $M_k$ is an embedded $\beta_k$-H\"older smooth manifold for some $\beta_k \geq 2$ and with positive reach $r_k>0$ when considered as a standalone manifold. We refer to \cite{gine2021mathematical} for details about H\"older smoothness.
\end{assumption}

\begin{assumption}[Intersections] The intersections between strata are non-degenerate in the sense that for any $k,\ell$, $M_k \cap M_\ell$ is a submanifold of co-dimension at least one of both $M_k$ and $M_\ell$, and the intersections are transversal (non-tangential).
\end{assumption}

The co-dimension condition implies that the singular intersection set $\m S_{\text{sing}}$ has strictly lower intrinsic dimension, and hence has measure zero. The transversality assumption ensures that the strata cross at a positive angle, preventing them from aligning tangentially. This implies that the distance between strata diverges at a linear rate when moving away from a singular region, which is important in our asymptotic analysis. Compactness of the strata also provides a margin between non-intersecting strata, preventing disjoint strata from becoming arbitrarily close. 

\subsection{Probabilistic Setup}

We first define a distribution $Q_*$, the \emph{intrinsic distribution}, that's assumed to be supported on a stratified space $\m S$ and can be written as a mixture of distributions supported on the individual manifolds:
$$Q_*=\sum_{k=1}^K \omega_kQ_k$$
where $Q_k$ is a probability measure on $M_k$ that is absolutely continuous with respect to the volume measure $\text{vol}_{M_k}$, $\omega_k \in (0,1)$, and $\sum_k \omega_k=1$.

Real-world data may be noisy and   generally need not lie exactly on a stratified space. We suppose that we have i.i.d. samples $X_1,...,X_n \in\mathbb{R}^D$ from a distribution $P_*$, which we model as a convolution of a distribution on $\mathcal{S}$ and Gaussian noise:
$$X=Y+\epsilon$$ where $Y \sim Q_*$ with distribution $Q_*$  is the intrinsic distribution that's supported exactly on the stratified space $\mathcal S$ and $\epsilon \sim \mathcal{N}(0_D,\sigma_*^2I_D)$ with $Y$ and $\epsilon$ being independent. In some cases, we may assume that $\sigma_* =\sigma_n\to 0$ as $n\to \infty$, which we specify separately in both the likelihood and diffusion approaches.
We note that when $\sigma_* >0$, $P_*$ is absolutely continuous with respect to the Lebesgue measure, and hence has a density $p_*$. This also allows us to view $P_*$ as:
$$P_*=\Big(\sum_{k=1}^K \omega_kQ_k\Big)*\mathcal{N}(0_D,\sigma_*^2I_D)=\sum_{k=1}^K \omega_k(Q_k * \mathcal{N}(0_D,\sigma_*^2I_D)).$$
In particular, we see that the convolution with Gaussian noise happens on each component of the mixture.

\begin{remark}
    We assume one noise level $\sigma_*$ across all strata. While one could consider stratum-dependent noise levels $\sigma_k$ to model heterogeneous measurement errors (e.g., different physical constraints), this can cause issues at intersections and complicates the analysis in diffusion models, which inherently use isotropic Gaussian noise. Furthermore,
     a single noise level ensures that the local variance is attributed to intrinsic geometry as opposed to variable noise magnitudes. 
\end{remark}

\begin{assumption}[Density Smoothness]
    For each $k$, the density $q_k$ of $Q_k$ with respect to $\operatorname{vol}_{M_k}$ is $\alpha_k$-H\"older smooth for some $\alpha_k \in [0,\beta_k-1]$ and is bounded from above and below on $M_k$.
\end{assumption}

\begin{assumption}\label{as:weights}
    There exists a constant $\omega_{\min}>0$ such that $\omega_k \geq \omega_{\min}$ for all $k$, and $\omega_k$ remains constant (independent of samples size).
\end{assumption}

\begin{remark} The conditions in Assumption~\ref{as:weights} imply that for a sample size of $n$, the average sample size $n_k$ for stratum $M_k$ is $\mb{E}[n_k] =\omega_kn\asymp n$
for all $k$ and that each component can locally be recovered asymptotically. These assumptions align with standard frameworks such as in \cite{aamari} and are necessary for identifying strata against noise. While one may consider the cases where $\omega_k$ can be arbitrarily small  or the number of strata $K$ grows as $n \to \infty$ (non-parametric mixtures), such settings require additional detection analysis and are left for future work.
\end{remark}

\subsection{Neural Network Sieve}

In both settings, we will use a class of ReLU neural networks to approximate the target function: in the likelihood regime, the network parametrizes the generator (or pushforward map), while in the diffusion regime, it parametrizes the score vector field.\\

\noindent \textbf{Definition (Neural network class):} For each $m\geq 1$, $\operatorname{ReLU}:\mb{R}^m \to \mb{R}^m$ denotes the coordinate-wise map $\operatorname{ReLU}(x_1,...,x_m)=(\max\{0,x_1\},...,\max\{0,x_m\})$. For each $i=0,...,L$, define the affine map $T_i(x)=A^{(i)}x+b^{(i)}$, where $A^{(i)} \in  \mathbb{R}^{ W_{i+1}\times W_i}$  are the weight matrices, $ b^{(i)} \in \mathbb{R}^{W_{i+1}}$ are the bias vectors. The standard feedforward neural network with depth $L$ is defined as
$$f_{\theta}= T_L \circ \operatorname{ReLU}\circ T_{L-1} \circ\operatorname{ReLU} \circ \cdots \circ \operatorname{ReLU} \circ T_0$$
where $\theta=\{A^{(0)},\ldots,\,A^{(L)}, b^{(0)}, \ldots, b^{(L)}\}$ and $\circ$ denotes function composition. A class of neural networks with depth $L$, 
width vector $W=(W_0,\,W_1,\ldots,\,W_{L+1})$,
sparsity $R$, parameter bound $B$, and function norm bound $V$ is defined as $$\Phi(L, W, R, B, V)=  \{f_{\theta}(\cdot): \sum_{i=0}^L(\|A^{(i)}\|_0+\|b^{(i)}\|_0) \leq R;\, \max _i\|A^{(i)}\|_{\infty} \vee\|b^{(i)}\|_{\infty} \leq B;\, \|f\|_{\infty}\leq V\,\big\}.$$

\section{ Sieve MLE: a likelihood-based approach}

In the setting of generative modeling, variational autoencoders are trained to maximize a variational lower bound of the log likelihood function. This motivates the exploration of distributions supported on stratified spaces through a likelihood lens. A central difficulty in discussing maximum likelihood estimators in the context of manifolds and stratified spaces is that distributions supported on lower-dimensional structures are singular with respect to Lebesgue measure; that is, they are concentrated on sets of Lebesgue measure zero.  Convolving that distribution with an isotropic Gaussian creates a distribution with a positive ambient density. Thus, the addition of noise is a theoretical necessity for a likelihood-based approach. However, this introduces an inherent trade-off on the noise variance. If the noise level is too large, the intrinsic geometry of the data is obscured, making recovery of the distribution on the stratified space difficult, while too little noise makes the likelihood more sharply concentrated, leading to singularity which can cause inconsistency in the ambient distribution estimator.

\citet{Chae} analyze a sieve MLE for distributions $P_*$ with smooth and lower-dimensional manifold structure and derive Wasserstein-1 rates for recovering the intrinsic distribution $Q_*$ from noisy samples. We build on this framework by considering a \textit{mixture-of-experts sieve} of neural networks and show that $Q_*$ can be approximated by a latent pushforward of a function in this sieve. We define a class $\m P$ of distributions $P_{f,\sigma}$ inspired by the mixture-of-experts model and show that the density $p_*$ of $P_*$ can be approximated in Hellinger distance by a sieve MLE $\hat{p}=p_{\hat{f},\hat{\sigma}} \in \m P$. With a careful partition of the ambient space and suitable noise level, it can be shown that the pushforward distribution $Q_{\hat{f}}$ converges to  the intrinsic distribution $Q_*$ in Wasserstein distance.

\subsection{Hellinger Convergence of a Sieve MLE for $p_*$}

We parametrize the singular distribution $Q_*$ on $\mathcal{S}$ via a latent pushforward model. Let $Z$ be a random variable on a latent space $\mathcal{Z}$ with a simple distribution $P_Z$ (e.g., Gaussian or uniform), and $f:\mathcal{Z}\to \m S \subset\mathbb{R}^D$ be a measurable map parametrized by a deep neural network. This induces a distribution on $\mathcal{S}$ through the pushforward $Q_f=f_\#P_Z$. We model the data-generating process as
$$X=f(Z)+\epsilon$$ where $\epsilon \sim \mathcal{N}(0,\sigma_*^2I_D)$, and $\epsilon$ and $Z$ are independent. Consider an interval $[\sigma_{\min},\sigma_{\max}]$ of possible noise levels for the data. For our purposes we will consider the regime where $\sigma_*=\sigma_{*,n} \to 0$ as $n \to \infty$, which is common in the manifold learning literature. Thus, $\sigma_{\max}$ can remain constant, but $\sigma_{\min}$ can be treated as a sequence tending to zero. Let $P_{f,\sigma}$ be the distribution of $X=f(Z)+\epsilon$, and define the class of distributions of this form as
$$\mathcal{P}=\{P_{f,\sigma} \mid f \in \mathcal{F}, \sigma \in [\sigma_{\min},\sigma_{\max}]\}.$$

\noindent Note that if $\sigma >0$, then $P_{f,\sigma}$ has a Lebesgue density given by
$$p_{f,\sigma}(x) = \int \phi_\sigma(x-f(z))dP_Z(z).$$

\noindent For a class of functions $\mathcal{F}$, an interval of possible noise levels $[\sigma_{\min},\sigma_{\max}]$, and any sequence $\eta_n \to 0$, one can define a Sieve MLE as $\hat{p}=p_{\hat{f},\hat{\sigma}}$, where $(\hat{f},\hat{\sigma}) \in \mathcal{F} \times [\sigma_{\min},\sigma_{\max}]$ are such that $$\sum_{i=1}^n \log p_{\hat{f},\hat{\sigma}}(X_i) \geq \sup_{(f,\sigma) \in \mathcal{F} \times [\sigma_{\min}, \sigma_{\max}]} \sum_{i=1}^n \log p_{f,\sigma}(X_i) - \eta_n.$$

\noindent The benefit of a sieve MLE is that it allows us to discuss finite sample performance of the estimator at any sample size $n$ as well as asymptotically. While establishing optimization guarantees of such an estimator in the realm of deep learning is a challenging task, in this paper we focus purely on the theoretical statistical convergence.

\paragraph{Generator Construction.} It is important to note that a manifold is not globally parameterizable in general, so one must work at the level of local charts and glue these local functions together via a partition of unity. A distribution on a stratified space has a natural hierarchical structure: we first select a stratum, then a chart within that stratum, and finally samples points from a local function on that chart.

The target distribution $Q_*$ can be realized by a pushforward measure of a generating function $f:\mathcal{Z} \to \mathcal{S} \subset \mathbb{R}^D$ as follows. Suppose $\m S = \bigcup_{k=1}^K M_k$ and $Q_*=\sum_{k=1}^K \omega_kQ_k$, where $\dim M_k=d_k$ and $Q_k$ is a distribution supported on $M_k$. Let $\mathcal{Z}=(0,1) \times \tilde{\mathcal{Z}}$, where $\tilde{\mathcal{Z}}$ is a $d_{\max}$-dimensional latent space (with $d_{\max}=\max_k d_k$), say $(0,1)^{d_{\max}}$, equipped with a uniform distribution. We then represent a latent sample as $(w,z)$, where $w \sim \text{Uniform}(0,1)$ is a routing variable selecting the stratum and chart, and $z \in \tilde{\m Z}$ represents its intrinsic coordinates.

To construct the local maps, suppose we have an $\alpha_k$-H\"older smooth density $q_k$ with respect to the volume measure on a compact $\beta_k$-H\"older smooth manifold $M_k$ with $\alpha_k \in [0,\beta_k-1]$. One can construct an $(\alpha_k+1)$-H\"older smooth transport map $f:\m Z_k \to U_k$, where $\m Z_k$ is $d_k$-dimensional unit cube equipped with the uniform distribution and $U_k$ is a chart on $M_k$ (see Lemma 10 in \cite{Chae} and the following remark, and Lemma 12.50 in \cite{villani2008optimal}).

To extend this to stratified space, recall that $Q_*=\sum_{k=1}^K \omega_kQ_k$. Let $I_1=(0,\omega_1)$ and $I_k=[\sum_{\ell =1}^{k-1} \omega_\ell,\sum_{\ell=1}^k\omega_\ell)$ for $k=2,...,K$ be disjoint intervals of lengths $\omega_1,...,\omega_K$, respectively, partitioning $(0,1)$. Now for each fixed $k$, following \cite{Chae}, suppose we cover $M_k$ with $J_k$ charts $\{U_j\}_{j=1}^{J_k}$ and define a partition of unity $\{\tau_j(\cdot)\}_{j=1}^{J_k}$ subordinate to that cover. Let $Q_{k,j}$ be the normalized measure on $U_j$ with corresponding Hausdorff density $q_{k,j}$ and define $\tilde{q}_{k,j}(y)=c_j\tau_j(y)q_{k,j}(y)$, where $c_j=[\int \tau_j(y)dQ_{k,j}(y)]^{-1}$. Then the Hausdorff density $q_k$ of $Q_k$ can be written as a mixture $q_k(y)=\sum_{j=1}^{J_k} \pi_{k,j}\tilde{q}_{k,j}(y)$, where $\pi_{k,j}=Q_{k,j}(U_j)/c_j$ for $j=1,...,J_k$  and each component $\tilde{q}_{k,j}(y)$ is representable as a pushforward measure $f_{k,j\#}P_{\wt Z}$. Hence, further partition each of the intervals as
$I_{1,1}=(0,\omega_1\pi_{1,1})$, and for all other $(k,j)$, define
$$I_{k,j}=\left[\sum_{\ell=1}^{k-1} \omega_
\ell+\omega_k\sum_{i=1}^{j-1} \pi_{k,i}\,,\sum_{\ell=1}^{k-1} \omega_\ell +\omega_k\sum_{i=1}^j \pi_{k,i}\right).$$
Let $h_{k,j}$ be the indicator function for $I_{k,j}$.
For $(w,z) \in \mathcal{Z}=(0,1) \times \tilde{\m Z}$, define $f^*$ as 
\begin{equation}\label{true_func} f^*(w,z)=\sum_{k=1}^K \sum_{j=1}^{J_k}h_{k,j}(w)f_{k,j}(z_{1:d_k}),
\end{equation}
where $z_{1:d_k}$ denotes projection onto the first $d_k$ variables of $\tilde{\m Z}$. Then one can show that by construction, we have that $Q_*=Q_{f^*}=f^*_{\#}P_Z$ where $P_Z={\rm Uniform}(0,1)\otimes P_{\wt Z}$. 

\paragraph{Mixture-of-Experts Sieve.} To define the sieve, we fix an integer $K\geq 1$ and a bound $V>0$. For each $k=1,...,K$, we consider the class $$\mathcal{F}^{(k)}=\Phi(L^{(k)},W^{(k)},R^{(k)},1,V^{(k)})$$ consisting of neural networks of depth $L^{(k)}$, widths $W^{(k)}$, sparsity $R^{(k)}$, weights bounded by 1, and output bound $\|f\|_\infty \leq V^{(k)}$. Also, let $\mathcal{H}=\Phi(\widetilde{L},\widetilde{W},\widetilde{R},1,1)$ consist of ReLU networks bounded by 1. For a positive integer $K\geq1$ and $K$-sequence $\mathbf{J}=(J_1,...,J_K)$ with $J_k \geq 1$ for all $k=1,...,K$, define the \textit{Mixture-of-Experts sieve} to be
$$\m G^{K,\mathbf{J}}=\left\{f(w,z)=\sum_{k=1}^K\sum_{j=1}^{J_k}h_{k,j}(w)f_{k,j}(z) :h_{k,j} \in \m H, f_{k,j}\in \m F^{(k)}, \|f\|_\infty \leq V \right\}.$$
The architecture for each component of the sieve may be different for theoretical purposes, but in practice one may take them all to be the same, i.e. $\m F^{(k)} = \mathcal{F}$ for all $k=1,...,K$. Our results will depend on the sample size $n$, and hence one can consider a sequence $\{\mathcal{G}^{K,\mathbf{J}}_n\}_{n=1}^\infty$.

\begin{remark}
    Mixture-of-Experts models traditionally use a softmax routing mechanism as a means of creating a distribution over the experts, so one could consider a class $\m H$ of ReLU neural networks with softmax output over all $h_{k,j}$'s. However, ReLU routing has been considered to be a promising alternative for sparse MoE architectures \citep{wang2024remoe}. Also, note that our sieve relaxes the probability simplex constraint on the gating functions (i.e., $\sum h_{k,j}=1)$ that is generally imposed by Mixture-of-Experts models. This relaxation allows us to approximate the true indicator routing functions that partition the domain with more tractability. One could also consider a (equivalent) hierarchical mixture architecture by partitioning $(0,1)$ into the mixture weights for the strata followed by another partition of $(0,1)$ for charts. 
\end{remark}

The following lemma shows that any function $f^*$ of the form (\ref{true_func}) can be approximated arbitrarily well by functions in $\m G^{K,\mathbf{J}}$.

\begin{lemma}\label{lm:approx} Suppose $f^*(w,z)=\sum_{k=1}^K\sum_{j=1}^{J_k} h_{k,j}^*(w)f^*_{k,j}(z)$ is of the form (\ref{true_func}) over $\m Z=(0,1)^{d_{\max}+1}$, where $f^*_{k,j}$ are $(\alpha_k+1)$-H\"{o}lder smooth with $\|f^*_{k,j}\|_\infty \leq V$ and $h^*_{k,j}$ are indicator functions for all $k=1,...,K$ and $j=1,...,J_k$. Then, for every $\delta > 0$, there exists an approximation $f(w,z)=\sum_{k=1}^K\sum_{j=1}^{J_k} h_{k,j}(w)f_{k,j}(z)$ consisting of:
\begin{enumerate}
    \item Routing networks $h_{k,j} \in \m H$ with $\widetilde{L}=O(\log \delta^{-1})$, $\|\widetilde{W}\|_\infty=O(1)$, and $\widetilde{R} = O(\log \delta^{-1})$,
    \item Expert networks $f_{k,j} \in \mathcal{F}^{(k)}$ with $L^{(k)}=O(\log \delta^{-1})$, $\|W^{(k)}\|_\infty = O(\delta^{-d_k/(\alpha_k+1)})$, $R^{(k)}=O(\delta^{-d_k/(\alpha_k+1)}\log \delta^{-1})$, and $V^{(k)} = V \vee 1$,
\end{enumerate} such that $\|f-f^*\|_2 \leq \delta$.
\end{lemma}
\begin{remark}
    Unlike the $L^\infty$ approximation required in \cite{Chae}, Lemma~\ref{lm:approx} relaxes this to an $L^2$ approximation of $f^*$. This is a necessity since indicator functions jump from 0 to 1, making such approximation by (continuous) ReLU functions impossible in $L^\infty$. This does not cause issues with their proof of Theorem 3 for Hellinger distance since it only requires an $L^2$-approximation error to give an upper bound on the KL divergence between the true density $p_*$ and the model density $p_{f,\sigma_*}$.
\end{remark}

Lemma~\ref{lm:approx} shows that our mixture-of-experts sieve is rich enough to approximate the true generative map, effectively controlling the approximation error. The stochastic error of the estimator can be analyzed via the bracketing entropy. Let $N_{[]}(\delta,\mathcal{P},d_H)$ denote the minimum number of $\delta$-brackets in the Hellinger metric $d_H$ to cover the density class $\m P$. Based on the structure of a function $f \in \mathcal{G}^{K,\mathbf{J}}$, the complexity of $f$ is governed by the complexities of the routing class $\mathcal{H}$ and the expert classes $\mathcal{F}^{(k)}$. Together, these yield an upper bound on the bracketing entropy via the standard covering numbers $N(\epsilon,\m F,\|\cdot\|_\infty)$ with respect to the supremum norm:

\begin{lemma}\label{lm:covering}
    Let $\m G^{K,\mathbf{J}}$ be the Mixture-of-Experts sieve with $\|f\|_\infty \leq V$, and let $M=\sum_{k=1}^K J_k$. Then there exist constants $c,c',c''$ depending only on $D,V,M,\sigma_{\max}$ and $\delta_*=\delta_*(D)$ such that for every $\delta \in (0,\delta_*]$,
    \begin{align*}
        \log N_{[]}(\delta, \mathcal{P},d_H) &\leq M\log N\left (c\sigma_{\min}^{D+3}\delta^4, \mathcal{H},\|\cdot\|_\infty\right)\\
        &+ \sum_{k=1}^K J_k\log N\left(c'\sigma_{\min}^{D+3}\delta^4, \mathcal{F}^{(k)}, \|\cdot\|_\infty\right) + \log \Big(\frac{c''}{\sigma_{\min}^{D+2}\delta^4}\Big).
    \end{align*}
\end{lemma}

\noindent The approximation and entropy established from Lemma~\ref{lm:approx} and Lemma~\ref{lm:covering} lead to the following theorem on the convergence in Hellinger distance of the sieve MLE $\hat{p}$ of $p_*$. The convergence rate of $d_H(\hat{p},p_*)$ is effectively determined by the worst-rate stratum:

\begin{theorem}\label{th:hellinger}
    Let $n\geq 1$ and $\gamma >0$. Suppose Assumptions 1-4 hold and assume $\sigma_*\in [\sigma_{\min},\sigma_{\max}]$ with $\sigma_{\min}\leq 1$. Define $j_* = \arg \max_k \frac{d_k}{\alpha_k+1}$ and
    $$\delta_{app}:=\Big(\frac{\sigma_*^2}{n}\Big)^{(\alpha_{j_*}+1)/(2(\alpha_{j_*}+1)+d_{j_*})}.$$ Then there exists $\m H$ and $\{\m F^{k}\}_{k=1}^K$ that defines the Mixture-of-Experts sieve $\mathcal{G}^{K,\mathbf{J}}$  so that
    \begin{enumerate}
        \item  $\log N(\delta,\mathcal{H},\|\cdot\|_\infty)\leq \widetilde{R}(A+1 \vee \log \delta^{-1})$ and $\log N(\delta,\mathcal{F}^{(k)},\|\cdot\|_\infty)\leq R^{(k)}(A+1 \vee \log \delta^{-1})$ for $k=1,...,K$, where  $\tilde{R} = O(\log \delta^{-1}_{\text{app}})$ is the routing sparsity,  $R^{(k)}=O(\delta_{\text{app}}^{-d_k/(\alpha_k+1)}\log \delta_{\text{app}}^{-1})$  is the expert sparsity, and  $A = O((\log n)^2)$.
        \item There exists $f\in \mathcal{G}^{K,\mathbf{J}}$ such that $\|f-f^*\|_2\leq \delta_{app}$.
        \item  Let $\delta_*$ be given as in Lemma~\ref{lm:covering}. Then there exist some constants $C_1,C_2,C$ so that the sieve MLE $\hat{p}$ satisfies
        $$P_*(d_H(\hat{p},p_*) > \epsilon_n^*) \leq 5e^{-C_1n\epsilon_n^{*2}} + \frac{C_2}{n},$$ where  $$\epsilon_n^*=C\left( \sqrt{\frac{(\widetilde{R}+\sum_{k}J_kR^{(k)})(A+\log(n/\sigma_{\min}))}{n}} \vee \frac{\delta_{\text{app}}}{\sigma_*}\right).$$ 
        provided that $\eta_n \leq \epsilon_n^{*2}/6$ and $\epsilon_n^* \leq \sqrt{2}\delta_*$.
    \end{enumerate}
     
\end{theorem}

In the case that $\sigma_* = n^{-\gamma}$ and $\sigma_{\min} = n^{-\zeta}$ for some $\zeta \geq \gamma \geq 0$, then the choice of $\delta_{\text{app}}$ optimizes $\epsilon_n^*$ in Theorem~\ref{th:hellinger}, which implies that
$$\epsilon_n^* =Cn^{-\frac{(\alpha_{j_*}+1)-d_{j_*}\gamma}{2(\alpha_{j_*}+1) + d_{j_*}}}(\log n)^{3/2}.$$
Observe that if $\alpha_{j_*}+1-d_{j_*}\gamma \leq 0$, then distributional convergence is not guaranteed. This is in particular the case when $\gamma$ is relatively large, so that the noise becomes too small quickly.   In that regime the target distribution becomes nearly singular, which can lead to extremely large likelihood values and makes MLE-based methods unstable. Hence, we require that the noise does not decrease too fast for Theorem~\ref{th:hellinger} and the following section.

\subsection{Convergence of Sieve MLE for $Q_*$}

The sieve MLE framework is easily adapted for  estimating the intrinsic distribution $Q_*$ on $\m S$. A sieve MLE $P_{\hat{f},\hat{\sigma}}$ for $P_*$ naturally defines an estimator $Q_{\hat{f}}$ for $Q_*$. One of the standard methods of recovering an intrinsic distribution on a manifold from an ambient noisy one involves utilizing its reach to project onto the manifold as a form of deconvolution. The reach of the stratified space is 0 if two or more strata intersect, so a tubular neighborhood where projection is unique cannot exist. To circumvent this, we perform excision to partition $\m S$ into regular and singular regions. The regular regions in particular will have positive reach, making standard projection-based results such as Theorem 7 in \cite{Chae} applicable. Formally, for a  $\rho>0$ (which we chose to be vanishing with $\sigma_*$) and for each $k=1,...,K$, let
$$A_k=\left\{x \in M_k \big|\  \text{dist}\Big(x,\bigcup_{\ell \neq k} M_\ell\Big) >2\rho\right\}, \quad A=\bigcup_{k=1}^K A_k,\quad\text{ and }\quad  I_{\mathrm{sing}} = \mathcal{S} \setminus \bigcup_k A_k.$$
Then each $A_k$ is now a closed subset with positive reach, and $I$ contains points in intersections and points $2\rho$-close to two or more manifolds that contributes an error on the order of $\rho$. This is formalized in the following theorem:



\begin{theorem}\label{th:wasserstein}
    Let $\mathcal{S}=\bigcup_{k=1}^K M_k$ be the closure of $f^*(\mathcal{Z})$. Suppose that $||f^*||_\infty \leq V$ for some constant $V$. Assume that $\mathcal{S}$ does not have an interior point in $\mathbb{R}^D$ and $\text{reach}(M_k) = r_k$ for some constants $r_k>0$ when considered as standalone manifold. 
     Further assume that there exist constants $c>0$ and  $\rho_1>0$ such that when $\rho\leq \rho_1$, $Q_*(A)>c$. If $d_H(p_{f,\sigma},p_*) \leq \epsilon \leq 1$ and $\|f\|_\infty \leq V$, then
    $$W_1(Q_{f},Q_*)\leq C(\epsilon+\sigma_*\sqrt{\log \epsilon^{-1}})$$ for some constant $C$ depending on $c,D,K,V$.
\end{theorem}

\begin{remark}
    Although the bound in Theorem~\ref{th:wasserstein} is of the same order as the one in Theorem 7 in \cite{Chae}, the constant is potentially much larger depending on the separation of the strata. While we do not assume that the stratified space has positive reach, our proof handles the singularity by removing a neighborhood around it, and this excision introduces additional dependence on the true noise level $\sigma_*$.
\end{remark}

We again emphasize the delicacy of the true noise level $\sigma_*$. If $\sigma_*$ is too large, the geometric information from the stratified structure can become obscured, which degrades the statistical recovery of $Q_*$. Conversely, if $\sigma_*$ is too small, the density $p_*$ becomes sharply concentrated and the Hellinger distance $d_H(\hat{p},p_*)$ may not converge, as discussed following Theorem~\ref{th:hellinger}. In the small-noise regime, one can apply noise injection (data perturbation) as in \cite{Chae} to regularize the density. The process involves fixing a perturbation level $\tilde{\sigma}>0$ and constructing perturbed observations $\tilde{X}_i = X_i + \tilde{\epsilon}_i$, where $\tilde{\epsilon}_i \sim \mathcal{N}(0,\tilde{\sigma}^2I_D)$. The perturbed data remain i.i.d. with respect to the regularized distribution $\tilde{P}_*=P_{f_*,\tilde{\sigma}_*}$, where $\tilde{\sigma}_*^2=\sigma_*^2 + \tilde{\sigma}^2$. By applying Theorem~\ref{th:hellinger} and Theorem~\ref{th:wasserstein} to $\tilde{X}_1,...,\tilde{X}_n$, we obtain a sieve MLE $\hat{P}_{\text{per}}=P_{\hat{f}_{\text{per}},\hat{\sigma}_{\text{per}}}$ for the ambient distribution and hence a sieve MLE $\hat{Q}_{\text{per}} = \hat{Q}_{\hat{f}_{\text{per}}}$ for $Q_*$.

Furthermore, suppose that $\sigma_*=n^{-\gamma}$ and $\sigma_{\min}=n^{-\zeta}$. Let $\tilde{\sigma}=n^{-(\alpha_{j_*}+1)/2(\alpha_{j_*}+1+d_{j_*})}$. By applying Theorem~\ref{th:wasserstein} on the perturbed distribution and adapting the argument in the proof of Theorem 9 in \cite{Chae}, we obtain 
$$P_*\left(W_1(\hat{Q}_{\text{per}},Q_*) >C_3\left( \epsilon_n^*+\sigma_*\sqrt{-\log \epsilon_n^*}\right)\right) \leq 5e^{-C_1n\epsilon_n^{*2}} + \frac{C_2}{n},$$
where $$\epsilon_n^* \asymp \begin{cases}
    n^{-\frac{\alpha_{j_*}+1-d_{j_*}\gamma}{2(\alpha_{j_*}+1)+d_{j_*}}}(\log n)^{3/2}& \text{if }\gamma<(\alpha_{j_*}+1)/(2(\alpha_{j_*}+1+d_{j_*})),\\
    n^{-\frac{\alpha_{j_*}+1}{2(\alpha_{j_*}+1+d_{j_*})}}(\log n)^{3/2}& \text{otherwise.}
\end{cases}$$
Data perturbation thus stabilizes the distribution and improves the rate of convergence when $\gamma \geq (\alpha_{j_*}+1)/(2(\alpha_{j_*}+1+d_{j_*}))$, and thus provides convergence guarantees in the small noise regime where the unperturbed estimator would fail.

\section{Diffusion Model for Stratified Learning}

Section 3 demonstrated the necessity of ambient noise to stabilize the likelihood of a singular distribution. That is, adding noise to the data helps prevent inconsistency when the noise-level is too low. However, selecting an optimal noise level can be practically rigid. Diffusion models naturally avoid this rigidity by constructing a time-dependent forward process to incrementally inject noise into the data rather than applying a single perturbation. In addition, diffusion models also exhibit a backward process that allows for easy sampling from a standard Gaussian. 

One can write the forward and backward processes in terms of stochastic differential equations \cite{song2020score}, and it has been shown in previous works (\cite{oko2023diffusion, TangYang1}) that diffusion models are minimax-optimal and can adapt to manifold structures. The main objectives in the setting of stratified learning are learning the structure (i.e., the number of manifolds and their dimensions) and the underlying distributions on the stratified space, both of which are not addressed in other contexts. To this end, we show that the score of distributions supported near a stratified space can be approximated by a score network and derive the convergences of both score estimation as well as the underlying distribution estimator. Note that unlike in the existing literature where the data is assumed to lie exactly on a submanifold embedded in a high-dimensional ambient space, we consider there a general setting where the data might admit some noise. In generalizing the theory of diffusion models  to the setting of stratified learning, it requires overcoming some substantial technical challenges and yield interesting results as discussed in Remark \ref{rem11}.  

\subsection{Forward-Backward Diffusion Models}
We focus on a common class of diffusion models: the forward-backward diffusion models. This class consists of a forward process used to corrupt data and construct a score function, and a dual backward process generate new samples. For the forward process, we define the following Ornstein-Uhlenbeck process:
$$d\overrightarrow{X}_t = -\beta_t\overrightarrow{X}_tdt + \sqrt{2\beta_t}dW_t, \quad \overrightarrow{X}_0 \sim p_*$$
where $\beta_t$ is a drift coefficient and $\{W_t: t \geq0\}$ denotes standard Brownian motion in $\mb{R}^D$. Solutions to this SDE are of the form $X_t=m_tX_0 + \int_0^t \frac{m_t}{m_s}\sqrt{2\beta_s}dW_s$, yielding the conditional marginals $p_t(\cdot|X_0)=\mathcal{N}(m_tX_0,\sigma_t^2I_D)$, where $m_t=\exp\Big(-\int_0^t \beta_s\,ds\Big)$ and $\sigma_t^2=1-m_t^2$. The forward process converges to a standard Gaussian distribution under the total variation metric exponentially fast as $t \to \infty$. In practice, one truncates the process at a finite time horizon $T$.

Under mild regularity conditions \citep{song2020score, haussmann1986time} the backward process is modeled by an SDE that simulates the forward process in reverse time:
$$d\overleftarrow{Y}_t= \Big[\beta_{T-t}\overleftarrow{Y}_t + 2\beta_{T-t}\nabla \log p_{T-t}(\overleftarrow{Y}_t)\Big]dt + \sqrt{2\beta_{T-t}}dW_t$$
The distribution of $\overleftarrow{Y}_t$ is $p_{T-t}$, so we can sample from $p_*$ by initializing the backward process with $\overleftarrow{Y}_0 \sim \mathcal{N}(0,I_D)$ as a proxy for $p_T$ and numerically solving the reverse SDE.  The reverse process requires access to the time-dependent score vector field $\nabla \log p_t(x)$, which is generally intractable. Standard score-matching \citep{hyvarinen2005estimation} is used to approximate the score $\nabla \log p_t(x)$ with a neural network $\hat{S}_\theta(x)$ by minimizing the objective
$$\mathcal{L}(\theta)=\mathbb{E}_{x \sim p}[||S_\theta(x)-\nabla \log p(x)||^2],$$
where $S_\theta:\mathbb{R}^D \to \mathbb{R}^D$ is a deep neural network parametrized by $\theta$.

To model the time-dependent score $\nabla \log p_t$, we must optimize this objective over the entire trajectory. To perform a tractable optimization problem that does not require access to the intractable $\nabla \log p_t(x)$, one typically minimizes the following denoising score-matching objective \citep{vincent2011connection}, which is equivalent up to a constant:
$$\mathcal{L}(\theta)=\frac{1}{n} \sum_{i=1}^n \int_0^T \mathbb{E}_{x_t \sim p_t(\cdot|x_i)} [||S_\theta(x_t,t) - \nabla \log p_t(x_t|x_i)||^2]\lambda(t)\,dt,$$
where $S_\theta:\mb{R}^D\times [0,T]\to \mb{R}^D$ is a neural network parametrized by $\theta$, and $\lambda(t)$ is a positive weighting function.

A fundamental challenge in the stratified learning setting is that the distribution $Q_*$ is supported on a singular subset of Lebesgue measure zero. If $\sigma_*>0$, then the ambient distribution $P_*$ has a Lebesgue density, which avoids this issue. On the other hand, if $\sigma_*=0$, then the forward process naturally regularizes the distribution through convolution with an isotropic Gaussian, making the marginals $p_t$ smooth and non-singular for any $t>0$. Note that the underlying geometry of the stratified space is most pronounced for small $t$ as the convolved measure remains tightly concentrated. However, this tight concentration leads to an optimization issue: the score $\nabla \log p_t(x)$ explodes as $t\to 0^+$ \citep{pidstrigach2022score, de2022convergence}, leading to numerical instability. Consequently, our analysis in Theorem~\ref{th:score} relies on rigorously controlling the approximation error from $t=T$ into the singular regime as $t \to 0$.

\subsection{Score Approximation and $W_1$ Convergence}
 Recall that we consider a distribution $Q_*=\sum_{k=1}^K \omega_k Q_k$ supported on a stratified space $\m S$. Assuming the data-generating process admits ambient noise $\epsilon \sim \mathcal{N}(0,\sigma_*^2I_D)$ for some $\sigma_* >0$, the true ambient distribution takes the form $$P_*=\Big(\sum_{k=1}^K \omega_kQ_k\Big) *\mathcal{N}(0,\sigma_*^2I_D)=\sum_{k=1}^K \omega_k(Q_k * \mathcal{N}(0,\sigma_*^2I_D)),$$
which admits a density with respect to the Lebesgue measure given by
$p_*(x)=\sum_{k=1}^K \omega_kp_k(x)$, where $$p_k(x)=\int \phi_{\sigma_*}(x-y)dQ_k(y).$$
Under the forward process, the transition kernel from an observation $x_0 \sim P_*$ to state $x$ at time $t$ is $p_t(x|x_0)=\mathcal{N}(x; m_t x_0, \sigma^2_t I_D)$. Hence, the marginal distribution at time $t$ is given by $$p_t(x) = \int p_t(x|x_0)p_*(x_0)dx_0.$$
Substituting our mixture form $p_*(x_0)=\sum_{k=1}^K \omega_kp_k(x_0)$ yields
$$p_t(x)= \int p_t(x|x_0)\sum_{k=1}^K \omega_kp_k(x_0)dx_0=\sum_{k=1}^K \omega_k\int p_t(x|x_0)p_k(x_0)dx_0=\sum_{k=1}^K \omega_kp_{k,t}(x). 
$$

Thus, the marginal density at each time $t$ preserves the mixture structure. This also induces the following decomposition of the score function:
\begin{equation}\label{eq:score_decomp}
    \nabla \log p_t(x)= \frac{\nabla p_t(x)}{p_t(x)}=\sum_{k=1}^K \omega_k\frac{\nabla p_{k,t}(x)}{p_t(x)}=\sum_{k=1}^K \frac{\omega_kp_{k,t}(x)}{p_t(x)}\nabla \log p_{k,t}(x).
\end{equation}

Equation~\ref{eq:score_decomp} reveals that the global score function is a convex combination of the component score vector fields $\nabla \log p_{k,t}(x)$, whose weights are given by the Bayesian posterior probability that a diffused point $x$ at time $t$ originates from stratum $k$. This yields a theoretical advantage to analyzing the score approximation error by independently controlling the error on each stratum $M_k$.

\paragraph{Theorem Setup} 
We perform denoising score-matching to obtain a neural network estimator $\hat{S}$. For simplicity, we assume the weighting function $\lambda(t)\equiv 1$. The estimator is obtained by solving:
\begin{align}
    \widehat S={\arg\min}_{S\in \m S_{NN}} \frac{1}{n}\sum_{i=1}^n \int_0^{T} \mb E_{x_t\sim p_t(\cdot\,|\,X_i)} \big[\|S(x_t,t) -\nabla\log p_t(x_t\,|\,X_i)\|^2\big] \,\dd t,\label{eqn:score_loss_fb}
\end{align}
To generate samples, we use the learned score network $\hat{S}$ in the backward SDE:
\begin{align}
     \dd \lY_t = \big[&\beta_{T - t} \lY_t \ + \ 2\beta_{T - t}\widehat S(\lY_t,T - t)\big]\,\dd t +\, \sqrt{2\beta_{T-t}}\,\dd W_t,\ \ \lY_0\sim \m N(0, I_D).\label{eqn:back_est}
\end{align}
We define a truncated estimator $\wh p $ for $P_*$ as the distribution of $\lY_{T-\tau}\cdot\bold{1}(\|\lY_{T-\tau}\|_{\infty}\leq L)$, where $\tau$ is the early-stopping time and $L$ is a sufficiently large constant. Then consider $\m S_{NN}$ realized by piecewise ReLU neural network
\begin{equation*}
\begin{aligned}
      &\m S_{NN}=\big\{S(x,t)=\sum_{l=0}^{L-1} S_l(x,t)\cdot \textbf{1}\big(t_{l}\leq t < t_{l+1}\big)\,\big|\, S_l\in \Phi(L_l, W_l, R_l, B_l, V_l),\ l\in[L] \big\},
\end{aligned}
\end{equation*}
where $\tau=t_0< t_1< \cdots<t_{L}=T$, $L=\lfloor \log_2\frac{T}{\tau} \rfloor +1$, and $\frac{t_{l+1}}{t_{l}}=2$ for any $0\leq l\leq L-2$. We use $\wt O(a_n)$ to be of order $a_n$ up to logarithmic term, i.e., $b_n=\wt O(a_n)$ iff there exist constants $c,C$ so that $b_n\leq C\, a_n (\log n)^{c}$ for $n\geq 2$.

\begin{theorem}[Forward-backward diffusion]\label{th:score}
    Suppose Assumptions 1,3,4 are satisfied, and the drift coefficient $\beta_t$ is  infinitely differentiable and uniformly bounded from above and below in $t$. There exists a constant $C_1$ so that if $T=C_1\log n$ and 
    \begin{equation*}
\tau=c_1\log n\cdot\left\{
\begin{array}{cc}
n^{-c},& \sigma_*>  n^{-\min_{k} \frac{\alpha_k+1}{2\alpha_k+d_k}}\sqrt{\log n}\\
n^{-\min_{k} \frac{2\alpha_k+2}{2\alpha_k+d_k}},& \sigma_*\leq n^{-\min_{k} \frac{\alpha_k+1}{2\alpha_k+d_k}}\sqrt{\log n},\\
\end{array}
\right.
 \end{equation*}
    where $c$ is any constant with $c\geq 1$ and $c_1$ is a sufficiently large positive constant. Then there exist suitable choices of $\{(L_l, W_l, R_l, B_l, V_l)\}_{l\in [L]}$ so that
  \begin{enumerate}
    \item For any $l\in [L]$, 
    \begin{equation*}
    \begin{aligned}
             &\mb{E}_{X_{1:n}\sim P_{*}^{\otimes n}}\bigg[  \int_{t_l}^{t_{l+1}} \int_{\mb R^{D}}   \|\wh S(x,t)-\nabla \log p_t(x)\|^2 p_t(x)\,\dd x\dd t\bigg]\\
             &\qquad=\wt O\bigg(\sum_{k=1}^K  \min\Big(\frac{ (\sigma_*+\sqrt{1\wedge t_l})^{-\max\{d_k,2\}}}{n}, n^{-\frac{2\alpha_k}{2\alpha_k+d_k}}+\frac{n^{-\frac{2\beta_k}{2\alpha_k+d_k}}}{\sigma_*^2+t_l}\Big)\bigg).
    \end{aligned}
    \end{equation*}
    \item   It holds that
    \begin{equation*}
       \mb{E}_{X_{1:n}\sim P_{*}^{\otimes n}}[W_1(\wh p,\, P_*)]=\wt O\bigg(\frac{1}{\sqrt{n}}+\sum_{k=1}^K 
       \min\Big(\frac{ \sigma_*^{-\frac{d_k}{2}+1}}{\sqrt{n}}, n^{-\frac{\alpha_k+1}{2\alpha_k+d_k}}\Big)\bigg).
    \end{equation*}
\end{enumerate}

\end{theorem}
    In  the case $\sigma_*>  n^{-\min_{k} \frac{\alpha_k+1}{2\alpha_k+d_k}}\sqrt{\log n}$, the early stopping time $\tau$ is introduced mainly for technical convenience, and it suffices that $\tau$ decays polynomially in $n$.
 
\begin{remark}
\label{rem11}
    The main technical challenges arise from three sources. First, unlike \cite{TangYang1,oko2023diffusion}, we allow Gaussian noise with level $\sigma_* \in [0,\infty)$, so the distribution need not lie exactly on a submanifold. An interesting consequence of Theorem~\ref{th:score} is that $\sigma_*$ actually has a positive effect on the Wasserstein bound for distribution learning. In particular, when $\sigma_*$ is of constant order, we achieve the parametric root-$n$ rate $\widetilde O(1/\sqrt{n})$. This improvement stems from two facts:  
    (1) The Gaussian noise has a smoothing effect on the intrinsic distribution supported on the stratified space, so the noised distribution has an infinitely smooth density with respect to Lebesgue measure when $\sigma_*$ is of constant order, making the problem intrinsically easier.  
    (2) The structure of the diffusion model estimator is well suited to handle such noise. Convolution with Gaussian noise resembles the forward process that injects Gaussian noise; on the reverse side, this can be viewed as an automatic early stopping in the backward diffusion process, which stabilizes the score approximation.
    
    Second, the stratified-space structure itself poses substantial challenges. In particular, when $\sigma_* = 0$, the data support is no longer a smooth submanifold with positive reach, a standard assumption in the literature on submanifold estimation or distribution estimation on submanifolds~\citep{TangYang1,tang2023minimax,10.1214/18-AOS1685}. Our key insight is that the evaluation metric is the Wasserstein distance, which is relatively insensitive to exact support recovery \citep{tang2023minimax}. Therefore, even if the stratified space cannot be accurately estimated near its intersections, this is less consequential when the goal is distribution recovery under Wasserstein distance. To address the stratified-space setting, we carefully exploit the additive structure of the score function in \eqref{eq:score_decomp} and construct neural network approximations to control the resulting approximation error. Lastly, we allow different components of the stratified space to have different manifold smoothness $\beta_k$ and density smoothness $\alpha_k$. This requires us to carefully choose the resolution and grid size in our approximation theory so as to adapt to the varying smoothness levels.
 
\end{remark}

\begin{remark}\label{rem12}

Both the diffusion-model and sieve-MLE approaches fit naturally into our stratified-space framework, but they address different aspects of the learning problem. On the one hand, diffusion-based learning is particularly well suited to the singular regime arising when $\sigma_*=0$ or when $\sigma_*$ decays rapidly. Because score-based diffusion models are trained through denoising with Gaussian perturbations, the method remains well posed even for noiseless observations, and we can obtain consistency results in $W_1$ distance without imposing a lower bound on the decay rate of $\sigma_*$ or artificially injecting noise into the data as in the likelihood based method. Moreover, the learned score field $\wh S(x,t)$ contains local geometric information; as discussed in Section~\ref{sec:estdim}, its behavior as $t \to 0$ can be leveraged to construct estimators of intrinsic dimension. On the other hand, the sieve MLE procedure is explicitly likelihood-based for an ambient-density model and is therefore less suitable when the target distribution is near singular or no ambient noise is present. However, when the data-generating mechanism includes Gaussian noise, our sieve model class, consisting of mixtures of generative models convolved with Gaussian noise, admits a natural deconvolution interpretation. In particular, fitting $P_{f,\sigma}$ yields not only an estimator of the observed ambient distribution $P_0$, but also an immediate estimator $Q_{\wh f}$ of the underlying noiseless distribution $Q_*$. This contrasts with diffusion models, which by default target the observed distribution rather than its deconvolved counterpart. Finally, the explicit mixture structure in the sieve model provides additional interpretability by enabling stratum-wise generation.

\end{remark}

\section{Estimating Number of Strata and Dimensions}\label{sec:estdim}

The theoretical results in Sections 3 and 4 establish that the likelihood and diffusion models implicitly adapt to the structure of a stratified space with the convergence rates ultimately depending on the intrinsic dimensions and smoothness of strata. However, in many practical settings, the underlying topological structure, such as the number of strata and the dimension of the each stratum is unavailable. In this section, we now address the inverse problem: given a learned score network, can we recover the number of strata and their intrinsic dimensions?

As observed previously, the score function of a distribution supported on a lower-dimensional submanifold explodes as $t \to 0^+$, which requires early stopping to prevent instability when training. Nevertheless, given a point $x$, the score varies minimally along tangent directions while its magnitude is dominant in normal directions. Recent works have leveraged this behavior of diffusion models to estimate local intrinsic dimension of a manifold via the Fokker-Planck equation, gauges, and Monte Carlo methods \citep{stanczuk2024diffusion, kamkari2024geometric, horvat2024gauge, leung2025convolutions}. More specifically, \cite{stanczuk2024diffusion} show that the score vector of the marginal density $p_t$ for points sufficiently close to the data manifold aligns with the normal space as $t\to 0^+$. They then propose a local intrinsic dimension estimator by approximating the normal space at a given data point through sampling score vectors of locally diffused points at small time scales. We extend this method to stratified spaces by showing that, at small time scales, the score is approximately normal to a single stratum around regular points. Near singularities where multiple strata intersect, the score becomes a combination of competing normal directions that is ultimately dominated by the normal space corresponding to the lowest-dimensional participating stratum in the limit.

\subsection{Geometry of the Score}
Equation~\ref{eq:score_decomp} illustrated that the score function for the marginal distribution at time $t$ can be written as a convex combination of the component scores for each manifold:
$$\nabla \log p_t(x) = \sum_{k=1}^K \frac{\omega_k p_{k,t}(x)}{p_t(x)} \nabla \log p_{k,t}(x),$$
where $\frac{\omega_k p_{k,t}(x)}{p_t(x)}$ are the posterior mixture weights. For regular points $x \in M_k \setminus \bigcup_{\ell \neq k} M_\ell$ outside of a small neighborhood of the intersection region, the posterior weight $\frac{\omega_k p_{k,t}(x)}{p_t(x)}$ will tend to 1 as $t \to 0$ as the relative density contributions from other strata become exponentially negligible. However, for points near a singularity, the score is a convex combination of the individual scores of the intersecting strata, with each score component being approximately normal to its respective stratum. To formalize this geometrically, we apply Laplace's method to analyze the interaction of the marginal densities $p_{k,t}$ and their gradients $\nabla p_{k,t}$ as $t \to 0^+$. As we exhibit in the following theorem, the local density contribution from each stratum scales differently based on its intrinsic dimension $d_k$. The relative volume scales proportionally to $\sigma^{d_k-d_\ell}$, so the contributions from higher-dimensional strata vanish as $t \to 0^+$. Thus, the limiting behavior of the score near a singularity is dominated by the normal space of the lowest-dimensional nearby stratum.

\begin{theorem}\label{th:normalalign}
    Suppose that the distribution $P_*$ is supported on a compact stratified space $\m S = \bigcup_{k=1}^K M_k \subset \mathbb{R}^D$ and let $p_t$ be the marginal density at time $t>0$. Let $$L(x)=\{k\in \{1,...,K\}: \operatorname{dist}(x,M_k) = \operatorname{dist}(x,\m S)\}.$$
    Suppose that Assumptions 1, 3, and 4 hold. Fix $x \notin \m S$ sufficiently close to $\m S$ so that, for every $k \in L(x)$, $x$ lies in a tubular neighborhood of $M_k$ making the projection map $\pi_k$ onto $M_k$ uniquely defined. Let $k_* \in L(x)$ satisfy $d_{k_*}<d_k$ for all $k \in L(x)\setminus \{k_*\}$. Then 
    $$\frac{\nabla \log p_t(x)}{\|\nabla \log p_t(x)\|}\to \frac{\pi_{k_*}(x)-x}{\|\pi_{k_*}(x)-x\|}.$$
\end{theorem}

\begin{remark}
    Theorem~\ref{th:normalalign} assumes that the distribution is supported on a stratified space. When noise is small and tends to zero (as in our assumption), this will still be sufficiently close to $S$. More explicitly, at any $t>0$, the conditional marginal densities are Gaussian with variance $\sigma_t^2 + m_t^2\sigma_*^2$, which goes to 0 if both $t\to 0$ and $\sigma_* \to 0$. We show empirically that the dimension estimation is robust to small noise in Section 6. Furthermore, the unique minimal dimension stratum assumption can be lifted, resulting in the asymptotic direction being governed by the lowest-dimensional strata. This leads to a convex combination of normal vectors along with an $o(1)$ remainder.
\end{remark}

\subsection{Local Intrinsic Dimension (LID) Estimation}
Theorem~\ref{th:normalalign} characterizes the limiting behavior of the score function for points sufficiently close to $\m S$. By simulating the forward process starting at a data point $x_0$ over a sufficiently small time interval, we can sample vectors that roughly span the normal space of the corresponding stratum (mixture if near a singularity). Utilizing this geometric property, we adapt the local intrinsic dimension estimation framework in \citep{stanczuk2024diffusion} to construct estimators for the number of strata and their intrinsic dimension. Given a sample of size $n$, we can compute a local intrinsic dimension estimate for each data point and aggregate them into a histogram. The number of modes of this histogram provides an estimator $\hat{K}$ for true number of strata $K$ with the locations of the modes yielding corresponding intrinsic dimension estimates $\hat{d}_1,...,\hat{d}_{\hat{K}}$.

In practice, the theoretical limit $t \to 0^+$ is inaccessible given a finite sample. The score-matching objective must be restricted to the temporal window $[\tau,T]$ for some early stopping time $\tau$ to maintain stability during training, which places a lower bound on the noise level $\sigma_\tau>0$. Near the singular region of $\m S$, this results in competing forces between surrounding strata influenced by local geometry and densities. This can result in spurious dimension estimates. Fortunately, the singular regions occupy a vanishingly small volume relative to regular points of $\m S$, and hence representing a low-probability event in the data-generating process. We can thus impose a minimal mass threshold to filter out these the dimension estimates that lack sufficient statistical support.

\begin{algorithm}[t]
\caption{Dimension Estimator}
\label{alg:dimest}
\begin{algorithmic}[1]
\Require Ambient dim $D$;
Sampling interval $[t_{start}, t_{end}]$; Number of score vectors $N$; Threshold $\alpha>0$.
\Statex \textbf{Definitions:} Let $m_t = \exp(-\int_0^t \beta(s)ds)$ be the mean coefficient.
\Statex \textbf{Data:} $\mathcal{X}=\{x_i\}_{i=1}^n$
\State Set $\mathcal{D}_k=\{\}$ for $k=1,...,D-1$
\For{$x_0 \in \mathcal{X}$}
    \State Set $S=[\,]$
    \For{$i=1,...,N$}
        \State Sample $t \sim \text{Unif}(t_{start},t_{end})$
        \State Sample $x_t^{(i)} \sim \mathcal{N}(m_t x_0,\; (1-m_t^2) I_D)$
        \State Append $s_\theta(x_{t}^{(i)},t)$ to $S$.
    \EndFor
    \State Compute the singular values $\lambda_1 \geq \cdots \geq \lambda_D$ of $S$.
    \State Let $\hat{k}=D-\arg\max_{i=1,...,D-1} (\lambda_i-\lambda_{i+1})$
    \State Append $x_0$ to $\mathcal{D}_{\hat{k}}$.
\EndFor
\State Discard sets $\m D_k$ where $\frac{|\m D_k|}{n}<\alpha$.
\State \Return Estimated number of strata $\hat{K}=|\{k: \m D_k \neq \emptyset\}|$ and their intrinsic dimensions $\{k:\m D_k \neq \emptyset\}$.
\end{algorithmic}
\end{algorithm}

The construction of the estimators based on these observations is summarized in Algorithm~\ref{alg:dimest}.
Recall that the singular values of a matrix $A$ are the square roots of the eigenvalues of $AA^\top$, so we consider second-moment matrices in the following. To establish consistency of these estimators at a regular point $x_0 \in \m S$, we first demonstrate that the second-moment matrix of the marginal score exhibits a spectral gap over a sufficiently small temporal window. For simplicity, we consider the case of a single manifold as the behavior of the score for regular points of $\m S$ follows analogously.

\begin{lemma}\label{lm:spectral_gap}
     Let $x_0$ be an element of a $d$-dimensional manifold $Y \subset \mathbb{R}^D$ and let $0 <t_{\text{start}}<t_{\text{end}}$. For $t \sim \text{Uniform}[t_{\text{start}},t_{\text{end}}]$ and $X \sim \mathcal{N}(m_{t}x_0,\sigma_{t}^2I_D)$, define the second-moment matrix of the score at $x_0$ as
     $$\Sigma(x_0) =\mb{E}_{t,X}[\nabla \log p_t(X)(\nabla \log p_t(X))^\top],$$
     Then the eigenvalues $\mu_1 \geq \cdots \geq \mu_D$ of $\Sigma(x_0)$ satisfy:
     \begin{enumerate}
         \item For $k \leq D-d$, $\mu_k =\mathbb{E}[1/\sigma_{t}^2] + O(\mb{E}[\sigma_t^{-1}])$.
         \item For $k > D-d$, $\mu_k = O(\mb{E}[\sigma_t^{-1}])$.
     \end{enumerate}
     Consequently, for sufficiently small $t_{\text{end}}$, the eigenvalues exhibit their largest spectral gap between indices $D-d$ and $D-d+1$.
\end{lemma}


While Theorem~\ref{th:normalalign} and Lemma~\ref{lm:spectral_gap} motivate the use of diffusion models in local intrinsic dimension estimation, these geometric properties assume access to the intractable marginal scores. By utilizing the fact that score approximation rates converge to 0 as $n\to \infty$ in Theorem~\ref{th:score}, we theoretically establish that our estimators of the number of strata and their intrinsic dimensions are statistically consistent. We define the local intrinsic dimension estimator as follows. First, take a sufficiently trained diffusion model $s_\theta(x,t)$ and a data point $x_0$. Then sample $N \geq D$ pairs $t_i \sim  [t_{\text{start}},t_{\text{end}}]$ and $X_i \sim \m N(m_tx_0)$, form the $D \times N$ matrix $S=[s_\theta(X_1,t_1),....,s_\theta(X_{N},t_{N})]$, and take the SVD of $S$. Let $\lambda_1 \geq \cdots \geq \lambda_D$ be the singular values of $S$, and define $$\hat{d}(x)=D-\operatorname{argmax}_i(\lambda_i-\lambda_{i+1}).$$
The optimal early stopping times and convergence rates in Theorem~\ref{th:score} fundamentally depend on intrinsic dimensions of the strata, since these dimensions determine the neural-network sizes required to attain the optimal score-approximation rate. Because Theorem~\ref{th:consistency} aims to estimate these dimensions from data, we cannot assume this information beforehand, nor take $t$ arbitrarily close to zero without a sufficiently large sample sizes. To decouple the consistency estimation from the dimension-dependent rate of score approximation, we impose the following high-level assumption:

\begin{assumption}
    For each $n \geq 1$, there exist sequences of times $0<t_1^{(n)},t_2^{(n)}\leq T$ and a score network $\hat{S}_n(x,t)$ with $\|\hat{S}_n(x,t)\|_\infty \leq V_n$ and $$\mb{E}_{X_{1:n}\sim P_*^{\otimes n}}\left[\frac{1}{t_2^{(n)}-t_1^{(n)}} \int_{t^{(n)}_1}^{t^{(n)}_2}\int_{\mb{R}^D}\|\hat{S}_n(x,t)-\nabla \log p_t(x)\|^2 p_t(x) \dd x\dd t \right] \leq \epsilon_n,$$
    with $\epsilon_n \to 0$ as $n \to \infty$.
\end{assumption}

\begin{theorem}\label{th:consistency}
    Suppose Assumptions 1-5 hold for a distribution $P_*$ supported on a stratified space $\m S=\bigcup_{k=1}^K M_k$ and sequences of times $\{t_1^{(n)},t_2^{(n)}\}_{n=1}^\infty$, and let $\Delta_n = t_2^{(n)}-t_1^{(n)}$. Further assume $V_n$ and $\sigma_{t_1^{(n)}}^{-1}$ grow at most polynomially in $n$. Fix $\ell \in \{1,...,K\}$ and $x_0 \in M_\ell$ with $\operatorname{dist}(x_0,\m S_{\text{sing}})\geq \delta >0$, where $\m S_{\text{sing}}=\bigcup_{i \neq j} (M_i \cap M_j)$. Also suppose $t_{2}^{(n)}\to 0$. Let $\{\hat{S}_n(x,t)\}_{n=1}^\infty$ be a sequence of score estimators satisfying the score approximation error $\epsilon_n$ on the intervals $[t_{1}^{(n)},t^{(n)}_{2}]$ with $\epsilon_n \to 0$. Sample $N_n$ i.i.d. pairs $t_i \sim Uniform(t^{(n)}_{1},t^{(n)}_{2})$ and $x_i \sim \m N(m_{t_i}x_0, \sigma_{t_i}^2I_D)$.
    If $N_n\to \infty$ and $n^\eta\sigma_{t^{(n)}_{1}}^{-d_\ell}\epsilon_n \to 0$ as $n\to \infty$ for some $\eta>0$, then the intrinsic dimension estimate $\hat{d}_n(x_0)$ is consistent, i.e.
    $$\mb{P}(\hat{d}_n(x_0)=d_\ell)\to 1\quad \text{as } n \to \infty.$$
\end{theorem}

\begin{remark}
    Although Lemma~\ref{lm:spectral_gap} and Theorem~\ref{th:consistency} are stated for the noiseless setting, the same argument extends to noisy observations $Y_0=X_0+Z$, $Z \in \m N(0,\sigma_*^2I_D)$, provided the noise level decreases with $n$ and satisfies $\sigma_{*,n} \leq C\sigma_{t_1^{(n)}}$. Indeed, the conditional density at the latent $X=x_0$ has the form $\m N(m_tx_0,\tilde{\sigma}_t^2I_D)$, where $\tilde{\sigma}_t^2=m_t^2\sigma_*^2 + \sigma_t^2$, so the proofs carry through by replacing $\sigma_t$ with $\tilde{\sigma}$. The constraint on $\sigma_*$ implies that $\tilde{\sigma}_t \asymp \sigma_t$ on $[t_1^{(n)},t_2^{(n)}]$, so the spectral gap is of order $\mb{E}[\tilde{\sigma}_t^{-2}]$, the density ratio exponent remains $d_\ell$, and the conclusions remain unchanged.
\end{remark}

In Theorem~\ref{th:consistency} we required a decreasing sequence of time values, which also decreases the respective noise levels. The distance $\delta$ from the singular region could also be allowed to decrease with $n$. Consequently, the $\delta$-neighborhood of $\m S_\text{sing}$ shrinks, allowing the score network to identify the dimensions of regular points arbitrarily close to the singular region. As the singular region has a strictly smaller intrinsic dimension than the component strata, $\m S_{\text{sing}}$ forms a set of measure zero under $P_*$. Thus, asymptotically, almost all points are regular, and one can expect the collection of the local dimension estimates to recover the distinct strata dimensions. This observation is summarized in the following corollary:

\begin{corollary}
    Suppose that the true intrinsic dimensions $d_1,...,d_K$ are distinct. Assume the conclusion of Theorem~\ref{th:consistency} holds uniformly over regular points at distance at least $\delta_n$ for some sequence $\delta_n \to 0$. Let $\mathcal{D}_n$ be the set of thresholded dimension estimates obtained from the modes of the histogram in Algorithm~\ref{alg:dimest} over all $n$ points.  Then
    $$\mathbb{P}(\mathcal{D}_n = \{d_k\}_{k=1}^K) \to 1 \quad \text{as } n \to \infty.$$
    Hence, the number of estimated strata satisfies $\mb{P}(\hat{K}=K) \to 1$.
\end{corollary}

In the single-manifold setting, the local intrinsic dimension estimator in \citet{stanczuk2024diffusion} is motivated by the asymptotic normal alignment of the score as $t \to 0^+$; however, they do not provide statistical guarantees. To the best of our knowledge, our analysis provides the first establishment of statistical consistency of diffusion-based local intrinsic dimension estimation in the single-manifold and stratified settings.

\section{Simulation study and real data analysis}

In this section,  we empirically demonstrate the effectiveness of the Algorithm 1 as a local intrinsic dimension estimator and compare it to standard baselines such as the Levina-Bickel MLE and Local PCA. For a sufficiently trained diffusion model, this will give a \textit{stratification} of the space into its components. We also exhibit the ability for a Mixture-of-VAE model to learn the stratified structure and learn generators for each manifold. In addition, we evaluate the error of the distribution estimator in $W_1$ distance and compare the performance of the two methods.

\subsection{Diffusion-based Dimension Estimation}

Two standard approaches to estimating local intrinsic dimension are Local Principal Component Analysis (Local PCA) and the Levina-Bickel Maximum Likelihood estimator (LB-MLE). While both of these methods base their estimators off a $k$-nearest neighbors approach, Algorithm~\ref{alg:dimest} samples score vectors diffused points originating from a single data point. Although the process of training a diffusion model and performing singular-value decomposition (SVD) on a potentially large matrix can be computationally expensive, Algorithm~\ref{alg:dimest} is inherently more robust to noise and also demonstrates competitive results to standard baselines in the noiseless setting.

Lemma~\ref{lm:spectral_gap} provided evidence that such a spectral gap is expected from Algorithm~\ref{alg:dimest} between the $D-d$ and $D-d+1$ eigenvalues for points on a $d$-dimensional manifold. Performing this procedure illustrated promising results in our synthetic datasets, but it was sensitive to the choice of time intervals. A natural substitute that we found to be more effective and robust was the ratio 
$$\frac{\lambda_i}{\max\{\lambda_{i+1},\epsilon\}},$$
where $\epsilon$ is a floor used for numerical stability. Part of this discrepancy can possibly be attributed to using a time window instead of one particular time $t_0$, which fluctuates the magnitude of the score vectors. We empirically observed that the ratio is more stable across time windows and agrees on windows where they are both accurate. Hence, the experiments below will use the ratio rather than the difference of singular values to estimate the local intrinsic dimension of points.

We implement denoising diffusion probabilistic models (DDPM) \citep{ho2020denoising} to train a neural network $\epsilon_\theta=\mathbb{E}[\epsilon|x_t]$ that estimates the noise $\epsilon$ in the equation 
$$x_t = m_tx_0 + \sigma_t\epsilon,\quad \epsilon \sim \mathcal{N}(0,I_D).$$
Since $p_t(x_t|x_0) = \mathcal{N}(m_tx_0,\sigma^2_tI_D)$, we have that $$\nabla \log p_t(x_t|x_0) = - \frac{x_t - m_tx_0}{\sigma_t^2}= -\frac{\sigma_t\epsilon}{\sigma_t^2}=-\frac{\epsilon}{\sigma_t}$$
Using Tweedie's formula, we can express the score as
$$\nabla \log p_t(x_t) = \mathbb{E}[\nabla \log p_t(x_t|x_0) \mid x_t] = - \frac{1}{\sigma_t}\mathbb{E}[\epsilon|x_t].$$
Therefore, by using $\epsilon_\theta \approx \mathbb{E}[\epsilon|x_t]$, we get that
$$\nabla \log p_t(x_t)\approx -\frac{\epsilon_\theta(x_t,t)}{\sigma_t}.$$

\subsubsection{Synthetic Datasets}

We consider two synthetic datasets that consist of low-dimensional manifold structures embedded in $\mathbb{R}^{50}$. We use a ReLU neural network with widths $(114,512,512,512,50)$. The 114 consists of 50 for the ambient dimension and 64 dimensions for the time embedding, which is another ReLU MLP with width vector $(1,64,64)$. We take $N=500$ score vectors in Algorithm~\ref{alg:dimest}.

\paragraph{Circle and Sphere} The first consists of a circle and sphere, which can be expressed as
\begin{align*} 
    M_1 &= \{(x,y,z) \in \mathbb{R}^3 : (x-0.5)^2 + y^2=1.2^2, z=0\}\\
    M_2 &= \{(x,y,z) \in \mathbb{R}^3: x^2+y^2+z^2=1\}
\end{align*}

We consider the uniform distribution $Q_1$ on the circle and uniform distribution $Q_2$ on the sphere, the mixture
$$Q_* = 0.4\cdot Q_1 + 0.6\cdot Q_2$$
as well as the convolved distribution
$$P_* = Q_* +\epsilon,$$
where $\epsilon \sim \mathcal{N}(0,\sigma_*^2I_D)$. We take the values $\sigma_*=0, 0.02, 0.05$ in our experiments.

The two manifolds intersect at two points, leaving an arc within the sphere and the rest of the circle on the outside. The data is simulated in $\mb{R}^3$ and then embedded into $\mathbb{R}^{50}$, where ambient noise is added. The Levina-Bickel MLE and Local PCA estimate a large portion correctly in the case of no noise, but the accuracy quickly descends to zero once noise is included. These are visualized in Figure~\ref{fig:circle_sphere_classical}, where these methods capture the lower-dimensional for $\sigma_*=0$, but fail to identify any points of the strata as one-, two-, or even three-dimensional when $\sigma_*=0.05$. Figure~\ref{fig:circle_sphere} illustrates the results from Algorithm~\ref{alg:dimest} applied to every point, and Tables~\ref{tab:dim_est_sigma0}, \ref{tab:dim_est_sigma002}, and \ref{tab:dim_est_sigma005} provide the dimension counts and accuracies. The diffusion-based method is stable across many time windows and provides the correct dimensions at various small noise levels.

\begin{figure}[h]
  \centering
  \begin{subfigure}[t]{0.48\linewidth}
    \centering
    \includegraphics[width=6cm]{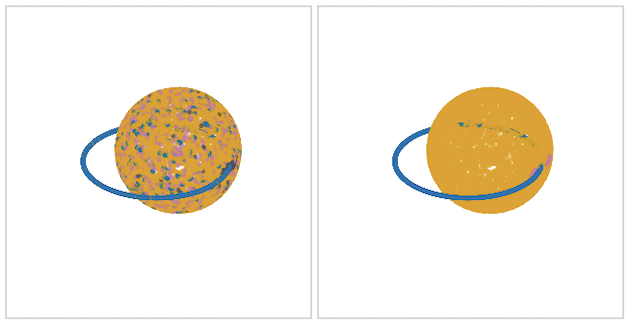}
    \caption{$\sigma_*=0$}
  \end{subfigure}
  \begin{subfigure}[t]{0.48\linewidth}
    \centering
    \includegraphics[width=6cm]{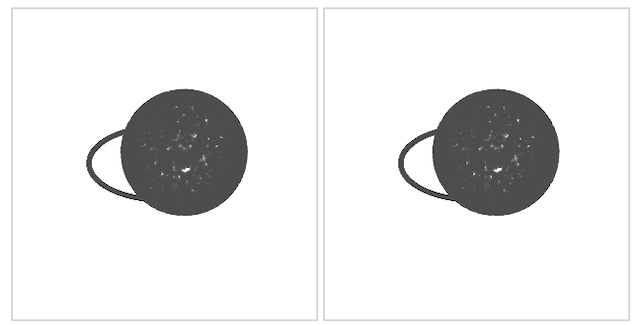}
    \caption{$\sigma_*=0.05$}
  \end{subfigure}
  \caption{Points Labeled as 1D (blue), 2D (orange), 3D (pink), and 4D+ (gray) by Levina-Bickel MLE (left) and Local PCA (right).}
  \label{fig:circle_sphere_classical}
\end{figure}

\begin{figure}[h]
  \centering
  \includegraphics[width=14cm]{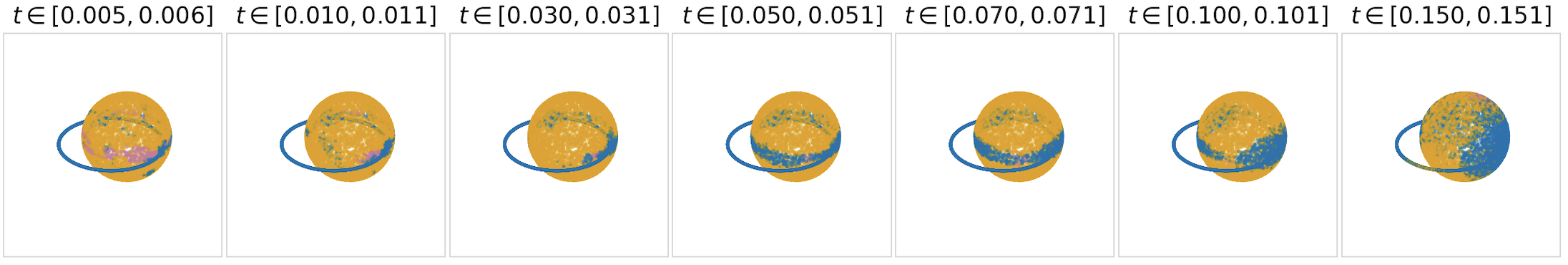}
  \includegraphics[width=14cm]{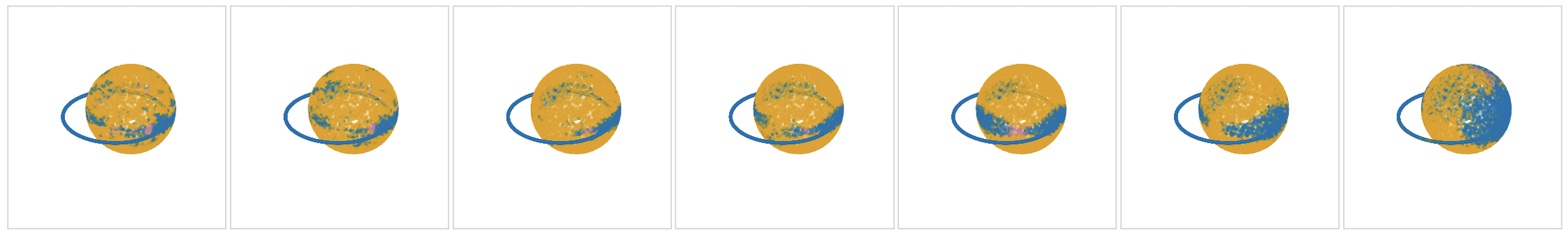}
  \includegraphics[width=14cm]{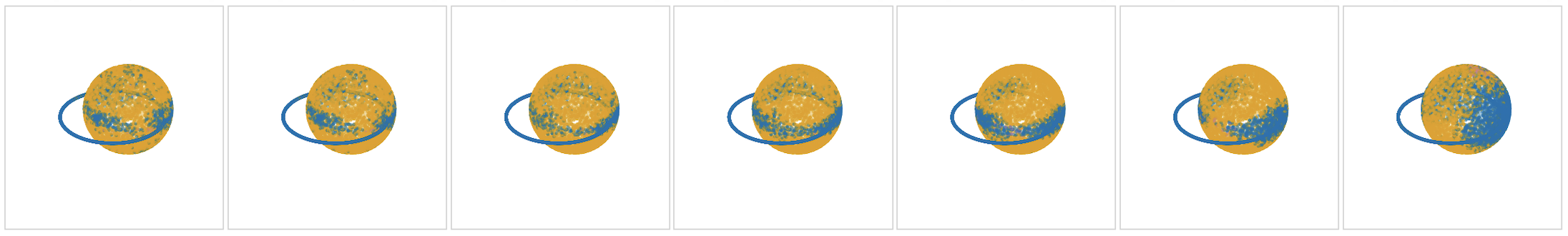}
  \caption{Points Labeled as 1D (blue), 2D (orange), and 3D (pink) by Algorithm~\ref{alg:dimest} over time intervals with $\sigma_*=0$ (top row), $\sigma_*=0.02$ (middle row), and $\sigma_*=0.05$ (bottom row).}
  \label{fig:circle_sphere}
\end{figure}

\begin{table}[H]
\centering
\scriptsize
\setlength{\tabcolsep}{2.5pt}
\begin{tabular}{lccccccccc}
\toprule
& $[0.005,0.006]$ & $[0.010,0.011]$ & $[0.030,0.031]$ & $[0.050,0.051]$ & $[0.070,0.071]$ & $[0.100,0.101]$ & $[0.150,0.151]$ & LB & L-PCA \\
\midrule
Dim $1$  & 43.5 & 44.8 & 44.7 & 49.2 & 44.3 & 45.8 & 65.9 & 43.5 & 39.9\\
Dim $2$  & 51.9 & 54.0 & 55.0 & 50.6 & 55.5 & 54.2 & 33.0 & 45.9 & 59.6\\
Dim $3$  &  4.5 &  1.2 &  0.3 &  0.2 &  0.2 &  0.1 &  1.1 & 7.8 & 0.5\\
Dim $4+$ &  0.0 &  0.0 &  0.0 &  0.0 &  0.0 &  0.0 &  0.0 & 1.0 & 0\\
\midrule
Accuracy & 91.92 & 93.96 & 95.04 & 90.64 & 84.97 & 80.90 & 67.17 & 83.76 & 99.40\\
\bottomrule
\end{tabular}
\caption{Algorithm~1 dimension estimates (in percentages) for the union of circle and sphere with $n=10000$ sample and no ambient noise ($\sigma_*=0)$, along with Levina-Bickel (LB) and Local PCA (L-PCA) estimates.}
\label{tab:dim_est_sigma0}
\end{table}

\begin{table}[H]
\centering
\scriptsize
\setlength{\tabcolsep}{2.5pt}
\renewcommand{\arraystretch}{0.92}
\begin{tabular}{lccccccc}
\toprule
& $[0.005,0.006]$ & $[0.010,0.011]$ & $[0.030,0.031]$ & $[0.050,0.051]$ & $[0.070,0.071]$ & $[0.100,0.101]$ & $[0.150,0.151]$ \\
\midrule
Dim $1$ & 48.6 & 48.4 & 46.5 & 48.0 & 43.8 & 44.7 & 66.3 \\
Dim $2$ & 50.1 & 50.8 & 53.1 & 51.8 & 55.5 & 55.3 & 32.8 \\
Dim $3$ &  1.3 &  0.9 &  0.4 &  0.2 &  0.7 &  0.0 &  0.9 \\
Dim $4+$&  0.0 &  0.0 &  0.0 &  0.0 &  0.0 &  0.0 &  0.0 \\
\midrule
Accuracy & 89.66& 90.65& 93.10& 91.81& 85.53& 80.81& 70.01 \\
\bottomrule
\end{tabular}
\caption{Algorithm~1 dimension estimates (in percentages) for the union of circle and sphere with $n=10000$ sample and noise level $\sigma_*=0.02$}
\label{tab:dim_est_sigma002}
\end{table}

\begin{table}[H]
\centering
\scriptsize
\setlength{\tabcolsep}{2.5pt}
\renewcommand{\arraystretch}{0.92}
\begin{tabular}{lccccccc}
\toprule
& $[0.005,0.006]$ & $[0.010,0.011]$ & $[0.030,0.031]$ & $[0.050,0.051]$ & $[0.070,0.071]$ & $[0.100,0.101]$ & $[0.150,0.151]$ \\
\midrule
Dim $1$ & 48.6 & 47.5 & 49.6 & 50.8 & 43.8 & 44.8 & 67.8 \\
Dim $2$ & 50.9 & 52.0 & 50.3 & 49.2 & 55.9 & 55.0 & 31.4 \\
Dim $3$ &  0.5 &  0.5 &  0.1 &  0.0 &  0.2 &  0.2 &  0.9 \\
Dim $4+$&  0.0 &  0.0 &  0.0 &  0.0 &  0.0 &  0.0 &  0.0 \\
\midrule
Accuracy  & 88.59& 90.52& 90.29& 89.16& 83.19& 81.09& 70.83 \\
\bottomrule
\end{tabular}
\caption{Algorithm~1 dimension estimates (in percentages) for the union of circle and sphere with $n=10000$ sample and noise level $\sigma_*=0.05$}
\label{tab:dim_est_sigma005}
\end{table}

\paragraph{Union of Four Manifolds} We consider a disjoint union of a 1-dimensional helix, a 2-dimensional torus, a 4-dimensional sphere, and a 7-dimensional sphere embedded in $\mathbb{R}^{50}$. Explicitly, embed the following manifolds:
\begin{align*}
    M_1&=\{(\cos t,\sin t, 0.1t) \in \mb{R}^3:0\leq t \leq 4\pi\}\\
    M_2&= \{((2+0.5\cos \theta)\cos \phi, (2+0.5\cos \theta)\sin \phi, r \sin \theta):0\leq \theta,\phi \leq 2\pi\}\\
    M_3&= \{x \in \mb{R}^5:\|x\|^2=1\}\\
    M_4&= \{x\in \mb{R}^8:\|x\|^2=1\}
\end{align*}

We then apply a translation to each manifold so that they are disjoint in the ambient space. We equip each manifold with the uniform distribution $Q_k$ and define the mixture
$$Q_*=0.15\cdot Q_1 + 0.2 \cdot Q_2 + 0.25 \cdot Q_3 + 0.4 \cdot Q_4.$$ We provide the confusion matrix for the best result of Algorithm~\ref{alg:dimest} over our chosen time intervals in Table 4.


\begin{table}[H]
\centering

\begin{minipage}{0.32\textwidth}
\centering
\scriptsize
\setlength{\tabcolsep}{2.5pt}
\begin{tabular}{lrrrrr}
\toprule
& \multicolumn{5}{c}{Predicted dim} \\
\cmidrule(lr){2-6}
True dim & 1 & 2 & 4 & 7 & Other \\
\midrule
1 \,(0.15) &    0 &   0 &    0 &    0 & 1500 \\
2 \,(0.20) &    0 &   0 &    0 &    0 & 2000 \\
4 \,(0.25) &    0 &   0 & 2499 &    0 &    1 \\
7 \,(0.40) &    0 &   0 &    0 &   80 & 3920 \\
\midrule
\multicolumn{6}{l}{Overall Acc: $25.79\%$ (Local PCA)}\\
\bottomrule
\end{tabular}
\end{minipage}\hfill
\begin{minipage}{0.32\textwidth}
\centering
\scriptsize
\setlength{\tabcolsep}{2.5pt}
\begin{tabular}{lrrrrr}
\toprule
& \multicolumn{5}{c}{Predicted dim} \\
\cmidrule(lr){2-6}
True dim & 1 & 2 & 4 & 7 & Other \\
\midrule
1 \,(0.15) &    0 &   0 &   0 &    0 & 1500 \\
2 \,(0.20) &    0 &   0 &   0 &  245 & 1755 \\
4 \,(0.25) &    0 &   0 & 933 &   92 & 1475 \\
7 \,(0.40) &    0 &   0 &  76 & 1024 & 2900 \\
\midrule
\multicolumn{6}{l}{Overall Acc: $19.57\%$ (LB)}\\
\bottomrule
\end{tabular}
\end{minipage}\hfill
\begin{minipage}{0.32\textwidth}
\centering
\scriptsize
\setlength{\tabcolsep}{2.5pt}
\begin{tabular}{lrrrrr}
\toprule
& \multicolumn{5}{c}{Predicted dim} \\
\cmidrule(lr){2-6}
True dim & 1 & 2 & 4 & 7 & Other \\
\midrule
1 \,(0.15) & 1338 & 136 &    0 &    0 &   26 \\
2 \,(0.20) &  254 & 867 &    0 &    0 &  879 \\
4 \,(0.25) &    1 &   0 & 2449 &    0 &   50 \\
7 \,(0.40) &    0 &   0 &    0 & 3920 &   80 \\
\midrule
\multicolumn{6}{l}{Overall Acc: $85.74\%$ (Algorithm 1)}\\
\bottomrule
\end{tabular}
\end{minipage}

\vspace{0.3cm}
\caption{Comparison of confusion matrices for baseline estimators (Local PCA, Levina-Bickel) and Algorithm 1 on the 4-manifold mixture with noise $\sigma_*=0.05$. Algorithm 1 results reflect the time interval with best accuracy, $t\in[0.020,0.021]$.}
\label{tab:all_methods_confusion}
\end{table}

\subsubsection{Real-World Datasets}
To illustrate the diffusion-based dimension estimation method on real-world datasets we apply Algorithm~\ref{alg:dimest} to molecular dynamics data, which is a common application in manifold learning. We perform the analysis on butane, which is a small molecule consisting of four atoms, and the standard alanine dipeptide molecule. 

\paragraph{Butane}
Butane consists of four atoms and can be represented as a 12-dimensional vector, and \cite{sule2025learning} show that the dynamics can be represented as 1- and 2-dimensional structures based on the dihedral angles. Table 5 shows that the classical estimators predict dimensions being between 6-8, which is well-above the two dimensions established.

\begin{table}[h]
\centering
\begin{tabular}{lccccc}
\hline
Estimator & $N$ & Mean dim & Median dim & IQR & Mode dim \\
\hline
Local PCA (k=20, 95\% var) 
  & 2000 & 7.928 & 8.000 & [8.000, 8.000] & 8 \\
Levina--Bickel MLE (k=20) 
  & 2000 & 7.043 & 6.761 & [5.598, 8.198] & 6 \\
\hline
\end{tabular}
\caption{Classical intrinsic dimension estimates for butane dynamics.}
\end{table}

For our diffusion-based estimator, we consider a range of $t$ values and deduce the estimated dimensions in Table 6. The estimated dimensions are much higher around 0.01, which is expected as the score is more difficult to approximate. The time values from 0.05-0.3 show a majority of 1-dimensional and 2-dimensional classifications. This matches what is expected, except that around 0.3 there is an influx in 1-dimensional counts and almost no points classified as 3-dimensional or higher. We suspect this is a result of oversmoothing, but it still aligns with our presumed knowledge of dimensions.

\begin{table}[t]
\centering
\small
\setlength{\tabcolsep}{6pt}
\begin{tabular}{c|cccc|c}
\toprule
Time & Dim 1 & Dim 2 & Dim 3 & Dim 11 & Mean \\
\midrule
$[0.01,0.011]$ & 153 (30.6\%) & 47 (9.4\%)  & 0 (0.0\%)  & 300 (60.0\%) & 7.10 \\
$[0.03,0.031]$ & 290 (58.0\%) & 131 (26.2\%) & 4 (0.8\%)  & 75 (15.0\%)  & 2.78 \\
$[0.05,0.051]$ & 322 (64.4\%) & 149 (29.8\%) & 13 (2.6\%) & 16 (3.2\%)   & 1.67 \\
$[0.07,0.071]$ & 312 (62.4\%) & 166 (33.2\%) & 18 (3.6\%) & 4 (0.8\%)    & 1.48 \\
$[0.10,0.101]$ & 251 (50.2\%) & 241 (48.2\%) & 8 (1.6\%)  & 0 (0.0\%)    & 1.51 \\
$[0.15,0.151]$ & 183 (36.6\%) & 314 (62.8\%) & 3 (0.6\%)  & 0 (0.0\%)    & 1.64 \\
$[0.20,0.201]$ & 224 (44.8\%) & 274 (54.8\%) & 1 (0.2\%)  & 1 (0.2\%)    & 1.57 \\
\bottomrule
\end{tabular}
\caption{Intrinsic dimension estimates for Butane dynamics over 500 samples.}
\label{tab:butane}
\end{table}

\paragraph{Alanine Dipeptide}
The dynamics of alanine dipeptide is frequently used as a baseline dataset for local intrinsic dimension estimators. It is typically better to restrict to the use of heavy atoms since the hydrogen atoms may contribute high-frequency motion and hence increase the intrinsic dimension estimate. Dimension-reduction techniques support the hypothesis that alanine dipeptide has the intrinsic structure of a two-dimensional torus \citep{hashemian2013modeling}, and this is usually accepted in practice. From Table 7, we observe that Algorithm~\ref{alg:dimest} assigns over half of the points to dimension 2 when the time is not too small, with the rest being classified as mostly dimension 1 or 3.

\begin{table}[!t]
\centering
\small
\setlength{\tabcolsep}{4pt}
\resizebox{\linewidth}{!}{%
\begin{tabular}{c|cccccc|c}
\hline
Time & Dim 1 & Dim 2 & Dim 3 & Dim 4 & Dim 5 & Dim 29 & Mean \\
\hline
$[0.01,0.011]$ & 7 (1.4\%) & 31 (6.2\%) & 96 (19.2\%) & 72 (14.4\%) & 84 (16.8\%) & 156 (31.2\%) & 12.62 \\
$[0.03,0.031]$ & 14 (2.8\%) & 90 (18.0\%) & 240 (48.0\%) & 28 (5.6\%) & 22 (4.4\%) & 104 (20.8\%) & 6.55 \\
$[0.05,0.051]$ & 26 (5.2\%) & 185 (37.0\%) & 247 (49.4\%) & 8 (1.6\%) & 2 (0.4\%) & 32 (6.4\%) & 4.10 \\
$[0.07,0.071]$ & 46 (9.2\%) & 243 (48.6\%) & 201 (40.2\%) & 1 (0.2\%) & 0 (0.0\%) & 9 (1.8\%) & 3.04 \\
$[0.10,0.101]$ & 92 (18.4\%) & 265 (53.0\%) & 141 (28.2\%) & 0 (0.0\%) & 0 (0.0\%) & 2 (0.4\%) & 2.51 \\
$[0.15,0.151]$ & 141 (28.2\%) & 285 (57.0\%) & 74 (14.8\%) & 0 (0.0\%) & 0 (0.0\%) & 0 (0.0\%) & 1.87 \\
$[0.20,0.201]$ & 177 (35.4\%) & 294 (58.8\%) & 28 (5.6\%) & 0 (0.0\%) & 0 (0.0\%) & 1 (0.2\%) & 1.75 \\
\hline
\end{tabular}%
}
\caption{Alanine dipeptide (heavy atoms) dimension estimates across time windows using 500 samples. For brevity, only the most frequent dimensions are shown.}
\label{tab:alanine_dims}
\end{table}

\subsection{ Distribution Estimation via Mixture-of-VAEs and Diffusion Models}

\subsubsection{Mixture-of-VAEs}
Variational autoencoders (VAEs) \citep{kingma2013auto} are trained by maximizing a tractable variational lower bound on the marginal log-likelihood, commonly referred to as the evidence lower bound (ELBO):
$$L(x)=\log p_\theta(x) - \text{KL}(q_\phi(z|x) \| p_\theta(z|x)).$$

To bridge our theoretical setting with practical implementation, we approximate the distribution on a stratified space as a mixture model $f=\sum_{k=1}^K h_kf_k$ (simplified to $J_k=1$), where $h_k$ acts as an indicator function on stratum $k$. The log-likelihood then takes the form
$$\log p(x) = \log \int \phi_\sigma(x-f(z))dP_Z(z)
    = \log \sum_{k=1}^K \int_{\{h_k=1\}} \phi_\sigma(x-f_k(z))dP_Z(z)$$

Let $A_k=\{z:h_k(z)=1\}$, define the prior weights $\omega_k=P(z \in A_k)$, and let the conditional prior be $P_k(z)=P_Z(z| z\in A_k)$. Introduce the categorical latent variable $C \sim \text{Cat}(\omega_1,...,\omega_k)$. Then the last expression can be framed as an expectation, and we can apply Jensen's inequality to derive our ELBO:
\begin{align*}
    \log \sum_{k=1}^K w_k \int \phi_\sigma(x-f_k(z))dP_k(z) &=\log \sum_{k=1}^K \int p(x,z,c=k)dz\\
    &= \log \mathbb{E}_{q(c,z|x)} \frac{p(x,z,c)}{q(z,c|x)}\\
    & \geq \mathbb{E}_{q(c,z|x)} \log\frac{p(x,z,c)}{q(z,c|x)}\\
    &= \mathbb{E}_{q(c,z|x)} [\log p(x,z,c) - \log q(z,c|x)] =: L_{ELBO}
\end{align*}
Factorizing $q(c,z|x)=q(c|x)q(z|c,x)$ and parametrizing the routing network as $q(c|x)=\text{softmax}(g_\theta(x))=: \hat{h}_k(x)$, the final decomposition becomes
$$L_{ELBO}=\sum_{k=1}^K \hat{h}_k(x)[\mathbb{E}_{q(z|x,k)} \log \phi_\sigma(x-f_k(z)) - \text{KL}(q(z|x,k)||p_k(z)) + \log w_k - \log \hat{h}_k(x)].$$

To effectively capture general manifolds that are not globally parametrizable, one would typically need to construct an atlas of $J_k$ charts for each manifold $k$. While mixtures of generative models have been used to perform this task on single manifolds \citep{TangYang2}, these approaches generally require knowledge of the intrinsic dimension and a pre-computed clustering of the data. Mixture-of-expert models naturally circumvent these requirements, making them more theoretically appealing. However, they are notoriously difficult to train in the unsupervised setting; frequently suffering from issues such as mode or posterior collapse. Furthermore, to learn the topological structure of a stratified space, the model must also separate the data by dimension. Assuming that the strata have distinct intrinsic dimensions makes this sufficient to identify the distinct submanifolds.

A $d$-dimensional manifold is locally homeomorphic to $\mathbb{R}^d$, so a VAE trained on a local neighborhood should learn a mapping of rank $d$. To explicitly incentivize the model to separate these dimensions, we introduce a projection consistency penalty to the objective function. Specifically, we assign each expert $k$ a dimension $d_k$ and penalize the loss function if the rank of its autoencoder deviates from $d_k$. Geometrically, an expert of dimension $d_k$ should have $d_k$ singular values close to 1, and $D-d_k$ near 0, where $D$ is the ambient dimension. If $\sigma_1 \geq \sigma_2 \geq \cdots \geq \sigma_D$ are the singular values, then the penalization term we impose for expert $k$ is $$P_k(x)=\frac{1}{d}\sum_{i=1}^k (\sigma_i-1)^2 + \frac{1}{D-d}\sum_{i=k+1}^D \sigma_i^2.$$

\begin{algorithm}
\caption{Stratified Mixture of VAEs}
\label{alg:movae}
\begin{algorithmic}[1]
\Require Ambient dim $D$; number of experts $K$, dimensions $d_k$ for $k=1,...,K$, hyperparameters $\beta,\beta_g,\gamma>0$;
\Statex \textbf{Data:} $\mathcal{X}=\{x_i\}_{i=1}^n$
\State Assign each expert $L_k$ a target dimension $d_k$ for each $k \in [K]$
\State Initialize encoders $\{q_{\phi_k}\}$, decoders $\{p_{\theta_k}\}$, gate logits $g_\psi(\cdot)$, and mixture logits $\pi$
\Repeat
  \State Sample mini-batch $\mathcal{B}=\{x_b\}_{b=1}^B\subset\mathcal{X}$.
  \State Set global weights $w \gets \mathrm{softmax}(\pi)$
  \For{$x \in \mathcal{B}$}
    \State $\hat{h}(x)\gets\mathrm{softmax}(g_\psi(x))$
    \For{$k=1$ \textbf{to} $K$}
      \State $z_k \sim q_{\phi_k}(z\mid x)$
      \State $\hat{x}_k \gets \mathbb{E}[x\mid z_k;\theta_k]$
      \State Compute decoder Jacobian $J_k(x)$ and penalty $P_k(x)$.
      \State $\ell_k(x) \gets \|x-\hat{x}_k\|_2^2
      + \beta \cdot\mathrm{KL}\!\big(q_{\phi_k}(z\mid x)\,\|\,\mathcal{N}(0,I)\big)
      -\beta_g(\log w_k - \log \hat{h}_k(x)) +\gamma\,P_k(x)$
    \EndFor
  \EndFor
  \State $\mathcal{L}(\mathcal{B}) \gets \frac{1}{B}\sum_{x\in\mathcal{B}}\sum_{k=1}^K \hat{h}_k(x)\,\ell_k(x)$
  \State Update $(\phi,\theta,\psi)$ with one Adam step on $\mathcal{L}(\mathcal{B})$.
\Until{convergence}
\end{algorithmic}
\end{algorithm}

Our training procedure is summarized in Algorithm~\ref{alg:movae}. By integrating this dimension-aware regularizer, our architecture forces the routing network to distinguish between dimensions, which is a mechanism that is absent in standard Mixture-of-VAEs. However, successfully training such a model requires carefully balancing weights of the objective function to avoid collapse. For instance, while standard $\beta$-VAE formulations employ a KL penalty greater than 1 to encourage disentanglement \citep{burgess2018understanding}, we operate in the regime where $\beta \ll 1$ to focus on separation of strata and lock in the mixture weights, and one can anneal $\beta$ back to 0.5 or 1 to enforce generative regularization. We also have a $\beta_g$ hyperparameter for the gate KL weight, which we leave at 0.01 for experiments. These aims to decouple potentially conflicting objectives. We also emphasize the effect of $\gamma$ on the training process: if $\gamma$ is too low, there will not be enough geometric signal, while too high $\gamma$ degrades reconstruction fidelity and hence prevents accurate dimension estimation. Recall that for $a>0$, $\operatorname{LeakyReLU}(x)=\max\{0,x\}+a\min\{0,x\}$). We consider a latent dimension of 4, and 2 hidden layers of width 256 for each encoder and decoder equipped with Leaky ReLU activation with $a=0.2$. The gate is also a Leaky ReLU network with architecture $(3,128,2)$ as we have only 2 experts for the output. To stabilize training early on, we perform Gumbel-softmax temperature annealing from $\tau=2.0$ to $\tau =0.1$.

\paragraph{Circle and Plane} The first stratified space we consider is the union of a circle and a 2-dimensional plane:
$$M_1=\{(\cos \theta, 0, \sin \theta): \theta \in [0,2\pi)\}$$
$$M_2=\left\{(x,y,0):x,y \in \left[-\frac{3}{2},\frac{3}{2}\right]\right\}.$$
We let $Q_i$ be the uniform distribution on $M_i$ and set the mixture weights as $\omega_1=\omega_2=0.5$. This space is exhibited in Figure 3, where the mixture of VAEs successfully routes the points to the correct experts apart from a small neighborhood around the singularities.

\begin{figure}[t]
    \centering
    \includegraphics[width=16cm]{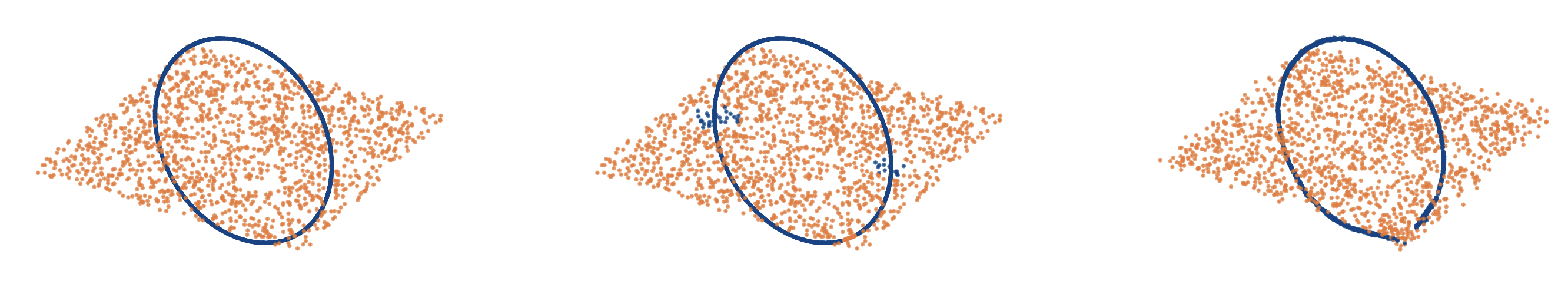}
    \caption{Ground truth (left), classified points (middle), reconstructions (right) for the union of a circle and $2$-dimensional plane with $n=3000$ samples.}
    \label{fig:circle_plane}
\end{figure}

\paragraph{Helix and Swiss Roll} We also consider the 1-dimensional helix and 2-dimensional Swiss roll:
$$M_1=\left\{(\cos t - \frac{5}{2}, \sin t, \frac{t}{2}-\pi): t \in [0,4\pi]\right\}$$
$$M_2=\left\{ \left(\frac{t\cos t}{5} +\frac{5}{2}, \frac{h}{5}, \frac{t\sin t}{5}\right): t\in \left[ \frac{3\pi}{2},\frac{9\pi}{2} \right], h \in\left[ -\frac{5}{2}, \frac{5}{2}\right]\right\},$$
with the mixture distribution $Q_*=\frac{1}{3}Q_1 + \frac{2}{3}Q_2$, where $Q_i$ is the uniform distribution on $M_i$. The space is displayed in Figure 4.

\begin{figure}[t]
    \centering
    \includegraphics[width=15cm]{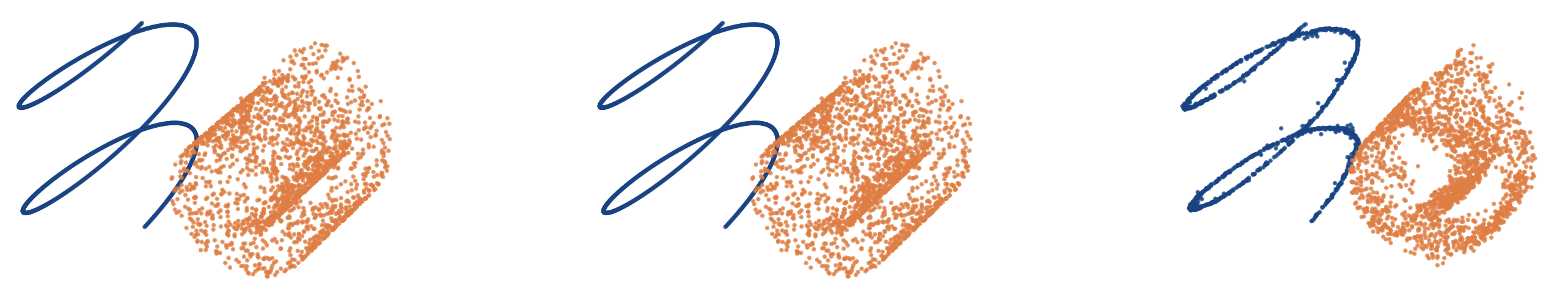}
    \caption{Ground truth (left), classified points (middle), reconstructions (right) for the union of a curve and Swiss roll with $n=3000$ samples.}
    \label{fig:noodle_swiss}
\end{figure}

\subsubsection{Distribution Estimation}
As discussed in Remark~\ref{rem12}, the sieve model class in the likelihood-based approach admits a natural deconvolution interpretation when ambient noise is present, while diffusion models target the observed distribution. Consequently, diffusion models perform extremely well in estimating the intrinsic distribution in the case of no noise or small noise levels, but struggle to recover the intrinsic distribution for moderate to high noise levels. On the other hand, VAEs provide better estimation of the intrinsic distribution as long as the noise level is not too small.

To illustrate these behaviors, we consider the 1-dimensional helix and 2-dimensional Swiss roll from Figure 4, but embedded in $\mb{R}^{15}$ with mixture weights $\omega_1=\omega_2=0.5$ and Leaky ReLU networks with $a=0.2$ with 4 hidden layers of width 512. We run standard DDPM, while the mixture of VAEs was trained with $\beta=0.01$ for the first 2000 epochs, then 2000 more with frozen routing while annealing $\beta$ to 0.5. The (sliced) Wasserstein-1 distance over various noise levels is portrayed in Figure 5 at $n=1000$ and $n=3000$ samples. We see that the mixture of VAEs attains small error at both of these sample sizes for noise levels up to around $\sigma_*=0.2$, but the gap is most pronounced at $\sigma=0.5$, where $W_1$ drops substantially as $n$ increases from $n=1000$ to $n=3000$. Also, the diffusion model is more accurate at low noise but eventually becomes worse than VAE as $\sigma_*$ grows. The curves cross around $\sigma_*=0.15$ when $n=3000$, and at even small $\sigma_*$ when $n=1000$, after which the diffusion model's sliced Wasserstein-1 distance increases steadily. This behavior is consistent with the fact that diffusion model learn the smoothed, noised-corrupted distribution rather than the underlying intrinsic one. Furthermore, while the VAE's error naturally grows at extreme noise levels ($\sigma_*=0.5$), this gap shrinks significantly as the sample size increases. 

\begin{figure}
    \centering
    \includegraphics[width=15cm]{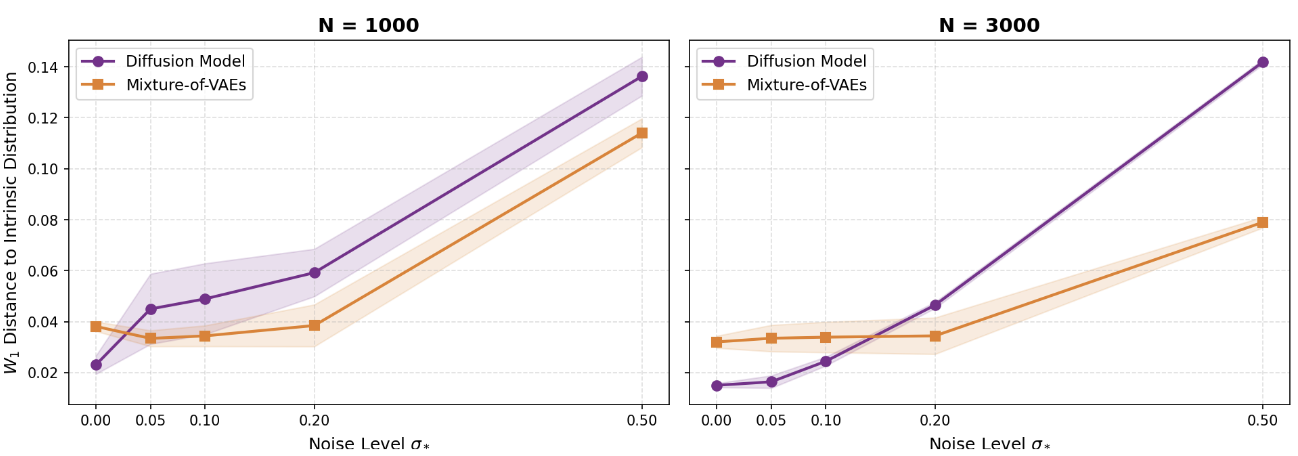}
    \caption{Comparison of the (Sliced) Wasserstein-1 distance between learned distribution and the intrinsic distribution via diffusion models and Mixture of VAEs at $n=1000$ and $n=3000$ samples.}
    \label{fig:noise_comparisons}
\end{figure}

\section{Discussion}

This paper develops a generative framework for \emph{stratified learning}. Going beyond the classical manifold setting, we study distributions supported on or near unions of manifolds of different dimensions, possibly with singular intersections. Our results show that deep generative models can adapt not only to low-dimensional geometry, but also to more general stratified structures. In particular, the paper provides two  generative approaches: a likelihood-based framework through a sieve MLE and a diffusion-based framework through score estimation. Together, these approaches elucidate how geometry, singularity, and ambient noise interact in determining the statistical  theory of  stratified learning.

A main message of the paper is that the two frameworks are suited to different regimes. The sieve-MLE approach provides a natural likelihood-based model for the ambient distribution and yields a direct estimator of the intrinsic distribution through the learned generator. This makes it particularly appealing when the data are generated with non-negligible Gaussian noise (when the noise delays to zero at a moderate rate) and one wishes to interpret the problem as deconvolution. At the same time, the analysis also reveals the inherent disadvantages  of likelihood-based methods in the near-singular regime: when the noise level is too small, the ambient density becomes sharply concentrated near the stratified support, which makes maximum-likelihood estimation unstable. In contrast, the diffusion-based framework remains well posed even in the noiseless setting, since the forward process itself regularizes the distribution through Gaussian smoothing. This makes diffusion models especially attractive for learning distributions on highly singular or nearly singular stratified spaces.

Another important insight from our analysis is that ambient noise does not play a purely detrimental role. In the likelihood framework, too much noise obscures the underlying geometry, while too little noise creates instability through singularity. For near-singular  case, the injected noise in a data perturbation approach becomes a blessing. In the diffusion framework, moderate noise can have a favorable effect by smoothing the target distribution and stabilizing score estimation. In particular, when the noise level is of constant order, the Wasserstein bound attains the parametric root-$n$ rate up to logarithmic factors for the ambient distribution estimation.  Overall, these results highlight a subtle geometry-noise interaction: ambient noise may either regularize or obscure the problem depending on the learning framework and the noise regime under consideration.

Beyond distribution estimation, our work also shows that diffusion models provide a principled route to learning the geometric information of the structures.  By exploiting the local behavior of the score field at small diffusion times, we establish consistency for estimating local intrinsic dimension and, consistency for estimating the number of strata. This connects generative modeling with geometric structure learning in a way that is not done by classical manifold-learning approaches.  At the same time,  our analysis highlights that dimension estimation is most reliable at regular points, while points near singular intersections remain intrinsically more delicate due to the competition between nearby strata.

There are several interesting directions for future research. These include extending the analysis to stratified spaces under more general geometric assumptions, as well as developing more flexible noise models that allow for varying and anisotropic noise across strata. It would also be of interest to study more general mixture distributions that permit small mixing weights, and to characterize the corresponding minimax rates under suitable beta-min conditions on the mixing weights and separation conditions on the mixture components. Beyond the development of statistical theory, it is also important to investigate the computational and optimization aspects of the proposed models.

\bibliography{references}

@article{Chae,
  author  = {Minwoo Chae and Dongha Kim and Yondai Kim and Lizhen Lin},
  title   = {A Likelihood Approach to Nonparametric Estimation of a Singular Distribution Using Deep Generative Models},
  journal = {Journal of Machine Learning Research},
  year    = {2023},
  volume  = {24},
  number  = {77},
  pages   = {1--42}
}

@inproceedings{TangYang1,
  title={Adaptivity of diffusion models to manifold structures},
  author={Tang, Rong and Yang, Yun},
  booktitle={International Conference on Artificial Intelligence and Statistics},
  pages={1648--1656},
  year={2024},
  organization={PMLR}
}

@inproceedings{Tang2024ConditionalDM,
  title={Conditional Diffusion Models are Minimax-Optimal and Manifold-Adaptive for Conditional Distribution Estimation},
  author={Rong Tang and Lizhen Lin and Yun Yang},
  booktitle={International Conference on Learning Representations},
  year={2025},

}

@article{TangYang2,
  title={Estimating Distributions with Low-dimensional Structures Using Mixtures of Generative Models},
  author={Tang, Rong and Yang, Yun},
  journal={arXiv preprint arXiv:2301.00890},
  year={2023}
}

@article{10.1214/18-AOS1685,
author = {Eddie Aamari and Cl{\'e}ment Levrard},
title = {{Nonasymptotic rates for manifold, tangent space and curvature estimation}},
volume = {47},
journal = {The Annals of Statistics},
number = {1},
publisher = {Institute of Mathematical Statistics},
pages = {177 -- 204},
year = {2019},
doi = {10.1214/18-AOS1685},
URL = {https://doi.org/10.1214/18-AOS1685}
}

@article{aamari,
  title={A theory of stratification learning},
  author={Aamari, Eddie and Berenfeld, Cl{\'e}ment},
  journal={arXiv preprint arXiv:2405.20066},
  year={2024}
}

@article{haro,
  title={Stratification learning: Detecting mixed density and dimensionality in high dimensional point clouds},
  author={Haro, Gloria and Randall, Gregory and Sapiro, Guillermo},
  journal={Advances in Neural Information Processing Systems},
  volume={19},
  year={2006}
}

@article{jacobs1991adaptive,
  title={Adaptive mixtures of local experts},
  author={Jacobs, Robert A and Jordan, Michael I and Nowlan, Steven J and Hinton, Geoffrey E},
  journal={Neural computation},
  volume={3},
  number={1},
  pages={79--87},
  year={1991},
  publisher={MIT Press}
}

@article{jordan1994hierarchical,
  title={Hierarchical mixtures of experts and the EM algorithm},
  author={Jordan, Michael I and Jacobs, Robert A},
  journal={Neural computation},
  volume={6},
  number={2},
  pages={181--214},
  year={1994},
  publisher={MIT Press}
}

@article{alberti2024manifold,
  title={Manifold learning by mixture models of VAEs for inverse problems},
  author={Alberti, Giovanni S and Hertrich, Johannes and Santacesaria, Matteo and Sciutto, Silvia},
  journal={Journal of Machine Learning Research},
  volume={25},
  number={202},
  pages={1--35},
  year={2024}
}

@inproceedings{stanczuk2024diffusion,
  title={Diffusion models encode the intrinsic dimension of data manifolds},
  author={Stanczuk, Jan Pawel and Batzolis, Georgios and Deveney, Teo and Sch{\"o}nlieb, Carola-Bibiane},
  booktitle={Forty-first International Conference on Machine Learning},
  year={2024}
}

@article{brown2022verifying,
  title={Verifying the union of manifolds hypothesis for image data},
  author={Brown, Bradley CA and Caterini, Anthony L and Ross, Brendan Leigh and Cresswell, Jesse C and Loaiza-Ganem, Gabriel},
  journal={arXiv preprint arXiv:2207.02862},
  year={2022}
}

@inproceedings{bendich2012local,
  title={Local homology transfer and stratification learning},
  author={Bendich, Paul and Wang, Bei and Mukherjee, Sayan},
  booktitle={Proceedings of the twenty-third annual ACM-SIAM symposium on Discrete Algorithms},
  pages={1355--1370},
  year={2012},
  organization={SIAM}
}

@inproceedings{oko2023diffusion,
  title={Diffusion models are minimax optimal distribution estimators},
  author={Oko, Kazusato and Akiyama, Shunta and Suzuki, Taiji},
  booktitle={International Conference on Machine Learning},
  pages={26517--26582},
  year={2023},
  organization={PMLR}
}

@article{kingma2013auto,
  title={Auto-encoding variational bayes},
  author={Kingma, Diederik P and Welling, Max},
  journal={arXiv preprint arXiv:1312.6114},
  year={2013}
}

@inproceedings{wu2022generalized,
  title={Generalized clustering and multi-manifold learning with geometric structure preservation},
  author={Wu, Lirong and Liu, Zicheng and Xia, Jun and Zang, Zelin and Li, Siyuan and Li, Stan Z},
  booktitle={Proceedings of the IEEE/CVF winter conference on applications of computer vision},
  pages={139--147},
  year={2022}
}

@article{sule2025learning,
  title={Learning collective variables that preserve transition rates},
  author={Sule, Shashank and Mehta, Arnav and Cameron, Maria K},
  journal={arXiv preprint arXiv:2506.01222},
  year={2025}
}

@article{nakada2020adaptive,
  title={Adaptive approximation and generalization of deep neural network with intrinsic dimensionality},
  author={Nakada, Ryumei and Imaizumi, Masaaki},
  journal={Journal of Machine Learning Research},
  volume={21},
  number={174},
  pages={1--38},
  year={2020}
}

@article{tang2023minimax,
  title={Minimax rate of distribution estimation on unknown submanifolds under adversarial losses},
  author={Tang, Rong and Yang, Yun},
  journal={The Annals of Statistics},
  volume={51},
  number={3},
  pages={1282--1308},
  year={2023},
  publisher={Institute of Mathematical Statistics}
}

@article{azangulov2024convergence,
  title={Convergence of diffusion models under the manifold hypothesis in high-dimensions},
  author={Azangulov, Iskander and Deligiannidis, George and Rousseau, Judith},
  journal={arXiv preprint arXiv:2409.18804},
  year={2024}
}

@article{robinson2025token,
  title={Token embeddings violate the manifold hypothesis},
  author={Robinson, Michael and Dey, Sourya and Chiang, Tony},
  journal={arXiv preprint arXiv:2504.01002},
  year={2025}
}

@article{li2025unraveling,
  title={Unraveling the localized latents: Learning stratified manifold structures in llm embedding space with sparse mixture-of-experts},
  author={Li, Xin and Sarwate, Anand},
  journal={arXiv preprint arXiv:2502.13577},
  year={2025}
}

@article{pidstrigach2022score,
  title={Score-based generative models detect manifolds},
  author={Pidstrigach, Jakiw},
  journal={Advances in Neural Information Processing Systems},
  volume={35},
  pages={35852--35865},
  year={2022}
}

@article{levina2004maximum,
  title={Maximum likelihood estimation of intrinsic dimension},
  author={Levina, Elizaveta and Bickel, Peter},
  journal={Advances in neural information processing systems},
  volume={17},
  year={2004}
}

@article{song2020score,
  title={Score-based generative modeling through stochastic differential equations},
  author={Song, Yang and Sohl-Dickstein, Jascha and Kingma, Diederik P and Kumar, Abhishek and Ermon, Stefano and Poole, Ben},
  journal={arXiv preprint arXiv:2011.13456},
  year={2020}
}

@article{kamkari2024geometric,
  title={A geometric view of data complexity: Efficient local intrinsic dimension estimation with diffusion models},
  author={Kamkari, Hamidreza and Ross, Brendan L and Hosseinzadeh, Rasa and Cresswell, Jesse C and Loaiza-Ganem, Gabriel},
  journal={Advances in Neural Information Processing Systems},
  volume={37},
  pages={38307--38354},
  year={2024}
}

@article{horvat2024gauge,
  title={On gauge freedom, conservativity and intrinsic dimensionality estimation in diffusion models},
  author={Horvat, Christian and Pfister, Jean-Pascal},
  journal={arXiv preprint arXiv:2402.03845},
  year={2024}
}

@article{leung2025convolutions,
  title={On Convolutions, Intrinsic Dimension, and Diffusion Models},
  author={Leung, Kin Kwan and Hosseinzadeh, Rasa and Loaiza-Ganem, Gabriel},
  journal={arXiv preprint arXiv:2506.20705},
  year={2025}
}

@article{de2022convergence,
  title={Convergence of denoising diffusion models under the manifold hypothesis},
  author={De Bortoli, Valentin},
  journal={arXiv preprint arXiv:2208.05314},
  year={2022}
}

@article{ho2020denoising,
  title={Denoising diffusion probabilistic models},
  author={Ho, Jonathan and Jain, Ajay and Abbeel, Pieter},
  journal={Advances in neural information processing systems},
  volume={33},
  pages={6840--6851},
  year={2020}
}

@article{schmidt2019deep,
  title={Deep relu network approximation of functions on a manifold},
  author={Schmidt-Hieber, Johannes},
  journal={arXiv preprint arXiv:1908.00695},
  year={2019}
}

@book{villani2008optimal,
  title={Optimal transport: old and new},
  author={Villani, C{\'e}dric and others},
  volume={338},
  year={2008},
  publisher={Springer}
}

@article{hyvarinen2005estimation,
  title={Estimation of non-normalized statistical models by score matching.},
  author={Hyv{\"a}rinen, Aapo and Dayan, Peter},
  journal={Journal of Machine Learning Research},
  volume={6},
  number={4},
  year={2005}
}

@article{vincent2011connection,
  title={A connection between score matching and denoising autoencoders},
  author={Vincent, Pascal},
  journal={Neural computation},
  volume={23},
  number={7},
  pages={1661--1674},
  year={2011},
  publisher={MIT Press}
}

@article{haussmann1986time,
  title={Time reversal of diffusions},
  author={Haussmann, Ulrich G and Pardoux, Etienne},
  journal={The Annals of Probability},
  pages={1188--1205},
  year={1986},
  publisher={JSTOR}
}

@article{lecun2015deep,
  title={Deep learning},
  author={LeCun, Yann and Bengio, Yoshua and Hinton, Geoffrey},
  journal={nature},
  volume={521},
  number={7553},
  pages={436--444},
  year={2015},
  publisher={Nature Publishing Group UK London}
}

@article{krizhevsky2012imagenet,
  title={Imagenet classification with deep convolutional neural networks},
  author={Krizhevsky, Alex and Sutskever, Ilya and Hinton, Geoffrey E},
  journal={Advances in neural information processing systems},
  volume={25},
  year={2012}
}

@article{vaswani2017attention,
  title={Attention is all you need},
  author={Vaswani, Ashish and Shazeer, Noam and Parmar, Niki and Uszkoreit, Jakob and Jones, Llion and Gomez, Aidan N and Kaiser, {\L}ukasz and Polosukhin, Illia},
  journal={Advances in neural information processing systems},
  volume={30},
  year={2017}
}

@inproceedings{cheng2016wide,
  title={Wide \& deep learning for recommender systems},
  author={Cheng, Heng-Tze and Koc, Levent and Harmsen, Jeremiah and Shaked, Tal and Chandra, Tushar and Aradhye, Hrishi and Anderson, Glen and Corrado, Greg and Chai, Wei and Ispir, Mustafa and others},
  booktitle={Proceedings of the 1st workshop on deep learning for recommender systems},
  pages={7--10},
  year={2016}
}

@article{barron2002universal,
  title={Universal approximation bounds for superpositions of a sigmoidal function},
  author={Barron, Andrew R},
  journal={IEEE Transactions on Information theory},
  volume={39},
  number={3},
  pages={930--945},
  year={2002},
  publisher={IEEE}
}

@article{bengio2013representation,
  title={Representation learning: A review and new perspectives},
  author={Bengio, Yoshua and Courville, Aaron and Vincent, Pascal},
  journal={IEEE transactions on pattern analysis and machine intelligence},
  volume={35},
  number={8},
  pages={1798--1828},
  year={2013},
  publisher={IEEE}
}

@article{belkin2003laplacian,
  title={Laplacian eigenmaps for dimensionality reduction and data representation},
  author={Belkin, Mikhail and Niyogi, Partha},
  journal={Neural computation},
  volume={15},
  number={6},
  pages={1373--1396},
  year={2003},
  publisher={MIT Press}
}

@article{belkin2006manifold,
  title={Manifold regularization: A geometric framework for learning from labeled and unlabeled examples.},
  author={Belkin, Mikhail and Niyogi, Partha and Sindhwani, Vikas},
  journal={Journal of machine learning research},
  volume={7},
  number={11},
  year={2006}
}

@inproceedings{ye2020mixtures,
  title={Mixtures of variational autoencoders},
  author={Ye, Fei and Bors, Adrian G},
  booktitle={2020 Tenth International Conference on Image Processing Theory, Tools and Applications (IPTA)},
  pages={1--6},
  year={2020},
  organization={IEEE}
}

@article{kruger2018set,
  title={Set regularities and feasibility problems},
  author={Kruger, Alexander Y and Luke, D Russell and Thao, Nguyen H},
  journal={Mathematical Programming},
  volume={168},
  number={1},
  pages={279--311},
  year={2018},
  publisher={Springer}
}

@article{federer1959curvature,
  title={Curvature measures},
  author={Federer, Herbert},
  journal={Transactions of the American Mathematical Society},
  volume={93},
  number={3},
  pages={418--491},
  year={1959},
  publisher={JSTOR}
}

@book{ghosal2017fundamentals,
  title={Fundamentals of nonparametric Bayesian inference},
  author={Ghosal, Subhashis and Van der Vaart, Aad W},
  volume={44},
  year={2017},
  publisher={Cambridge University Press}
}

@article{wong1995probability,
  title={Probability inequalities for likelihood ratios and convergence rates of sieve MLEs},
  author={Wong, Wing Hung and Shen, Xiaotong},
  journal={The Annals of Statistics},
  pages={339--362},
  year={1995},
  publisher={JSTOR}
}

@article{wang2024remoe,
  title={Remoe: Fully differentiable mixture-of-experts with relu routing},
  author={Wang, Ziteng and Zhu, Jun and Chen, Jianfei},
  journal={arXiv preprint arXiv:2412.14711},
  year={2024}
}

@article{liu2025improving,
  title={Improving the euclidean diffusion generation of manifold data by mitigating score function singularity},
  author={Liu, Zichen and Zhang, Wei and Li, Tiejun},
  journal={arXiv preprint arXiv:2505.09922},
  year={2025}
}

@book{gine2021mathematical,
  title={Mathematical foundations of infinite-dimensional statistical models},
  author={Gin{\'e}, Evarist and Nickl, Richard},
  year={2021},
  publisher={Cambridge university press}
}

@article{hashemian2013modeling,
  title={Modeling and enhanced sampling of molecular systems with smooth and nonlinear data-driven collective variables},
  author={Hashemian, Behrooz and Mill{\'a}n, Daniel and Arroyo, Marino},
  journal={The Journal of chemical physics},
  volume={139},
  number={21},
  year={2013},
  publisher={AIP Publishing}
}

@article{pope2021intrinsic,
  title={The intrinsic dimension of images and its impact on learning},
  author={Pope, Phillip and Zhu, Chen and Abdelkader, Ahmed and Goldblum, Micah and Goldstein, Tom},
  journal={arXiv preprint arXiv:2104.08894},
  year={2021}
}

@inproceedings{chen2023score,
  title={Score approximation, estimation and distribution recovery of diffusion models on low-dimensional data},
  author={Chen, Minshuo and Huang, Kaixuan and Zhao, Tuo and Wang, Mengdi},
  booktitle={International Conference on Machine Learning},
  pages={4672--4712},
  year={2023},
  organization={PMLR}
}

@article{potaptchik2024linear,
  title={Linear convergence of diffusion models under the manifold hypothesis},
  author={Potaptchik, Peter and Azangulov, Iskander and Deligiannidis, George},
  journal={arXiv preprint arXiv:2410.09046},
  year={2024}
}

@article{burgess2018understanding,
  title={Understanding disentangling in $\beta$-VAE},
  author={Burgess, Christopher P and Higgins, Irina and Pal, Arka and Matthey, Loic and Watters, Nick and Desjardins, Guillaume and Lerchner, Alexander},
  journal={arXiv preprint arXiv:1804.03599},
  year={2018}
}
\appendix
\section{Proofs for Section 3}
\subsection{Proof of Lemma~\ref{lm:approx}}
Suppose $\|f^*\|_\infty \leq V$ and takes the form $$f^*(w,z)=\sum_{k=1}^K \sum_{j=1}^{J_k} h^*_{k,j}(w)f^*_{k,j}(z).$$

For each $k\in \{1,...,K\}$ and $j=1,...,J_k$, Lemma 5 in \cite{Chae} implies there exists a neural network $f_{k,j} \in \m F^{(k)}=\Phi(L^{(k)},W^{(k)},R^{(k)},1,V^{(k)})$ with $L^{(k)} =O(\log \delta^{-1})$, $\|W^{(k)}\|_\infty =O(\delta^{-d_k/(\alpha_k+1)})$, $R^{(k)}=O(\delta^{-d_k/(\alpha_k+1)}\log \delta^{-1})$ satisfying $\|f_{k,j}-f^*_{k,j}\|_\infty \leq \frac{\delta}{2KJ_k}$.\\

\noindent Although one can approximate $\mathbf{1}_{[a,b)}$ with a shallow ReLU network with bounded parameters, it would require a large width and sparsity to create a steep slope, which dominates order of the overall error. However, we can use depth in order to stay within bounded weights:

\begin{lemma}\label{lm:gate_approx}
    Fix $0\leq a<b\leq 1$, $1\leq p < \infty$ and $0<\xi<1$. Then there exists $h \in  \Phi(L,W,R,1,1)$ with $L=O(\log \xi^{-1})$, $\|W\|_\infty=O(1)$ and $R=O(\log \xi^{-1})$ such that
    \begin{enumerate}
        \item $h(x)=1$ for $x \in [a,b]$
        \item $h(x)=0$ for $x \leq a-\xi$ and $x \geq b + \xi$.
        \item $\|h-\mathbf{1}_{[a,b]}\|_{L^p} \leq (2\xi)^{1/p}$
    \end{enumerate}
\end{lemma}

From Lemma~\ref{lm:gate_approx} with $\xi \asymp \delta^2$, for each $k=1,...,K$ and $j=1,...,J_k$, there exists $h_{k,j}$ with $\widetilde{L}=O(\log \delta^{-1})$, $\|\widetilde{W}\|_\infty = O(1)$ and $\widetilde{R} = O(\log \delta^{-1})$ such that
$\|h_{k,j}-h^*_{k,j}\|_2 \leq \frac{\delta}{2VKJ_k}.$\\

Let $f(w,z)=\sum_{k=1}^K\sum_{j=1}^{J_k} h_{k,j}(w)f_{k,j}(z)$. Note that $h^*_{k,j}$ have mutually disjoint support for each $k,j$, and so $\|f_{k,j}^*\|_\infty \leq V$. It then follows that
\begin{align*}
    \|f-f^*\|_2 &\leq \sum_{k=1}^K \sum_{j=1}^{J_k} \|h_{k,j}f_{k,j} - h^*_{k,j}f^*_{k,j}\|_2\\
    &\leq \sum_{k=1}^K \sum_{j=1}^{J_k} \|h_{k,j}f_{k,j} - h_{k,j}f^*_{k,j} +h_{k,j}f^*_{k,j} - h^*_{k,j}f^*_{k,j}\|_2 \\
    & \leq \sum_{k=1}^K \sum_{j=1}^{J_k} \left(\|h_{k,j}\|_\infty\|f_{k,j}-f^*_{k,j}\|_2 + \|f^*_{k,j}\|_\infty\|h_{k,j}-h^*_{k,j}\|_2\right)\\
    & \leq \sum_{k=1}^K \sum_{j=1}^{J_k}\left( 1\cdot \|f_{k,j} - f^*_{k,j}\|_\infty + \|f^*_{k,j}\|_\infty \|h_{k,j}-h^*_{k,j}\|_2\right)\\
    & \leq \sum_{k=1}^K \sum_{j=1}^{J_k} \left(\frac{\delta}{2KJ_k} + V\frac{\delta}{2VKJ_k}\right) = \delta.
\end{align*}

$\hfill{\Box}$

\subsection{Proof of Lemma~\ref{lm:covering}}
Recall the following Lemma from Chae et. al:

\begin{lemma}\label{lm:chae_cover}
    Let $\mathcal{F}$ be a class of functions from $\mathcal{Z}$ to $\mathbb{R}^D$ such that $\|f\|_\infty \leq V$ for every $f \in \mathcal{F}$. Let $\mathcal{P}=\{P_{f,\sigma}: f\in \mathcal{F}, \sigma \in[\sigma_{min},\sigma_{max}]\}$ with $\sigma_{min}\leq 1$. There, there exists constants $c=c(D,V,\sigma_{max})$, $c'=c'(D,V,\sigma_{max})$ and $\delta_*=\delta_*(D)$ such that $$\log N_{[]}(\delta,\mathcal{P},d_H) \leq \log N(c\sigma_{min}^{D+3}\delta^4,\mathcal{F},\|\cdot\|_\infty)+\log \Big(\frac{c'}{\sigma_{min}^{D+2}\delta^4}\Big)$$ for every $\delta \in (0,\delta_*].$
\end{lemma}

Note that this lemma holds for any class of bounded functions $\mathcal{F}$ from $\mathcal{Z}\to \mathbb{R}^D$. Given a stratified space consisting of $K$ manifolds and a decomposition of each manifold into $J_k$ charts, a function $f$ in the mixture-of-experts sieve $\m G^{K,\mathbf{J}}$ is of the form $f(w,z)=\sum_{k=1}^K \sum_{j=1}^{J_k}h_{k,j}(w)f_{k,j}(z)$, where $h_{k,j} \in \mathcal{H}$ and $f_{k,j} \in \mathcal{F}^{(k)}$ for each $k=1,...,K$ and $j=1,...,J_k$. Since $\|h_{k,j}\|_\infty \leq 1$ and $\|f_{k,j}\|_\infty \leq V$, for any $\tilde{f}(w,z)=\sum_{k=1}^K \sum_{j=1}^{J_k} \tilde{h}_{k,j}(w)\tilde{f}_{k,j}(z) \in \m G^{K,\mathbf{J}}$,
\begin{align*}
    \|f-\tilde{f}\|_\infty &=\left\|\sum_{k=1}^K \sum_{j=1}^{J_k} (h_{k,j}f_{k,j}-\tilde{h}_{k,j}\tilde{f}_{k,j})\right\|_\infty\\
    &\leq \sum_{k=1}^K \sum_{j=1}^{J_k}\|h_{k,j} \|_\infty\|f_{k,j}-\tilde{f}_{k,j}\|_\infty + \|\tilde{f}_{k,j}\|_\infty\|h_{k,j}-\tilde{h}_{k,j}\|_\infty\\
    & \leq \sum_{k=1}^K \sum_{j=1}^{J_k} \|f_{k,j}-\tilde{f}_{k,j}\|_\infty + V\|h_{k,j}-\tilde{h}_{k,j}\|_\infty
\end{align*} 
Let $M=\sum_{k=1}^K J_k$. Taking an $\frac{\epsilon}{2M}$-cover of $\m F^{(k)}$ and $\frac{\epsilon}{2VM}$-cover of $\m H$ for each $k=1,...,K$ and $j=1,...,J_k$ gives us an $\epsilon$-cover for $\m G^{K,\mathbf{J}}$, and hence
\begin{align*}
    \log N(\epsilon, \m G^{K,\mathbf{J}},\|\cdot\|_\infty) &\leq \sum_{k=1}^K \sum_{j=1}^{J_k}\log N\left(\frac{\epsilon}{2M}, \m F^{(k)},\|\cdot\|_\infty\right)+ M \log N\left(\frac{\epsilon}{2VM},\m H,\|\cdot\|_\infty\right)\\
    & =\sum_{k=1}^K J_k \log N\left(\frac{\epsilon}{2M}, \m F^{(k)}, \|\cdot \|_\infty\right) +M \log N\left( \frac{\epsilon}{2VM},\m H, \|\cdot \|_\infty\right)
\end{align*}
Applying Lemma~\ref{lm:chae_cover} and setting $\epsilon=c\sigma_{\min}^{D+3}\delta^4$ gives the claim after absorbing constants.
$\hfill{\Box}$\\

\subsection{Proof of Theorem~\ref{th:hellinger}}
(1) Follows from Lemma~\ref{lm:covering} and (2) follows from Lemma~\ref{lm:approx} with $\delta=\delta_{\operatorname{app}}$.

For (3), we follow the idea in the proofs of Theorem 3 and Corollary 6 in \cite{Chae} and provide their proofs for completeness. To this end, choose constants $c_1,c_2,c_3,c_4$ with $c_1=1/3$ and $c_3>2$ as in Theorem 1 in \cite{wong1995probability} and constants $c,c'$ in their Lemma 1.\\

Let $\hat{R}=\widetilde{R} + \sum_{k=1}^K J_kR^{(k)}$ be the total sparsity. For $\delta \in (0,c_3\delta_*]$,
$$\log N_{[]}\left(\frac{\delta}{c_3},\m P,d_H\right) \leq 4(\hat{R}+1)\log \delta^{-1} + \hat{R}A + (D+3)(\hat{R}+1)\log \sigma_{\min}^{-1}+c_5\hat{R}$$
for some $c_5=c_5(c,c',c_3)$ by (1). Thus, for all $\epsilon \leq \sqrt{2}\delta_* \leq c_3\delta_*/\sqrt{2}$,
\begin{align*}
    &\int_{\epsilon^2/2^8}^{\sqrt{2}\epsilon} \sqrt{\log N_{[]}(\delta/c_3, \m P, d_H)}\,d\delta \\
    &\leq \sqrt{2}\epsilon\sqrt{\hat{R}A+(D+3)(\hat{R}+1)\log \sigma_{\min}^{-1}+c_5\hat{R}}+\sqrt{2}\epsilon\sqrt{4(\hat{R}+1)}\sqrt{\log \frac{2^8}{\epsilon^2}}
\end{align*}

This can be bounded by $c_4n^{1/2}\epsilon_n^2$, where we take $\epsilon=\epsilon_n=c_6\sqrt{n^{-1}\hat{R}(A+\log(n/\sigma_{\min}))}$, for a large enough constant $c_6=c_6(c_4,c_5,D)$. Then by Example B.12 in \cite{ghosal2017fundamentals},
\begin{align*}
    \text{KL}(p_*,p_{f,\sigma_*}) &\leq \int \text{KL}(N(f^*(z),\sigma_*^2),N(f(z),\sigma_*^2))dP_Z(z)\\
    & = \int \frac{\|f^*(z)-f(z)\|_2^2}{2\sigma_*^2}dP_Z(z) \leq \frac{\delta_{\text{app}}^2}{2\sigma_*^2}.
\end{align*}
We also have that
$$\int \left(\log \frac{\phi_\sigma(x)}{\phi_\sigma(x-y)} \right)^2\phi_\sigma(x)dx =\int \frac{\|y\|_2^4 +4|x^Ty|^2}{4\sigma^2}\phi_\sigma(x)dx \leq \frac{\|y\|^4_2}{4\sigma^2} + \|y\|_2^2\int \frac{\|x\|^2}{\sigma^2} \phi_\sigma(x)dx.$$
Thus, by Example B.12, (B.17) and Exercise B.8 in \cite{ghosal2017fundamentals},
\begin{align*}
    \int \left(\log \frac{p_*(x)}{p_{f,\sigma_*}(x)}\right)^2dP_*(x) &\leq \iint \left( \log \frac{\phi_{\sigma_*}(x-f_*(z)}{\phi_{\sigma_*}(x-f(z))}\right)^2\phi_\sigma(x-f_*(z))dxdP_z + 4\text{KL}(p_*,p_{f,\sigma_*})\\
    &\leq \frac{\delta_{\text{app}}^4}{4\sigma_*^2} + \delta_{\text{app}}^2\int \frac{\|x\|^2_2}{\sigma_*^2}\phi_{\sigma_*}(x)dx + \frac{2\delta_{\text{app}}^2}{\sigma_*^2} \leq c_7\frac{\delta_{\text{app}}^2}{\sigma_*^2},
\end{align*}
for some constant $c_7$. Let $\delta_n=\frac{\delta_{\text{app}}^2}{2\sigma_*^2}$ and $\tau_n=c_7\frac{\delta^2_{\text{app}}}{\sigma_*^2}$. Then by Theorem 4 in \cite{wong1995probability} with $\epsilon_n^*=\epsilon_n \vee \sqrt{12\delta_n}$,
$$P_*(d_H(\hat{p},p_*) > \epsilon_n^*) \leq 5e^{-c_2n\epsilon_n^{*2}} + \frac{12\tau_n}{n\epsilon_n^{*2}} \leq 5e^{-c_2n\epsilon_n^{*2}} + \frac{\tau_n}{n\delta_n} = 5e^{-c_2n\epsilon_n^{*2}} + \frac{2c_7}{n}.$$ 
The Theorem follows after redefining constants. $\hfill{\Box}$

\subsection{Proof of Theorem~\ref{th:wasserstein}}
We will make use of the following lemma to separate the contributions from the partitioned space:

\begin{lemma}\label{lm:wasser}
    Let $Y \subseteq \mathbb{R}^D$ be a compact set, let $R=\text{diam}(Y)$, and let $\sum_{k=1}^K b_kQ_k$ be a mixture distribution on $Y$. Then for any mixture distribution $\sum_{k=1}^K a_k\tilde{Q}_k$ that is also supported on $Y$, $$W_1(\sum_{k=1}^K a_k\tilde{Q}_k, \sum_{k=1}^K b_kQ_k) \leq \sum_{k=1}^K b_k W_1(\tilde{Q}_k,Q_k) + \frac{R}{2}\sum_{k=1}^K |a_k-b_k|.$$
\end{lemma}

We proceed to assume that $\epsilon$ and $\sigma_*\sqrt{\log \epsilon^{-1}}$ are sufficiently small; otherwise we just need a suitably large enough constant $C$ so that the inequality will hold. Recall from the proof of Theorem 7 in \cite{Chae} that for the choice $t_*=(2D\sigma_*^2\log(D/\epsilon))^{1/2}$, we have that
$$\int_{\|x\|_2 \geq t_*} \phi_{\sigma_*}(x)dx \leq \epsilon.$$
Throughout the proof, we fix the following parameters:
$$\eta_1=2t_*,\quad \eta_0=4t_*,\quad \rho=8t_*,$$
and assume that $\rho < \min\{r_*,1,\rho_1\}$. Suppose $\sigma \leq C_0\sigma_*\sqrt{\log(D/\epsilon)}$ for some constant $C_0$. This implies for $Z\sim \m N(0,\sigma^2 I_D),Z_* \sim \m N(0,\sigma_*^2I_D)$ that there exists a constant $c_0>0$ such that
$$\mb{P}(\|Z\|<t_*)\geq c_0,\quad \mb{P}(\|Z_*\|\geq t_*)\leq \epsilon.$$

To handle intersections, we define for each manifold $M_k$: $$A_k=\left\{x \in M_k \mid \text{dist}\left(x,\bigcup_{\ell \neq k} M_\ell \right) \geq 2\rho \right\},\quad A=\bigcup_{k=1}^K A_k,\quad I_{\mathrm{sing}}=\m S\setminus A.$$
\begin{remark}
    If two manifolds are sufficiently close without intersection, this excision would remove part of the manifold away from the intersections. However, we will assume enough of a margin and let $\rho$ be sufficiently small so that this excision only occurs at around singularities.
\end{remark}

Recall that we assume each stratum is compact with regularity assumptions along with intersections being transversal in Assumptions 1 and 2. This implies the local metric inequality $\text{dist}(x,M_k \cap (\bigcup_{\ell \neq k} M_\ell)) \lesssim \text{dist}(x, \bigcup_{\ell \neq k} M_\ell)$ (see Theorem 1 in \citep{kruger2018set}). By compactness, this local bound holds globally. We also have the reverse bound induced by the inclusion $M_k \cap (\bigcup_{\ell \neq k} M_\ell) \subset \bigcup_{\ell \neq k} M_\ell$, which implies $\text{dist}(x,M_k \cap (\bigcup_{\ell \neq k} M_\ell)) \asymp \text{dist}(x, \bigcup_{\ell \neq k} M_\ell)$. Hence, $A_k$ is (up to a constant depending on the angle $\theta_0$ between strata at intersections) the complement of a $2\rho$-tubular neighborhood of the singularity $\m S_{\text{sing}}$ in $M_k$. Combining this with $\text{reach}(M_k) = r_k > \rho$ and standard properties of sets with positive reach \citep{federer1959curvature}, we have that $\text{reach}(A_k) \geq \rho$. Since the sets $A_k$ are separated by $\rho$, we conclude that 
$$\text{reach}(A) \geq \min \left\{ \min_k \text{reach}(A_k),\rho\right\} \geq \rho.$$

Since $f(\m Z)$ need not lie entirely in $\m S$, we partition $\mb{R}^D$ as follows. Let
$$U=\{x \in \mb{R}^D:\operatorname{dist}(x,A)<\eta_0\},\quad U^c=\mb{R}^D \setminus U.$$
Also, let $$a_1=Q_f(U),\quad a_2=Q_f(U^c),\quad b_1=Q_*(A),\quad b_2= Q_*(I_{\mathrm{sing}}),$$
and the normalized restrictions
$$Q_f^{(1)}=\frac{1}{a_1}Q_f|_U,\quad Q_f^{(2)}=\frac{1}{a_2}Q_f|_{U^c},\quad Q_*^{(1)}=\frac{1}{b_1}Q_*|_A, \quad Q_*^{(2)}=\frac{1}{b_2}Q_*|_{I_\mathrm{sing}}.$$
This implies that $a_1+a_2=Q_f(\mb{R}^D)=1$ and $b_1+b_2=Q_*(\m S)=1$. Since $\|f\|_\infty,\|f^*\|_\infty \leq V$, all measures are supported in $Y=\overline{B(0,V)}$. So by Lemma~\ref{lm:wasser} with $R=\operatorname{diam}(Y)\leq 2V$, we have that
$$W_1\left(Q_f,Q_*\right) \leq \sum_{i=1}^2 b_iW_1(Q_f^{(i)},Q_*^{(i)})+ \frac{R}{2}\sum_{i=1}^2 |a_i-b_i|.$$

\paragraph{Bounding $b_2$ and $b_2W_1(Q_f^{(2)},Q_*^{(2)})$.} Recall that the densities on each manifold are bounded from above and that intersections have positive codimension. Since the set $I_{\mathrm{sing}}$ is essentially a union of geodesic tubes of radius $\rho$ around lower-dimensional subsets, its volume within $M_k$ is $O(\rho^d)$ for some $d \geq 1$. Thus, there exists a constant $c$ such that
$$b_2=\sum_{k=1}^K \omega_k\int_{I_{\mathrm{sing}}\cap M_k} q_k \, d\text{vol}_{M_k}(x) \leq c \sum_{k=1}^K\text{vol}_{M_k}(I_{\mathrm{sing}} \cap M_k) \lesssim \rho$$
As $Q_f^{(2)}$ and $Q_*^{(2)}$ are supported in $Y$, we have $W_1(Q_f^{(2)},Q_*^{(2)})\leq 2V$, and so $$b_2W_1(Q_f^{(2)},Q_*^{(2)}) \leq 2V b_2 \lesssim 2V\cdot\rho \lesssim \sigma_* \sqrt{\log \epsilon^{-1}}.$$

\paragraph{Bounding $a_2$.} Let $\m S^{\eta_1}=\{x\in \mb{R}^D:\operatorname{dist}(x,\m S) < \eta_1\}$ and consider the splitting
$$a_2=\underbrace{Q_f(U^c \cap \m S^{\eta_1})}_{\tilde{a}_2} + \underbrace{Q_f(\mb{R}^D \setminus \m S^{\eta_1})}_{\hat{a}_2},$$
where $\tilde{a}_2$ is the proportion of points closer to $\m S$, and $\hat{a}_2$ is the proportion farther away.

If $X \in \mb{R}^D \setminus S^{\eta_1}$ and $\|Z\| <t_*$, we have $\operatorname{dist}(X+Z,\m S)\geq \operatorname{dist}(X,\m S)-\|Z\|>t_*=\eta_1/2$, so $X+Z \not\in \m S^{\eta_1/2}$. Hence
$$P_{f,\sigma}(\mb{R}^D \setminus \m S^{\eta_1/2}) \geq \int_{\mb{R}^D \setminus \m S^{\eta_1}} \mb{P}(\|Z\|<t_*)Q_f(dx)=\hat{a}_2\mb{P}(\|Z\|<t_*).$$

Conversely, since $Q_*$ is supported on $\m S$,
$$P_*(\mb{R}^D \setminus \m S^{\eta_1/2}) = \int_{\m S} \mb{P}(X+Z_* \notin \m S^{\eta_1/2})Q_*(dx) \leq \mb{P}(\|Z_*\|\geq t_*)\leq \epsilon.$$
With our bound on Hellinger distance, we have that $|P_{f,\sigma}(\cdot)-P_*(\cdot)|\leq d_{TV}(P_{f,\sigma},P_*) \leq \sqrt{2}\epsilon$, which implies that
$$\hat{a}_2\mb{P}(\|Z\|<t_*)\leq \epsilon + \sqrt{2}\epsilon.$$
As $\mb{P}(\|Z\|<t_*)\geq c_0>0$, we have
$$\hat{a}_2 \lesssim \epsilon.$$

Now for the mass near $\m S$, we define the set $$V=\{ x \in \mb{R}^D: \operatorname{dist}(x,I_{\mathrm{sing}})<3t_*\}.$$
Consider $y \in U^c \cap \m S^{\eta_1}$ and let $x \in \m S$ be such that $\|x-y\|<\eta_1$. If $x \in A$, then $\operatorname{dist}(x,A)\leq \|x-y\| <\eta_1<\eta_0$, so $y \not\in U^c$, a contradiction. Hence $x \in I_{\mathrm{sing}}$ and $\operatorname{dist}(y,I_{\mathrm{sing}})<2t_*$, so that $y \in V$. For any $X \in U^c \cap \m S^{\eta_1}$ and any $Z$, we hence have
$$\|Z\|< t_* \implies \operatorname{dist}(X+Z, I_{\mathrm{sing}})\leq \operatorname{dist}(X,I_{\mathrm{sing}})+\|Z\|<3t_*,$$
which implies that $X + Z \in V$. Therefore,
$$P_{f,\sigma}(V) \geq \int_{U^c \cap \m S^{\eta_1}} \mb{P}(\|Z\|<t_*)Q_f(dx)=\tilde a_2\mb{P}(\|Z\|<t_*).$$
Define the intrinsic set
$$I=\{x \in \m S:\operatorname{dist}(x,I_{\mathrm{sing}})<4t_*\}$$
and note that $I$ is a $O(\rho)$-tubular neighborhood of $I_{\mathrm{sing}}$ in $\m S$, and hence the same volume argument from the bound on $b_2$ applies, i.e. $Q_*(I) \lesssim \rho$. If we take $X \sim Q_*$ and $X+Z_* \in V$ with $\|Z_*\|<t_*$, then $\operatorname{dist}(X,I_{\mathrm{sing}})\leq \operatorname{dist}(X+Z_*,I_{\mathrm{sing}}) + \|Z_*\|<4t_*$, which means that $X \in I$. Hence
$$P_*(V)=\int_{I} \mb{P}(X+Z_* \in V)Q_*(dx)+\int_{\m S \setminus I} \mb{P}(X+Z_* \in V)Q_*(dx)\leq Q_*(I) + \mb{P}(\|Z_*\|\geq t_*)\lesssim \rho + \epsilon.$$
With the Hellinger bound, we have that $P_{f,\sigma}(V) \leq P_*(V) + \sqrt{2}\epsilon$, and so
$$\tilde{a}_2\mb{P}(\|Z\|<t_*)\lesssim \rho +\epsilon,$$
and $\mb{P}(\|Z\|<t_*)\geq 1/2$ yields $\tilde{a}_2\lesssim \rho + \epsilon.$
Hence $$a_2=\tilde{a}_2 + \hat a_2 \lesssim \rho+\epsilon \lesssim \sigma_* \sqrt{\log \epsilon^{-1}} + \epsilon.$$

\paragraph{Bounding $|a_1-b_1|$ and $|a_2-b_2|$.} Note that $a_1+a_2=b_1+b_2=1$, so $$|a_1-b_1|=|a_2-b_2|\leq a_2+b_2 \lesssim \rho + \epsilon.$$
Hence $$\frac{R}{2}(|a_1-b_1|+|a_2-b_2|)\lesssim \epsilon + \sigma_*\sqrt{\log \epsilon^{-1}}.$$

\paragraph{Bounding $W_1(Q_f^{(1)},Q_*^{(1)})$.} The assumption that $\rho\leq \rho_1$ implies that $b_1 \geq c$, and so
$$a_1 \geq b_1-|a_1-b_1|\geq c-C_0(\epsilon+\sigma_*\sqrt{\log \epsilon^{-1}}) \geq c/2$$
for some constant $C_0 >0$ as $\sigma_*\sqrt{\log \epsilon^{-1}}$ is small. This allows us to control the $L^1$ distance on the restricted densities. Let $p_{f,\sigma}^{(1)}=Q_f^{(1)}*\phi_{\sigma}$ and $p_*^{(1)}=Q_*^{(1)}*\phi_{\sigma_*}$ so that
$$p_{f,\sigma}=a_1p_{f,\sigma}^{(1)}+a_2(Q_f^{(2)}*\phi_{\sigma}),\quad p_*=b_1p_*^{(1)}+b_2(Q_*^{(2)}*\phi_{\sigma_*}).$$
Rearranging gives us that
$$a_1p_{f,\sigma}^{(1)}-b_1p_*^{(1)}=(p_{f,\sigma}-p_*)-(a_2Q_f^{(2)}*\phi_{\sigma} - b_2Q_*^{(2)}*\phi_{\sigma_*}).$$
Then
$$\|a_1p_{f,\sigma}^{(1)}-b_1p_*^{(1)}\|_1 \leq \|p_{f,\sigma}-p_*\|_1 +a_2+b_2.$$
Using the relation $p_{f,\sigma}^{(1)}-p_*^{(1)}=\frac{1}{a_1}(a_1p_{f,\sigma}^{(1)}-b_1p_*^{(1)})+\frac{b_1-a_1}{a_1}p_*^{(1)}$ implies that
$$\|p_{f,\sigma}^{(1)}-p_*^{(1)}\|_1 \leq \frac{1}{a_1}(\|p_{f,\sigma}-p_*\|_1 + a_2+b_2+|a_1-b_1|) \leq \frac{2}{c}(2\sqrt{2}d_H(p_{f,\sigma},p_*)+a_2+b_2+|a_1-b_1|).$$
By our assumption on the Hellinger distance and bounds in the previous steps, we have that each of the terms is $\lesssim \epsilon+\sigma_*\sqrt{\log \epsilon^{-1}}$. Since $Q_*^{(1)}$ is supported on the set $A$ with $\operatorname{reach}(A)\geq 4t_*$, we can apply the argument in Theorem 7 in \cite{Chae} to get $\sigma \leq 2t_*$ and
\begin{align*}
    W_1(Q_f^{(1)},Q_*^{(1)}) &\lesssim \sigma +\|p_{f,\sigma}^{(1)}-p_*^{(1)}\|_1 + \sigma_*\\
    &\lesssim t_* + \|p_{f,\sigma}^{(1)}-p_*^{(1)}\|_1 \lesssim \epsilon + \sigma_*\sqrt{\log \epsilon^{-1}}.
\end{align*}
Thus $b_1W_1(Q_f^{(1)},Q_*^{(1)})\lesssim \epsilon+\sigma_*\sqrt{\log \epsilon^{-1}}$.

Putting all bounds together gives us that
$$W_1(Q_f,Q_*) \lesssim \epsilon + \sigma_*\sqrt{\log \epsilon^{-1}}.$$

 $\hfill{\Box}$

\subsection{Proof of Lemma~\ref{lm:gate_approx}}
Let $m=\lceil \log_2 2\xi^{-p}\rceil$ and $\eta=2^{-m}$. Consider the function
$$g_{c,\eta}(x)=\operatorname{ReLU}(x-c)-\operatorname{ReLU}(x-c-\eta),$$
and define $D(u,v)=(\operatorname{ReLU}(u+v),\operatorname{ReLU}(u+v))$. In particular, for $u=v\geq 0$, this gives $D(u,v)=(2u,2u)$. Set $(u_0,v_0)=(g_{c,\eta}(x),g_{c,\eta}(x))$ and
$$(u_i,v_i)=D(u_{i-1},v_{i-1}).$$
Then $u_m=v_m=2^mg_{c,\eta}(x)$. Let $f_{c,\eta}(x)=u_m=2^mg_{c,\eta}(x)=\frac{1}{\eta}g_{c,\eta}(x)$. Then
$$f_{c,\eta}(x)=\begin{cases}
    0,& x \leq c;\\
    \frac{x-c}{\eta},& c<x < c+\eta;\\
    1,& x \geq c+\eta.
\end{cases}
$$
Set $\Phi(x)=f_{a-\eta,\eta}(x)-f_{b,\eta}(x)$. Then properties 1, 2, and 3 of the lemma are satisfied with the specified layers, width and sparsity. $\hfill{\Box}$

\subsection{Proof of Lemma~\ref{lm:wasser}}
By the triangle inequality,
$$W_1(\sum_{k=1}^K a_k\tilde{Q}_k, \sum_{k=1}^K b_kQ_k) \leq W_1(\sum_{k=1}^K a_k\tilde{Q}_k, \sum_{k=1}^K b_k\tilde{Q}_k)+W_1(\sum_{k=1}^K b_k\tilde{Q}_k, \sum_{k=1}^K b_kQ_k).$$

The second term on the right captures the error in the distribution estimation and can be bounded as follows. Take an optimal coupling of $\tilde{Q}_k$ and $Q_k$ for each $k$, say $\gamma_k$. Define a new measure $\gamma = \sum_{k=1}^K b_k\gamma_k$. This gives us marginals of $\sum_{k=1}^K b_k\tilde{Q}_k$ and $\sum_{k=1}^K b_kQ_k$, so this is a coupling of the two mixtures. Thus, an optimal coupling would satisfy $$W_1(\sum_k b_k\tilde{Q}_k,\sum_k b_kQ_k)\leq \int |x-y|d\gamma(x,y)=\sum_k b_k \int |x-y|d\gamma_k(x,y)=\sum_k b_kW_1(\tilde{Q}_k,Q_k).$$

The first term on the right of the beginning inequality captures the error in the estimation of the mixture weights. By (\cite{villani2008optimal} Theorem 6.15), $W_1(\mu,\nu)\leq R||\mu-\nu||_{TV}$, where $R=diam(\mathcal{Y}_*)$ and $||\sum_{k=1}^K a_k\tilde{Q}_k - \sum_{k=1}^K b_k\tilde{Q}_k||_{TV} \leq \frac{1}{2}\sum_{k=1}^K|a_k-b_k|$. Thus
$$W_1(\sum_{k=1}^K a_k\tilde{Q}_k, \sum_{k=1}^K b_k\tilde{Q}_k) \leq \frac{R}{2}\sum_{k=1}^K |a_k-b_k|$$
Adding these bounds together gives the claim. $\hfill{\Box}$\\

\section{Proofs for Section 4}
For the sake of simplicity  and without loss of generality, in the following analysis,  we assume $\m M=\cup_{k=1}^K \m M_k\subset \mb B_1(0_D)$. Recall that $\sigma_{t}=\sqrt{1-\exp(-2\int_0^{ t} \beta_s\,\dd s)}$ and $m_t=\sqrt{1-\sigma_t^2} $.  Then there exists a constant $C$ so that  $\frac{1}{c} \sqrt{t\wedge 1}\leq \sigma_t\leq c\sqrt{t\wedge 1}$ and  $  \frac{1}{c}(t\wedge 1)\leq 1-m_t\leq   c(t\wedge 1)$. For any $x\in \mb R^D$, we have 
\begin{equation*}
    \begin{aligned}
        &p_t(x)\propto\mb{E}_{y\sim P_*}\Big[\exp\big(-\frac{\|x-m_ty\|^2}{2\sigma_t^2}\big)\Big]\\
        &\propto \sum_{k=1}^K \omega_k\cdot\mb{E}_{y\sim Q_k}\Big[\int_{\mb R^D}\exp\big(-\frac{\|x-m_tz\|^2}{2\sigma_t^2}\big)\exp\big(-\frac{\|y-z\|^2}{2\sigma_*^2}\big)\,\dd z\Big]\\
        &\propto  \sum_{k=1}^K \omega_k\cdot\mb{E}_{y\sim Q_k}\Bigg[\int_{\mb R^D}\exp\bigg(-\Big(\frac{m_t^2}{2\sigma_t^2}+\frac{1}{2\sigma_*^2}\Big)\Big\|z-\frac{\frac{m_t}{\sigma_t^2}x+\frac{y}{\sigma_*^2}}{\frac{m_t^2}{\sigma_t^2}+\frac{1}{\sigma_*^2}}\Big\|^2\bigg)\,\dd z\cdot\exp\bigg(-\frac{\|x-m_ty\|^2}{2(m_t^2\sigma_*^2+\sigma_t^2)}\bigg)\Bigg]\\
         &\propto  \sum_{k=1}^K \omega_k\cdot\mb{E}_{y\sim Q_k}\Bigg[\exp\bigg(-\frac{\|x-m_ty\|^2}{2(m_t^2\sigma_*^2+\sigma_t^2)}\bigg)\Bigg].
    \end{aligned}
\end{equation*}
So let $\wt \sigma_t=\sqrt{m_t^2\sigma_*^2+\sigma_t^2}$, we can write 
\begin{equation*}
    p_t(x)=\frac{1}{(2\pi\wt \sigma_t^2)^{\frac{D}{2}}}  \sum_{k=1}^K \omega_k\cdot\mb{E}_{y\sim Q_k}\Bigg[\exp\bigg(-\frac{\|x-m_ty\|^2}{2\wt \sigma_t^2}\bigg)\Bigg],
\end{equation*}
and 
\begin{equation*}
    \begin{aligned}
        \nabla \log p_t(x)&=\frac{\nabla p_t(x)}{p_t(x)}= \frac{\sum_{k=1}^K \omega_k\cdot\mb{E}_{y\sim Q_k}\Big[\exp\big(-\frac{\|x-m_ty\|^2}{2\wt \sigma_t^2}\big)\cdot \Big(-\frac{x-m_ty}{\wt \sigma_t^2}\Big)\Big]}{\sum_{k=1}^K \omega_k\cdot\mb{E}_{y\sim Q_k}\Big[\exp\big(-\frac{\|x-m_ty\|^2}{2\wt \sigma_t^2}\big)\Big]}.
    \end{aligned}
\end{equation*}
We first use Lemma B.2 in~\cite{TangYang1} to relate the generalization error of the score function $\nabla \log p_t(X_t)$ to the generalization error of the distribution $P_*$ under the $W_1$ metric 
\begin{lemma}\label{lemma1}(Lemma B.2 of~\cite{TangYang1})
    Suppose $\wh S(x,t)\lesssim \frac{\sqrt{\log n}}{\sigma_t}$, then 
    \begin{equation*}
W_1\left(\wh p, P_*\right) \lesssim \frac{1}{n}+\tau^{\frac{1}{2}}+\sum_{i=0}^{L-1} \sqrt{\big((t_i \log n)\wedge 1 \big)\int_{t_i}^{t_{i+1}}\int_{\mb R^D}\left\|\wh S(x,t)-\nabla \log p_t(x)\right\|^2 p_t(x)\,\dd x\,\dd t}. 
    \end{equation*}
\end{lemma}
\noindent Moreover, when $S(x_t,t)\lesssim \sqrt{\frac{\log n}{t\wedge 1}}$, we have for any $l\in \{0,1,\cdots,L-1\}$,
  \begin{equation*}
  \begin{aligned}
          &\int_{t_l}^{t_{l+1} }\int_{\mb R^D}\|S(x_t, t)-\nabla \log p_t(x_t|x)\|^2p_t(x_t|x)\,\dd{x_t}\dd t\\
          &\leq \int_{t_l}^{t_{l+1} }\int 2\cdot\|S(x_t,t)\|^2p_t(x_t|x) \,\dd{x_t}\dd t+\int_{t_l}^{t_{l+1} }\int 2\cdot\|\nabla \log p_t(x_t|x)\|^2p_t(x_t|x) \,\dd{x_t}\dd t\lesssim \log ^2 n.
  \end{aligned}
  \end{equation*}
 Then by Theorem 4.3 of~\cite{oko2023diffusion}, for any $l\in \{0,1,\cdots,L-1\}$, it holds that
 \begin{equation*}
     \begin{aligned}
& \mb E\left[\int_{t_l}^{t_{l+1}} \int_{\mathbb{R}^D}\left\|\wh S(x, t)-\nabla \log p_t(x)\right\|^2 p_t(x)\, \dd x \dd t\right] \\
& \lesssim \inf_{S\in \Phi(L_l,W_l,R_l,B_l,V_l)}\int_{t_l}^{t_{l+1}} \int_{\mathbb{R}^D}\left\|S(x, t)-\nabla \log p_t(x)\right\|^2 p_t(x)\, \dd x \dd t+\frac{(\log n)^2}{n} R_l L_l \log \left(n  L_l  \left\|W_l\right\|_\infty  B_l\right).
\end{aligned}
\end{equation*}
We will choose $ R_l,L_l$, $\left\|W_l\right\|_\infty$ and $B_l$ so  that the approximation error satisfies $$ \inf_{S\in \Phi(L_l,W_l,R_l,B_l,V_l)}\int_{t_l}^{t_{l+1}} \int_{\mathbb{R}^D}\left\|\wh S(x, t)-\nabla \log p_t(x)\right\|^2 p_t(x)\, \dd x \dd t=\wt O\bigg(\sum_{k=1}^K  \min\Big(\frac{ (\sigma_*+\sqrt{1\wedge t_l})^{-d_k}}{n}, n^{-\frac{2\alpha_k}{2\alpha_k+d_k}}\Big)\bigg).$$ We first introduce the following lemma, which states that it is sufficient  to approximate the score function $\nabla \log p_t(x)$ only for values of $x$ that are in close proximity to $\m M=\cup_{k=1}^K \m M_k$.
\begin{lemma}\label{lemma2.1}
  If $\underset{x\in \mb R^D}{\sup}\underset{t\in[\tau,T]}{\sup}[\|S(x,t)\|_{\infty}\wt\sigma_t]\leq c {\sqrt{\log n}} $. Then, there exist constants $(c_0,c_1,c_2,c_3)$ so that for any $i\in \{0,1,\cdots, L-1\}$ and $t\in [t_i,t_{i+1}]$ with $1<\frac{t_{i+1}}{t_i}\leq 2$, 
    \begin{enumerate}
        \item Denote ${\operatorname{dist}}(x,\m M)$ as the distance of point $x\in \mb R^D$ to $\m M$. Then
        \begin{equation*}
            \begin{aligned}
            & \int\left\|\nabla \log p_t(x)-S(x,t)\right\|^2 p_t(x) \, \dd x \\
\leq & \int\left\|\nabla \log p_t(x)-S(x, t)\right\|^2 p_t(x) \cdot 1\left(\operatorname{dist}(x, \m M) \leq c_0\wt\sigma_{t_i}\sqrt{\log n}\right) \,\dd x +(1+c^2)\cdot c_1\frac{1}{n^2}.
            \end{aligned}
        \end{equation*}
 \item For any $x\in \mb R^D$ satisfying ${\opn{dist}(x,\m M)\leq c_0\wt\sigma_{t_i}\sqrt{\log n}}$, we have 
 \begin{enumerate}
     \item $\|\nabla \log p_t(x)\|_{\infty} \leq c_2 \frac{\sqrt{\log n}}{\wt\sigma_{t_i}}$.
     \item $(2\pi\wt\sigma_t^2)^{\frac{D}{2}}p_t(x)\geq n^{-c_3}$.
 \end{enumerate}
\end{enumerate}
\end{lemma}
\noindent Furthermore, we will use Lemma C.2 of~\cite{TangYang1} to  bound the covering number of $\m M_k$.
\begin{lemma}\label{lemma2.2}(Lemma C.2 of~\cite{TangYang1})
For any $k\in [K]$ and any $\epsilon>0$, there exists an $\epsilon$-cover $N_{k,\epsilon}$ of $\m M_k$ so that $N_{k,\epsilon}\subset \m M_k$ and $|N_{k,\epsilon}|\lesssim (\epsilon\wedge 1)^{-d_k}$, moreover, for any $x_0\in \m M_k$ and $r\geq \epsilon$, we have 
\begin{equation*}
    \big|\{x\in N_{k,\epsilon}\,:\, \|x-x_0\|\leq r\}\big|\lesssim \big(\frac{r\wedge 1}{\epsilon\wedge 1}\big)^{d_k}.
\end{equation*}
\end{lemma}
\noindent Let us fix an $i\in \{0,1,\cdots,L-1\}$ and time interval  $t\in [t_{i},t_{i+1}]$ where $1<\frac{t_{i+1}}{t_{i}}\leq 2$, $t_i\geq \tau$ and $t_{i+1}\leq T$.  Using Lemma~\ref{lemma2.1}, we only need to control 
\begin{equation*}
    \begin{aligned}
        &\inf_{S\in \Phi(L_l,W_l,R_l,B_l,V_l)}\int_{t_i}^{t_{i+1}} \int_{\mathbb{R}^D}\left\|S(x, t)-\nabla \log p_t(x)\right\|^2 p_t(x) \cdot 1\left(\operatorname{dist}(x, \m M) \leq c_0\wt\sigma_{t_i}\sqrt{\log n}\right)\, \dd x \dd t\\
      & \leq (t_{i+1}- t_i)\cdot \inf_{S\in \Phi(L_i,W_i,R_i,B_i,V_i)}\sup_{t\in [t_i,t_{i+1}]}\sup_{x\,:\, \operatorname{dist}(x, \m M) \leq c_0\wt\sigma_{t_i}\sqrt{\log n}}\left\|S(x, t)-\nabla \log p_t(x)\right\|^2.\\
    \end{aligned}
\end{equation*}
 Depending on the value of $t_i$ and $\sigma_*$, we will control the above term using different methods. 
\subsection{Case 1: for any $k\in [K]$, $\wt \sigma_{t_i}\geq \frac{n^{-\frac{1}{2\alpha_k+d_k}}}{\sqrt{\log n}}$}
 For any $k\in [K]$, let $N_{k,\epsilon^*}$ be an $\epsilon^*$-cover of $\m M_k$ with $\epsilon^*=\wt\sigma_{t_i}\sqrt{\log n}$ so that statements in Lemma~\ref{lemma2.2} are satisfied.  Then the cardinality of $N_{k,\epsilon^*}$, denoted by $|N_{k,\epsilon^*}|$, satisfies  $|N_{k,\epsilon^*}|=\m O\big(1\vee (\epsilon^*)^{-d_k}\big)$.  Denote $J_k=|N_{k,\epsilon^*}|$ and write $N_{k,\epsilon^*}=\{x_{k,1}^*,x_{k,2}^*,\cdots, x_{k,J_k}^*\}$. For ease of notation, we denote 
 \begin{equation*}
     (A)=\sum_{k=1}^K \omega_k\cdot\mb{E}_{y\sim Q_k}\Big[\exp\big(-\frac{\|x-m_ty\|^2}{2\wt \sigma_t^2}\big)\cdot \Big(-\frac{x-m_ty}{\wt \sigma_t^2}\Big)\Big]
 \end{equation*}
 and 
 \begin{equation*}
     (B)=\sum_{k=1}^K \omega_k\cdot\mb{E}_{y\sim Q_k}\Big[\exp\big(-\frac{\|x-m_ty\|^2}{2\wt \sigma_t^2}\big)\Big],
 \end{equation*}
 so that $\nabla \log p_t(x)=\frac{(A)}{(B)}$. Then using Lemma~\ref{lemma2.1}, we have for any $x\in \mb R^D$ satisfying ${\opn{dist}(x,\m M)\leq c_0\wt\sigma_{t_i}\sqrt{\log n}}$, it holds that $(B)\geq n^{-c_3}$ and $\|\frac{(A)}{(B)}\|\leq c_2\frac{\sqrt{\log n}}{\wt \sigma_{t_i}}$. Define the following partition function
\begin{equation*}
   \rho(x)=\left\{
    \begin{array}{cc}
     1    & |x|<1 \\
     0    & |x|>2 \\
     2-|x|&1<|x|\leq 2.
    \end{array}
    \right.
\end{equation*}
 Let $c_4$ be a large enough constant, we approximate $(A)$ by 
 \begin{equation*}
     (A')=\sum_{k=1}^K \omega_k\frac{\sum_{j=1}^{J_k}\cdot\mb{E}_{y\sim Q_k}\Big[\exp\big(-\frac{\|x-m_ty\|^2}{2\wt \sigma_t^2}\big)\cdot \Big(-\frac{x-m_ty}{\wt \sigma_t^2}\Big)\Big]\cdot \rho\big(\frac{|x-x^*_{kj}|}{c_4\wt \sigma_{t_i}\sqrt{\log n}}\big)}{\max\left(1,\sum_{j=1}^{J_k} \rho\big(\frac{|x-x^*_{kj}|}{c_4\wt \sigma_{t_i}\sqrt{\log n}}\big)\right)},
 \end{equation*}
 and approximate $(B)$ by 
 \begin{equation*}
     (B')=\max\bigg(n^{-c_3},\sum_{k=1}^K \omega_k\frac{\sum_{j=1}^{J_k}\cdot\mb{E}_{y\sim Q_k}\Big[\exp\big(-\frac{\|x-m_ty\|^2}{2\wt \sigma_t^2}\big)\Big]\cdot \rho\big(\frac{|x-x^*_{kj}|}{c_4\wt \sigma_{t_i}\sqrt{\log n}}\big)}{\max\left(1,\sum_{j=1}^{J_k} \rho\big(\frac{|x-x^*_{kj}|}{c_4\wt \sigma_{t_i}\sqrt{\log n}}\big)\right)}\bigg).
 \end{equation*}
Then  with sufficiently large $c_4$, for any $t\in [t_i,t_{i+1}]$ and $x\in\mb R^D$ with ${\rm dist}(x,\m M)\leq c_0\wt \sigma_{t_i}\sqrt{\log n}$,
\begin{equation*}
    \begin{aligned}
        &\|(A)-(A')\|\\
        &=\bigg\|\sum_{k=1}^{K} \omega_k \cdot\bold{1}\Big(\sum_{j=1}^{J_k}\rho(\frac{|x-x_{kj}^*|}{c_4\wt \sigma_{t_i}\sqrt{\log n}})< 1\Big) \cdot\Big(\frac{\sum_{j=1}^{J_k} \rho\big(\frac{|x-x^*_{kj}|}{c_4\wt \sigma_{t_i}\sqrt{\log n}}\big)}{\max\left(1,\sum_{j=1}^{J_k} \rho\big(\frac{|x-x^*_{kj}|}{c_4\wt \sigma_{t_i}\sqrt{\log n}}\big)\right)}-1\Big)\\
        &\qquad\qquad \cdot\mb{E}_{y\sim Q_k}\Big[\exp\big(-\frac{\|x-m_ty\|^2}{2\wt \sigma_t^2}\big)\cdot \Big(-\frac{x-m_ty}{\wt \sigma_t^2}\Big)\Big]\bigg\|\\
         &\leq\bigg\|\sum_{k=1}^{K} \omega_k \cdot\bold{1}\Big(\forall j,\, \|x-x_{kj}^*\|\geq \frac{c_4\wt \sigma_{t_i}\sqrt{\log n}}{\sqrt{D}}\Big) \cdot\mb{E}_{y\sim Q_k}\Big[\exp\big(-\frac{\|x-m_ty\|^2}{2\wt \sigma_t^2}\big)\cdot \Big(-\frac{x-m_ty}{\wt \sigma_t^2}\Big)\Big]\bigg\|\\
                &\leq\bigg\|\sum_{k=1}^{K} \omega_k \cdot\bold{1}\Big({\rm dist}(x,\m M_k)\geq \big(\frac{c_4}{\sqrt{D}}-1\big)\wt \sigma_{t_i}\sqrt{\log n}\Big) \cdot\mb{E}_{y\sim Q_k}\Big[\exp\big(-\frac{\|x-m_ty\|^2}{2\wt \sigma_t^2}\big)\cdot \Big(-\frac{x-m_ty}{\wt \sigma_t^2}\Big)\Big]\bigg\|\\
                &\leq n^{-c_3-1},
    \end{aligned}
\end{equation*}
and similarly
\begin{equation*}
    \begin{aligned}
                &\frac{\sqrt{\log n}}{\wt \sigma_{t_i}}|(B)-(B')|\\
                &\leq\frac{\sqrt{\log n}}{\wt \sigma_{t_i}}\bigg|\sum_{k=1}^{K} \omega_k \cdot\bold{1}\Big({\rm dist}(x,\m M_k)\geq \big(\frac{c_4}{\sqrt{D}}-1\big)\wt \sigma_{t_i}\sqrt{\log n}\Big) \cdot\mb{E}_{y\sim Q_k}\Big[\exp\big(-\frac{\|x-m_ty\|^2}{2\wt \sigma_t^2}\big)\Big]\bigg|\\
                &\leq n^{-c_3-1}.
    \end{aligned}
\end{equation*}
 Combined with $(B)\geq n^{-c_3}$ and , we have 
 \begin{equation*}
 \begin{aligned}
         &\Big\| \nabla \log p_t(x)-\frac{(A')}{(B')}\Big\|=\Big\|\frac{(A)}{(B)}-\frac{(A')}{(B')}\Big\|\leq \Big\|\frac{(A)}{(B)}\Big\|\cdot\frac{|(B)-(B')|}{|(B')|}+\frac{\|(A)-(A')\|}{|B'|}\lesssim n^{-1}.
 \end{aligned}
 \end{equation*}
Then for any  $k\in [K]$, $j\in [J_k]$, and $x\in \mb R^D$ with $\|x-x_{kj}^*\|\leq 2c_4\wt \sigma_{t_i}\sqrt{\log n}$, we have
 \begin{equation*}
 \begin{aligned}
       \{y\in \m M_k\,:\, \|x-y\|\leq c_5\wt\sigma_{t_i}\sqrt{\log n}\}&\subset   \{y\in \m M_k\,:\, \|y-x_{kj}^*\|\leq (2c_4+c_5)\wt\sigma_{t_i}\sqrt{\log n}\}.\,
 \end{aligned}
 \end{equation*}
Let $c_6=2c_4+c_5$,  we further approximate $(A')$ by 
  \begin{equation*}
     (A^{''})=\sum_{k=1}^K \omega_k\frac{\sum_{j=1}^{J_k} \int_{\m M_k\cap \mb B_{c_6\wt\sigma_{t_i}\sqrt{\log n}}(x_{kj}^*)}\exp\big(-\frac{\|x-m_ty\|^2}{2\wt \sigma_t^2}\big)\cdot \Big(-\frac{x-m_ty}{\wt \sigma_t^2}\Big)q_k(y)\,\dd {\rm vol}_{\m M_k}(y)\cdot \rho\big(\frac{|x-x^*_{kj}|}{c_4\wt \sigma_{t_i}\sqrt{\log n}}\big)}{\max\left(1,\sum_{j=1}^{J_k} \rho\big(\frac{|x-x^*_{kj}|}{c_4\wt \sigma_{t_i}\sqrt{\log n}}\big)\right)},
 \end{equation*}
 and approximate $(B')$ by 
 \begin{equation*}
        (B^{''})=\max\bigg(n^{-c_3},\sum_{k=1}^K \omega_k\frac{\sum_{j=1}^{J_k} \int_{\m M_k\cap \mb B_{c_6\wt\sigma_{t_i}\sqrt{\log n}}(x_{kj}^*)}\exp\big(-\frac{\|x-m_ty\|^2}{2\wt \sigma_t^2}\big)q_k(y)\,\dd {\rm vol}_{\m M_k}(y)\cdot \rho\big(\frac{|x-x^*_{kj}|}{c_4\wt \sigma_{t_i}\sqrt{\log n}}\big)}{\max\left(1,\sum_{j=1}^{J_k} \rho\big(\frac{|x-x^*_{kj}|}{c_4\wt \sigma_{t_i}\sqrt{\log n}}\big)\right)}\bigg).
 \end{equation*}
 When $c_5$ is sufficiently large, we have 
 \begin{equation*}
     \begin{aligned}
&\|(A')-(A^{''})\|\\
&=\bigg\|\sum_{k=1}^K \omega_k\frac{\sum_{j=1}^{J_k} \int_{\m M_k\setminus\mb B_{c_6\wt\sigma_{t_i}\sqrt{\log n}}(x_{kj}^*)}\exp\big(-\frac{\|x-m_ty\|^2}{2\wt \sigma_t^2}\big)\cdot \Big(-\frac{x-m_ty}{\wt \sigma_t^2}\Big)q_k(y)\,\dd {\rm vol}_{\m M_k}(y)\cdot \rho\big(\frac{|x-x^*_{kj}|}{c_4\wt \sigma_{t_i}\sqrt{\log n}}\big)}{\max\left(1,\sum_{j=1}^{J_k} \rho\big(\frac{|x-x^*_{kj}|}{c_4\wt \sigma_{t_i}\sqrt{\log n}}\big)\right)}
\bigg\|\\
&\leq  \sum_{k=1}^K \omega_k\frac{\sum_{j=1}^{J_k} \int_{\m M_k\setminus\mb B_{c_6\wt\sigma_{t_i}\sqrt{\log n}}(x_{kj}^*)}\exp\big(-\frac{\|x-m_ty\|^2}{2\wt \sigma_t^2}\big)\cdot \Big\|\frac{x-m_ty}{\wt \sigma_t^2}\Big\|\cdot q_k(y)\,\dd {\rm vol}_{\m M_k}(y)\cdot \rho\big(\frac{|x-x^*_{kj}|}{c_4\wt \sigma_{t_i}\sqrt{\log n}}\big)}{\max\left(1,\sum_{j=1}^{J_k} \rho\big(\frac{|x-x^*_{kj}|}{c_4\wt \sigma_{t_i}\sqrt{\log n}}\big)\right)}\\
 &\leq  \sum_{k=1}^K \omega_k\frac{\sum_{j=1}^{J_k} \int_{\m M_k\setminus\mb B_{c_5\wt\sigma_{t_i}\sqrt{\log n}}(x)}\exp\big(-\frac{\|x-m_ty\|^2}{2\wt \sigma_t^2}\big)\cdot \Big\|\frac{x-m_ty}{\wt \sigma_t^2}\Big\|\cdot q_k(y)\,\dd {\rm vol}_{\m M_k}(y)\cdot \rho\big(\frac{|x-x^*_{kj}|}{c_4\wt \sigma_{t_i}\sqrt{\log n}}\big)}{\max\left(1,\sum_{j=1}^{J_k} \rho\big(\frac{|x-x^*_{kj}|}{c_4\wt \sigma_{t_i}\sqrt{\log n}}\big)\right)}\\
  &\leq  \sum_{k=1}^K \omega_k  \int_{\m M_k\setminus\mb B_{c_5\wt\sigma_{t_i}\sqrt{\log n}}(x)}\exp\big(-\frac{\|x-m_ty\|^2}{2\wt \sigma_t^2}\big)\cdot \Big\|\frac{x-m_ty}{\wt \sigma_t^2}\Big\|\cdot q_k(y)\,\dd {\rm vol}_{\m M_k}(y) \\
  &\leq n^{-c_3-1},
     \end{aligned}
 \end{equation*}
and similarly, 
\begin{equation*}
    \begin{aligned}
                &\frac{\sqrt{\log n}}{\wt \sigma_{t_i}}|(B')-(B^{''})|\\
                &\leq \frac{\sqrt{\log n}}{\wt \sigma_{t_i}}\cdot \sum_{k=1}^K \omega_k  \int_{\m M_k\setminus\mb B_{c_5\wt\sigma_{t_i}\sqrt{\log n}}(x)}\exp\big(-\frac{\|x-m_ty\|^2}{2\wt \sigma_t^2}\big)  \cdot q_k(y)\,\dd {\rm vol}_{\m M_k}(y) \\
  &\leq n^{-c_3-1}.
    \end{aligned}
\end{equation*}
Therefore, combined with $(B'')\geq (B)-2n^{-c_3-1}\frac{\wt\sigma_{t_i}}{\sqrt{\log n}}\geq \frac{n^{-c_3}}{2}$ and $\|\frac{(A')}{(B')}\|\leq  \|\frac{(A)}{(B)}\|+n^{-1}\leq 2c_2\frac{\sqrt{\log n}}{\wt \sigma_{t_i}}$, we have 
\begin{equation*}
        \Big\|\frac{(A')}{(B')}-\frac{(A^{''})}{(B^{''})}\Big\|\leq \Big\|\frac{(A')}{(B')}\Big\|\cdot\frac{|(B')-(B^{''})|}{|(B^{''})|}+\frac{\|(A')-(A^{''})\|}{|B^{''}|}\lesssim n^{-1}.
\end{equation*}
Then let $\ms L=c_7\log n$ with a sufficiently large $c_7$ and consider 
\begin{equation*}
    \begin{aligned}
        (A^{'''})&=\sum_{k=1}^K \omega_k\frac{\sum_{j=1}^{J_k} \int_{\m M_k\cap \mb B_{c_6\wt\sigma_{t_i}\sqrt{\log n}}(x_{kj}^*)}\sum_{l=0}^{\ms L}(-1)^l \frac{\|x-m_ty\|^{2l}}{l!2^{l}\wt \sigma_t^{2l}} \big(-\frac{x-m_ty}{\wt \sigma_t^2}\big)
        q_k(y)\,\dd {\rm vol}_{\m M_k}(y)\cdot \rho\big(\frac{|x-x^*_{kj}|}{c_4\wt \sigma_{t_i}\sqrt{\log n}}\big)}{\max\left(1,\sum_{j=1}^{J_k} \rho\big(\frac{|x-x^*_{kj}|}{c_4\wt \sigma_{t_i}\sqrt{\log n}}\big)\right)}\\
        &=\sum_{k=1}^K \frac{\sum_{j=1}^{J_k}\sum_{l=0}^{\ms L}  \sum_{0\leq s_1\leq 2l+1}  \sum_{s_2\in \mb N_0^D, |s_2|\leq 2l+1} \omega_k  (\frac{1}{\wt \sigma_t})^{2l+2} \cdot m_t^{s_1}\cdot x^{s_2}\cdot \rho\big(\frac{|x-x^*_{kj}|}{c_4\wt \sigma_{t_i}\sqrt{\log n}}\big)\cdot a_{k,j,l,s_1,s_2}}{\max\left(1,\sum_{j=1}^{J_k} \rho\big(\frac{|x-x^*_{kj}|}{c_4\wt \sigma_{t_i}\sqrt{\log n}}\big)\right)},
    \end{aligned}
\end{equation*}
where for $s_2=(s_{21},s_{22},\cdots, s_{2D})$, we denote $x^{s_2}=\prod_{j=1}^{D} x_j^{s_{2j}}$, and  $ a_{k,j,l,s_1,s_2}$'s are $D$-dimensional coefficients that satisfy $ (\frac{1}{\wt \sigma_t})^{2l+2} \|a_{k,j,l,s_1,s_2}\|=\exp(\m O(\log^2 n))$. Similarly, we define
\begin{equation*}
    \begin{aligned}
       & (B^{'''})=\max\bigg(n^{-c_3},\sum_{k=1}^K \omega_k\frac{\sum_{j=1}^{J_k} \int_{\m M_k\cap \mb B_{c_6\wt\sigma_{t_i}\sqrt{\log n}}(x_{kj}^*)}\sum_{l=0}^{\ms L}(-1)^l \frac{\|x-m_ty\|^{2l}}{l!2^{l}\wt \sigma_t^{2l}}  
        q_k(y)\,\dd {\rm vol}_{\m M_k}(y)\cdot \rho\big(\frac{|x-x^*_{kj}|}{c_4\wt \sigma_{t_i}\sqrt{\log n}}\big)}{\max\left(1,\sum_{j=1}^{J_k} \rho\big(\frac{|x-x^*_{kj}|}{c_4\wt \sigma_{t_i}\sqrt{\log n}}\big)\right)}\bigg)\\
          &=\max\bigg(n^{-c_3},\sum_{k=1}^K \frac{\sum_{j=1}^{J_k}\sum_{l=0}^{\ms L}  \sum_{0\leq s_1\leq 2l}  \sum_{s_2\in \mb N_0^D, |s_2|\leq 2l} \omega_k  (\frac{1}{\wt \sigma_t})^{2l} \cdot m_t^{s_1}\cdot x^{s_2}\cdot \rho\big(\frac{|x-x^*_{kj}|}{c_4\wt \sigma_{t_i}\sqrt{\log n}}\big)\cdot b_{k,j,l,s_1,s_2}}{\max\left(1,\sum_{j=1}^{J_k} \rho\big(\frac{|x-x^*_{kj}|}{c_4\wt \sigma_{t_i}\sqrt{\log n}}\big)\right)}\bigg),
    \end{aligned}
\end{equation*}
where $ b_{k,j,l,s_1,s_2}$'s are $1$-dimensional coefficients that satisfy $ (\frac{1}{\wt \sigma_t})^{2l} |b_{k,j,l,s_1,s_2}|=\exp(\m O(\log^2 n))$. Using the fact that for any $x\geq 0$, it holds that $|\exp(-x)-\sum_{l=0}^{\ms L} (-1)^l\frac{x^l}{l!}|\leq \frac{|x|^{\ms L+1}}{(\ms L+1)!}\leq (\frac{|x|e}{(\ms L+1)})^{\ms L+1}$.  When  $\ms L=c_7\log n$ with a sufficiently large $c_7$,  we have
 \begin{equation*}
     \begin{aligned}
         &\|(A^{''})-(A^{'''})\|\\
         &\leq\sum_{k=1}^K \Bigg( \frac{\omega_k}{\max\Big(1,\sum_{j=1}^{J_k} \rho\big(\frac{|x-x^*_{kj}|}{c_4\wt \sigma_{t_i}\sqrt{\log n}}\big)\Big)}\\
         &\cdot \sum_{j=1}^{J_k} \int_{\m M_k\cap \mb B_{c_6\wt\sigma_{t_i}\sqrt{\log n}}(x_{kj}^*)}\Big|\exp(-\frac{\|x-m_ty\|^2}{2\wt\sigma_t^2})-\sum_{l=0}^{\ms L}(-1)^l \frac{\|x-m_ty\|^{2l}}{l!2^{l}\wt \sigma_t^{2l}}\Big| \cdot\big\|\frac{x-m_ty}{\wt \sigma_t^2}\big\|
        q_k(y)\,\dd {\rm vol}_{\m M_k}(y)\\
        &\qquad\cdot \rho\big(\frac{|x-x^*_{kj}|}{c_4\wt \sigma_{t_i}\sqrt{\log n}}\big)\Bigg)\\
             &\leq\sum_{k=1}^K \Bigg( \frac{\omega_k}{\max\Big(1,\sum_{j=1}^{J_k} \rho\big(\frac{|x-x^*_{kj}|}{c_4\wt \sigma_{t_i}\sqrt{\log n}}\big)\Big)}\\
         &\cdot \sum_{j=1}^{J_k} \int_{\m M_k\cap \mb B_{(c_6+2c_4)\wt\sigma_{t_i}\sqrt{\log n}}(x)}\Big|\exp(-\frac{\|x-m_ty\|^2}{2\wt\sigma_t^2})-\sum_{l=0}^{\ms L}(-1)^l \frac{\|x-m_ty\|^{2l}}{l!2^{l}\wt \sigma_t^{2l}}\Big| \cdot \big\|\frac{x-m_ty}{\wt \sigma_t^2}\big\|
        q_k(y)\,\dd {\rm vol}_{\m M_k}(y)\\
        &\qquad\cdot \rho\big(\frac{|x-x^*_{kj}|}{c_4\wt \sigma_{t_i}\sqrt{\log n}}\big)\Bigg)\\
        &\lesssim \frac{\sqrt{\log n}}{\wt \sigma_{t_i}}
      \sup_{k\in [K]} \sup_{y\in \m M_k\cap \mb B_{(c_6+2c_4)\wt\sigma_{t_i}\sqrt{\log n}}(x)} \Big|\exp(-\frac{\|x-m_ty\|^2}{2\wt\sigma_t^2})-\sum_{l=0}^{\ms L}(-1)^l \frac{\|x-m_ty\|^{2l}}{l!2^{l}\wt \sigma_t^{2l}}\Big|(\wt\sigma_{t_i}\sqrt{\log n}))^{d_k}\\
      & \lesssim \frac{\sqrt{\log n}}{\wt \sigma_{t_i}}
      \sup_{k\in [K]} (\wt\sigma_{t_i}\sqrt{\log n}))^{d_k}\cdot \sup_{0\leq z\leq (c_6+2c_4)^2\log n+1} \Big|\exp(-z)-\sum_{l=0}^{\ms L} \frac{(-z)^{l}}{l!}\Big|\\
      &\lesssim \frac{\sqrt{\log n}}{\wt \sigma_{t_i}}
      \sup_{k\in [K]} (\wt\sigma_{t_i}\sqrt{\log n}))^{d_k}\cdot\big (e\frac{(c_6+2c_4)^2\log n+1}{c_7\log n+1}\big)^{c_7\log n+1}\lesssim n^{-c_3-1},
         \end{aligned}
 \end{equation*}
and similarly, 
\begin{equation*}
    \begin{aligned}
                &\frac{\sqrt{\log n}}{\wt \sigma_{t_i}}|(B^{''})-(B^{'''})|\\
                   &\leq \frac{\sqrt{\log n}}{\wt \sigma_{t_i}}\sum_{k=1}^K \Bigg( \frac{\omega_k}{\max\Big(1,\sum_{j=1}^{J_k} \rho\big(\frac{|x-x^*_{kj}|}{c_4\wt \sigma_{t_i}\sqrt{\log n}}\big)\Big)}\\
         &\cdot \sum_{j=1}^{J_k} \int_{\m M_k\cap \mb B_{c_6\wt\sigma_{t_i}\sqrt{\log n}}(x_{kj}^*)}\Big|\exp(-\frac{\|x-m_ty\|^2}{2\wt\sigma_t^2})-\sum_{l=0}^{\ms L}(-1)^l \frac{\|x-m_ty\|^{2l}}{l!2^{l}\wt \sigma_t^{2l}}\Big| 
        q_k(y)\,\dd {\rm vol}_{\m M_k}(y) \cdot \rho\big(\frac{|x-x^*_{kj}|}{c_4\wt \sigma_{t_i}\sqrt{\log n}}\big)\Bigg)\\
        &\lesssim \frac{\sqrt{\log n}}{\wt \sigma_{t_i}}
      \sup_{k\in [K]} (\wt\sigma_{t_i}\sqrt{\log n}))^{d_k}\cdot\big (e\frac{(c_6+2c_4)^2\log n+1}{c_7\log n+1}\big)^{c_7\log n+1}\lesssim n^{-c_3-1}.
    \end{aligned}
\end{equation*}
So we have 
\begin{equation*}
        \Big\|\frac{(A^{''})}{(B^{''})}-\frac{(A^{'''})}{(B^{'''})}\Big\|\lesssim n^{-1}.
\end{equation*}
Now we will construct neural network to approximate $\frac{(A^{'''})}{(B^{'''})}$. We consider the following lemmas in~\cite{oko2023diffusion} for the approximation of $m_t$, $\sigma_t$,  monomial and reciprocal function. In what follows, we write $\Phi(L, W, R, B)=\Phi(L, W, R, B, \infty)$. 
 \begin{lemma}\label{lemma3.3}
     (Lemma 3.3 in~\cite{oko2023diffusion}) There exist neural networks $\phi_m(t), \phi_\sigma(t) \in \Phi(L, W, B, R)$ that approximates $m_t$ and $\sigma_t$ up to $\varepsilon$ for all $t \geq 0$, where $L=\mathcal{O}\left(\log ^2\left(\varepsilon^{-1}\right)\right)$,$\|W\|_{\infty}=\mathcal{O}\left(\log ^3\left(\varepsilon^{-1}\right)\right)$, $R=\mathcal{O}\left(\log ^4\left(\varepsilon^{-1}\right)\right)$, and $B=\exp \left(\mathcal{O}\left(\log ^2\left(\varepsilon^{-1}\right)\right)\right)$.
 \end{lemma}

 \begin{lemma}\label{LemmaF.6}
     (Lemma F.6 in~\cite{oko2023diffusion}) Let $d \geq 2, C \geq 1,0<\varepsilon_{\text {error }} \leq 1$. For any $\varepsilon>0$, there exists a neural network $\phi_{\text {mult}}\left(x_1, x_2, \cdots, x_d\right) \in \Phi(L, W, R, B)$ with $L=\mathcal{O}\left(\log d\left(\log \varepsilon^{-1}+d \log C\right)\right),\|W\|_{\infty}=48 d$, $R= \mathcal{O}\left(d \log \varepsilon^{-1}+d \log C\right), B=C^d$ such that for all $x \in[-C, C]^d$   and   $x^{\prime} \in \mathbb{R}$  with $\left\|x-x^{\prime}\right\|_{\infty} \leq \varepsilon_{\text {error}}$,
     \begin{equation*}
       \left|\phi_{\text {mult}}\left(x_1^{\prime}, x_2^{\prime}, \cdots, x_d^{\prime}\right)-\prod_{d^{\prime}=1}^d x_{d^{\prime}}\right| \leq \varepsilon+d C^{d-1} \varepsilon_{\text {error}}, 
     \end{equation*}
and  for all $x\in [-C,C]$, $|\phi_{\text{mult}}(x)|\leq C^d$  Note that some of $x_i, x_j(i \neq j)$ can be shared. For $\prod_{i=1}^I x_i^{\omega_i}$ with $\omega_i \in \mathbb{Z}_{+}(i=1,2, \cdots, I)$ and $\sum_{i=1}^I \omega_i=d$, there exists a neural network satisfying the same bounds as above, and the network is denoted by $\phi_{\text {mult}}(x ; \omega)$.
 \end{lemma}
 \begin{lemma}\label{lemmaF.7}
     (Lemma F.7 in~\cite{oko2023diffusion}) For any $0<\varepsilon<1$, there exists $\phi_{\text {rec}} \in \Phi(L, W, R, B)$ with $L \leq \mathcal{O}\left(\log ^2 \varepsilon^{-1}\right),\|W\|_{\infty}=\mathcal{O}\left(\log ^3 \varepsilon^{-1}\right), R=\mathcal{O}\left(\log ^4 \varepsilon^{-1}\right)$, and $B=\mathcal{O}\left(\varepsilon^{-2}\right)$ such that
$$
\left|\phi_{\text {rec}}\left(x^{\prime}\right)-\frac{1}{x}\right| \leq \varepsilon+\frac{\left|x^{\prime}-x\right|}{\varepsilon^2}, \quad \text { for all } x \in\left[\varepsilon, \varepsilon^{-1}\right] \text { and } x^{\prime} \in \mathbb{R}.
$$
 \end{lemma}
Therefore, using Lemmas~\ref{lemma3.3},~\ref{LemmaF.12},~\ref{LemmaF.6} and~\ref{lemmaF.7}, we 
\begin{enumerate}
    \item Approximate $m_t$ by $\phi_{m}(t)\in \Phi(L_1',W_1',R_1',B_1')$ with $L_1'=\Theta(\log^4 n)$, $\|W_1'\|_{\infty}=\Theta(\log^6 n)$, $R'_1=\Theta(\log^8 n)$ and $B_1=\exp(\Theta(\log^4 n))$.
    \item Approximate $\sigma_t$ by $\phi_{\sigma}(t)\in \Phi(L'_2,W'_2,R'_2,B'_2)$ with $L'_2=\Theta(\log^4 n)$, $\|W'_2\|_{\infty}=\Theta(\log^6 n)$, $R'_2=\Theta(\log^8 n)$ and $B'_2=\exp(\Theta(\log^4 n))$.
     \item Approximate $\frac{1}{x}$ by $\phi_{rec}(x)\in \Phi(L'_3,W'_3,R'_3,B'_3)$ with $L'_3=\Theta(\log^4 n)$, $\|W'_3\|_{\infty}=\Theta(\log^6 n)$, $R'_3=\Theta(\log^8 n)$ and $B'_3=\exp(\Theta(\log^4 n))$.
 \item For vector $x\in \mb R^D$ and $s\in \mb N^D$, approximate $x^{s}$ by $\phi^{[D]}_{vpower}(x;s)\in \Phi(L'_4,W'_4,R'_4,B'_4)$ with  
 $L'_4=\Theta(\log^2 n\cdot\log \log n)$, $\|W'_4\|_{\infty}=\Theta(\log n)$, $R'_4=\Theta(\log^3 n)$ and $B'_4=\exp(\Theta(\log n\cdot \log \log n))$.
    \item For $x\in \mb R$,  Approximate $x^a$ by $\phi_{power}(x;a)\in \Phi(L'_5,W'_5,R'_5,B'_5)$ with $L'_5=\Theta(\log^2 n\cdot\log\log n)$, $\|W'_5\|_{\infty}=\Theta(\log n)$, $R'_5=\Theta(\log^3 n)$ and $B'_5=\exp(\Theta(\log n\cdot \log \log n))$.
    \item For $x,y\in \mb R$, Approximate $x\cdot y$ by $\phi_{mult}(x,y)\in \Phi(L'_6,W'_6,R'_6,B'_6)$ with $L'_6=\Theta(\log^2 n)$, $\|W'_6\|_{\infty}=\Theta(1)$, $R'_6=\Theta(\log^2 n)$ and $B'_6=\exp(\Theta(\log^2 n))$. For $x\in \mb R$, $y\in \mb R^D$, we denote $\phi_{mult}(x,y)=(\phi_{mult}(x,y_1),\phi_{mult}(x,y_2),\cdots, \phi_{mult}(x,y_D))$.
    \end{enumerate}
Then we denote $\phi_{\frac{1}{\wt\sigma^2}}(t)=\phi_{rec}(\sigma_*^2\cdot \phi_{power}(\phi_m(t),2)+\phi_{power}(\phi_{\sigma}(t),2))$ and $\phi_{\rho}(x;x^*,a)=\rho\big(\frac{|x-x^*|}{a}\big)$. Using the fact that $\rho(x)={\rm ReLU}(2-|x|)-{\rm ReLU}(1-|x|)$ and $|x|={\rm ReLU}(x)+{\rm ReLU}(-x)$,  $\phi_{\rho}(x;x^*,a)$ can be realized exactly by ReLU neural network.  We construct 
\begin{equation*}
\begin{aligned}
        \phi^*_{(A)}(x,t)&=\sum_{k=1}^K \phi_{mult}\Bigg(\phi_{rec}\Big(\max\big(1,\sum_{j=1}^{J_k} \phi_\rho(x;x_{kj}^*,c_4\wt \sigma_{t_i}\sqrt{\log n})\big)\Big)\\
        & \sum_{j=1}^{J_k}\sum_{l=0}^{\ms L}  \sum_{0\leq s_1\leq 2l+1}  \sum_{s_2\in \mb N_0^D, |s_2|\leq 2l+1} \omega_k \cdot a_{k,j,l,s_1,s_2}\cdot \phi_{mult}\bigg(\phi_{mult}\Big(\phi_{mult}\big(\phi_{power}(\phi_{\frac{1}{\wt\sigma^2}}(t),l+1),\\
        &\qquad\phi_{power}(\phi_{m}(t),s_1)\big),\phi^{[D]}_{vpower}(x;s_2)\Big), \phi_\rho(x;x_{kj}^*,c_4\wt \sigma_{t_i}\sqrt{\log n})\bigg)
      \Bigg),
\end{aligned}
\end{equation*}
and 
\begin{equation*}
\begin{aligned}
        \phi^*_{(B)}(x,t)&=\max\Bigg(n^{-c_3},\sum_{k=1}^K \phi_{mult}\Bigg(\phi_{rec}\Big(\max\big(1,\sum_{j=1}^{J_k} \phi_\rho(x;x_{kj}^*,c_4\wt \sigma_{t_i}\sqrt{\log n})\big)\Big)\\
        & \sum_{j=1}^{J_k}\sum_{l=0}^{\ms L}  \sum_{0\leq s_1\leq 2l}  \sum_{s_2\in \mb N_0^D, |s_2|\leq 2l} \omega_k \cdot b_{k,j,l,s_1,s_2}\cdot \phi_{mult}\bigg(\phi_{mult}\Big(\phi_{mult}\big(\phi_{power}(\phi_{\frac{1}{\wt\sigma^2}}(t),l),\\
        &\qquad\phi_{power}(\phi_{m}(t),s_1)\big),\phi^{[D]}_{vpower}(x;s_2)\Big), \phi_\rho(x;x_{kj}^*,c_4\wt \sigma_{t_i}\sqrt{\log n})\bigg)
      \Bigg)\Bigg).
\end{aligned}
\end{equation*}
Then using  Lemmas~\ref{lemma3.3},~\ref{LemmaF.6} and~\ref{lemmaF.7}, we can get 
\begin{equation*}
    \| \phi^*_{(A)}(x,t)-(A^{'''})\|\lesssim  n^{-c_3-1},
\end{equation*}
and 
\begin{equation*}
      \frac{\sqrt{\log n}}{\wt \sigma_{t_i}}|\phi^*_{(B)}(x,t)-(B^{'''})|\lesssim    n^{-c_3-1}.
\end{equation*}
So
\begin{equation*}
\begin{aligned}
       &\Big\|\frac{\phi^*_{(A)}(x,t)}{\phi^*_{(B)}(x,t)}-\nabla \log p_t(x)\Big\|\\
       &\leq   \Big\|\frac{\phi^*_{(A)}(x,t)}{\phi^*_{(B)}(x,t)}-\frac{(A^{'''})}{(B^{'''})}\Big\|+\Big\|\frac{(A^{''})}{(B^{''})}-\frac{(A^{'''})}{(B^{'''})}\Big\|+\Big\|\frac{(A^{''})}{(B^{''})}-\frac{(A^{'})}{(B^{'})}\Big\|+\Big\|\nabla \log p_t(x) -\frac{(A^{'})}{(B^{'})}\Big\|\\
       &\lesssim n^{-1}.
\end{aligned}
    \end{equation*}
Define
\begin{equation*}
    \phi^*(x,t)=\max\bigg(-c_2\frac{\sqrt{\log n}}{\wt \sigma_{t_i}},\min\bigg(c_2\frac{\sqrt{\log n}}{\wt \sigma_{t_i}},\phi_{mult}\Big(\phi_{rec}\big( \phi^*_{(B)}(x,t)\big), \phi^*_{(A)}(x,t)\Big)\bigg)\bigg).
\end{equation*}
Since $\|\nabla \log p_t(x)\|\leq c_2\frac{\sqrt{\log n}}{\wt \sigma_{t_i}}$, it holds for any $t\in [t_i,t_{i+1}]$ and $x\in\mb R^D$ with ${\rm dist}(x,\m M)\leq c_0\wt \sigma_{t_i}\sqrt{\log n}$,
\begin{equation*}
    \begin{aligned}
       & \big\| \phi^*(x,t)-\nabla \log p_t(x)\big\|\\
       &\leq \Big\|\phi_{mult}\Big(\phi_{rec}\big( \phi^*_{(B)}(x,t)\big), \phi^*_{(A)}(x,t)\Big)-\frac{\phi^*_{(A)}(x,t)}{\phi^*_{(B)}(x,t)}\Big\|+\Big\|\frac{\phi^*_{(A)}(x,t)}{\phi^*_{(B)}(x,t)}-\nabla \log p_t(x)\Big\|\lesssim n^{-1}.
    \end{aligned}
\end{equation*}
Furthermore, note that $\max(x,a)={\rm ReLU}(x-a)+a$ and $\min(x,a)=-\big({\rm ReLU}(-x+a)-a\big)$, based on Lemmas F.1-F.3 in~\cite{oko2023diffusion} for the concatenation and parallelization of neural networks, there exists $L_i,W_i,R_i,B_i$ with $L_i=\Theta(\log^4 n)$, $\|W_i\|_{\infty}=\Theta\big(\log^6 n+\sum_{k=1}^K J_k (\log n)^{D+2})$, $R_i=\Theta\big(\log^8 n+\sum_{k=1}^K J_k (\log n)^{D+4})$, and $B_i=\exp(\Theta(\log^4 n))$ so that $\phi^*\in \Phi(L_i,W_i,R_i,B_i,c_2\frac{\sqrt{\log n}}{\wt \sigma_{t_i}})$. So we have 
 \begin{equation*}
     \begin{aligned}
& \mb E\left[\int_{t_i}^{t_{i+1}} \int_{\mathbb{R}^D}\left\|\wh S(x, t)-\nabla \log p_t(x)\right\|^2 p_t(x)\, \dd x \dd t\right] \\
& \lesssim \inf_{S\in \Phi(L_i,W_i,R_i,B_i,V_i)}\int_{t_i}^{t_{i+1}} \int_{\mathbb{R}^D}\left\|S(x, t)-\nabla \log p_t(x)\right\|^2 p_t(x)\, \dd x \dd t+\frac{(\log n)^2}{n} R_i L_i \log \left(n  L_i  \left\|W_i\right\|_\infty  B_i\right)\\
& \lesssim (t_{i+1}- t_i)\cdot  \sup_{t\in [t_i,t_{i+1}]}\sup_{x\,:\, \operatorname{dist}(x, \m M) \leq c_0\wt\sigma_{t_i}\sqrt{\log n}}\left\|\phi^*(x, t)-\nabla \log p_t(x)\right\|^2+\frac{(\log n)^2}{n} R_i L_i \log \left(n  L_i  \left\|W_i\right\|_\infty  B_i\right)+\frac{1}{n^2}\\
&=\wt{\m O}\Big(\sum_{k=1}^K \frac{(\wt\sigma_{t_i})^{-d_k}}{n}\Big).
\end{aligned}
\end{equation*}
Then note that $\wt \sigma_{t_i}\geq \max_{k\in[K]} \frac{n^{-\frac{1}{2\alpha_k+d_k}}}{\sqrt{\log n}}$ and 
\begin{equation*}
    \wt \sigma_{t_i}^2=m_{t_i}^2\sigma_*^2+\sigma_{t_i}^2\geq m_{t_i}^2\sigma_*^2\bold{1}(t_i<\frac{1}{2})+\frac{1}{2}\sigma_{t_i}^2\bold{1}(t_i\geq \frac{1}{2})+\frac{1}{2}\sigma_{t_i}^2\gtrsim \sigma_*^2+1\wedge t_i,
\end{equation*}
where we have used that $\sigma_*\leq 1$. Hence, we have 
 \begin{equation*}
     \begin{aligned}
& \mb E\left[\int_{t_i}^{t_{i+1}} \int_{\mathbb{R}^D}\left\|\wh S(x, t)-\nabla \log p_t(x)\right\|^2 p_t(x)\, \dd x \dd t\right] \\
&=\wt{\m O}\Big(\sum_{k=1}^K \frac{(\wt\sigma_{t_i})^{-d_k}}{n}\Big)=\wt{\m O}\Big(\sum_{k=1}^K \min(n^{-\frac{2\alpha_k}{2\alpha_k+d_k}}, \frac{(\sigma_*+\sqrt{1\wedge t_i})^{-d_k}}{n})\Big),
\end{aligned}
\end{equation*}
and
  \begin{equation*}
     \begin{aligned}
&\mb E\left[((t_i\log n)\wedge 1)\cdot\int_{t_i}^{t_{i+1}} \int_{\mathbb{R}^D}\left\|\wh S(x, t)-\nabla \log p_t(x)\right\|^2 p_t(x)\, \dd x \dd t\right] \\
&=\wt{\m O}\Big(\sum_{k=1}^K \frac{(\wt\sigma_{t_i})^{-d_k+2}}{n}\Big)=\wt{\m O}\Big(\frac{1}{n}+\sum_{k=1}^K \min(n^{-\frac{2\alpha_k+2}{2\alpha_k+d_k}}, \frac{\sigma_*^{2-d_k}}{n})\Big).
\end{aligned}
\end{equation*}

\subsection{Case 2: $\exists\, k\in [K]$, $\wt \sigma_{t_i}< \frac{n^{-\frac{1}{2\alpha_k+d_k}}}{\sqrt{\log n}}$.}
 Consider the sets 
 \begin{equation*}
     \begin{aligned}
        & \m K_1=\{k\in [K]\,:\,  \wt \sigma_{t_i}\geq \frac{n^{-\frac{1}{2\alpha_k+d_k}}}{\sqrt{\log n}}\},\\
         & \m K_2=\{k\in [K]\,:\,  \wt \sigma_{t_i}< \frac{n^{-\frac{1}{2\alpha_k+d_k}}}{\sqrt{\log n}}\}.
     \end{aligned}
 \end{equation*}
 Then we write 
 \begin{equation*}
     \begin{aligned}
         &\sum_{k=1}^K \omega_k\cdot\mb{E}_{y\sim Q_k}\Big[\exp\big(-\frac{\|x-m_ty\|^2}{2\wt \sigma_t^2}\big)\cdot \Big(-\frac{x-m_ty}{\wt \sigma_t^2}\Big)\Big]\\
         &=\underbrace{\sum_{k\in \m K_1} \omega_k\cdot\mb{E}_{y\sim Q_k}\Big[\exp\big(-\frac{\|x-m_ty\|^2}{2\wt \sigma_t^2}\big)\cdot \Big(-\frac{x-m_ty}{\wt \sigma_t^2}\Big)\Big]}_{(A_1)}\\
         &\qquad+\underbrace{\sum_{k\in \m K_2} \omega_k\cdot\mb{E}_{y\sim Q_k}\Big[\exp\big(-\frac{\|x-m_ty\|^2}{2\wt \sigma_t^2}\big)\cdot \Big(-\frac{x-m_ty}{\wt \sigma_t^2}\Big)\Big]}_{(A_2)},
     \end{aligned}
 \end{equation*}
  and  \begin{equation*}
     \begin{aligned}
         &\sum_{k=1}^K \omega_k\cdot\mb{E}_{y\sim Q_k}\Big[\exp\big(-\frac{\|x-m_ty\|^2}{2\wt \sigma_t^2}\big)\Big]\\
         &=\underbrace{\sum_{k\in \m K_1} \omega_k\cdot\mb{E}_{y\sim Q_k}\Big[\exp\big(-\frac{\|x-m_ty\|^2}{2\wt \sigma_t^2}\big) \Big]}_{(B_1)}+\underbrace{\sum_{k\in \m K_2} \omega_k\cdot\mb{E}_{y\sim Q_k}\Big[\exp\big(-\frac{\|x-m_ty\|^2}{2\wt \sigma_t^2}\big) \Big]}_{(B_2)}
     \end{aligned}
 \end{equation*}
  Then similar as case 1, we can approximate terms $(A_1)$ and $(B_1)$ by 
\begin{equation*}
\begin{aligned}
        \phi^*_{(A_1)}(x,t)&=\sum_{k\in \m K_1}\phi_{mult}\Bigg(\phi_{rec}\Big(\max\big(1,\sum_{j=1}^{J_k} \phi_\rho(x;x_{kj}^*,c_4\wt \sigma_{t_i}\sqrt{\log n})\big)\Big)\\
        & \sum_{j=1}^{J_k}\sum_{l=0}^{\ms L}  \sum_{0\leq s_1\leq 2l+1}  \sum_{s_2\in \mb N_0^D, |s_2|\leq 2l+1} \omega_k \cdot a_{k,j,l,s_1,s_2}\cdot \phi_{mult}\bigg(\phi_{mult}\Big(\phi_{mult}\big(\phi_{power}(\phi_{\frac{1}{\wt\sigma^2}}(t),l+1),\\
        &\qquad\phi_{power}(\phi_{m}(t),s_1)\big),\phi^{[D]}_{vpower}(x;s_2)\Big), \phi_\rho(x;x_{kj}^*,c_4\wt \sigma_{t_i}\sqrt{\log n})\bigg)
      \Bigg),
\end{aligned}
\end{equation*}
and 
\begin{equation*}
\begin{aligned}
        \phi^*_{(B_1)}(x,t)&=\sum_{k\in \m K_1} \phi_{mult}\Bigg(\phi_{rec}\Big(\max\big(1,\sum_{j=1}^{J_k} \phi_\rho(x;x_{kj}^*,c_4\wt \sigma_{t_i}\sqrt{\log n})\big)\Big)\\
        & \sum_{j=1}^{J_k}\sum_{l=0}^{\ms L}  \sum_{0\leq s_1\leq 2l}  \sum_{s_2\in \mb N_0^D, |s_2|\leq 2l} \omega_k \cdot b_{k,j,l,s_1,s_2}\cdot \phi_{mult}\bigg(\phi_{mult}\Big(\phi_{mult}\big(\phi_{power}(\phi_{\frac{1}{\wt\sigma^2}}(t),l),\\
        &\qquad\phi_{power}(\phi_{m}(t),s_1)\big),\phi^{[D]}_{vpower}(x;s_2)\Big), \phi_\rho(x;x_{kj}^*,c_4\wt \sigma_{t_i}\sqrt{\log n})\bigg)
      \Bigg),
\end{aligned}
\end{equation*}
respectively. Furthermore, it holds for any $t\in [t_i,t_{i+1}]$ and $x\in \mb R^D$ with ${\rm dist}(x,\m M)\leq c_0\wt \sigma_i\sqrt{\log n}$ that 
\begin{equation*}
    \| \phi^*_{(A_1)}(x,t)-(A)\|\lesssim  n^{-c_3-1},
\end{equation*}
and 
\begin{equation*}
      \frac{\sqrt{\log n}}{\wt \sigma_{t_i}}|\phi^*_{(B)}(x,t)-(B)|\lesssim    n^{-c_3-1}.
\end{equation*}
Then we approximate terms $(A_2)$ and $(B_2)$. For any $k\in \m K_2$,  let $\wt N_{k,\epsilon^*_k}$ be an $\epsilon^*_k$-cover of $\m M_k$ with $\epsilon^*_k=n^{-\frac{1}{2\alpha_k+d_k}}$ so that statements in Lemma~\ref{lemma2.2} are satisfied.  Then $|\wt N_{k,\epsilon^*_k}|=\m O\big(n^{\frac{d_k}{2\alpha_k+d_k}}\big)$. Denote $\wt J_k=|\wt N_{k,\epsilon^*_k}|$ and write $\wt N_{k,\epsilon^*_k}=\{\wt x_{k,1}^*,\wt x_{k,2}^*,\cdots, \wt x_{k,\wt J_k}^*\}$. Then define 
 \begin{equation*}
     (A_2^{'})=\sum_{k\in \m K_2} \omega_k\frac{\sum_{j=1}^{\wt J_k} \mb{E}_{y\sim Q_k}\Big[\exp\big(-\frac{\|x-m_ty\|^2}{2\wt \sigma_t^2}\big)\cdot \Big(-\frac{x-m_ty}{\wt \sigma_t^2}\Big)\Big]\cdot \rho\big(\frac{|x-\wt x^*_{kj}|}{c_4 n^{-{1}/(2\alpha_k+d_k)}}\big)}{\max\left(1,\sum_{j=1}^{\wt J_k} \rho\big(\frac{|x-\wt x^*_{kj}|}{c_4n^{-{1}/(2\alpha_k+d_k)} }\big)\right)},
 \end{equation*}
 and
 \begin{equation*}
     (B_2^{'})=\sum_{k\in \m K_2} \omega_k\frac{\sum_{j=1}^{\wt J_k} \mb{E}_{y\sim Q_k}\Big[\exp\big(-\frac{\|x-m_ty\|^2}{2\wt \sigma_t^2}\big)\Big]\cdot \rho\big(\frac{|x-\wt x^*_{kj}|}{c_4n^{-{1}/(2\alpha_k+d_k)}}\big)}{\max\left(1,\sum_{j=1}^{\wt J_k} \rho\big(\frac{|x-\wt x^*_{kj}|}{c_4n^{-{1}/(2\alpha_k+d_k)} }\big)\right)}.
 \end{equation*}
With sufficiently large $c_4$, for any $t\in [t_i,t_{i+1}]$ and $x\in\mb R^D$ with ${\rm dist}(x,\m M)\leq c_0\wt \sigma_{t_i}\sqrt{\log n}$,
\begin{equation*}
    \begin{aligned}
        &\|(A_2)-(A'_2)\|\\
        &=\bigg\|\sum_{k\in \m K_2} \omega_k \cdot\bold{1}\Big(\sum_{j=1}^{\wt J_k} \rho\big(\frac{|x-\wt x^*_{kj}|}{c_4n^{-{1}/(2\alpha_k+d_k)} }\big)\leq 1\Big) \cdot\Big(\frac{\sum_{j=1}^{\wt J_k}  \rho\big(\frac{|x-\wt x^*_{kj}|}{c_4n^{-{1}/(2\alpha_k+d_k)} }\big)}{\max\left(1,\sum_{j=1}^{\wt J_k}  \rho\big(\frac{|x-\wt x^*_{kj}|}{c_4n^{-{1}/(2\alpha_k+d_k)} }\big)\right)}-1\Big)\\
        &\qquad\qquad \cdot\mb{E}_{y\sim Q_k}\Big[\exp\big(-\frac{\|x-m_ty\|^2}{2\wt \sigma_t^2}\big)\cdot \Big(-\frac{x-m_ty}{\wt \sigma_t^2}\Big)\Big]\bigg\|\\
         &\leq\bigg\|\sum_{k\in \m K_2}\omega_k \cdot\bold{1}\Big(\forall j,\, \|x-\wt x_{kj}^*\|\geq \frac{c_4n^{-{1}/(2\alpha_k+d_k)} }{\sqrt{D}}\Big) \cdot\mb{E}_{y\sim Q_k}\Big[\exp\big(-\frac{\|x-m_ty\|^2}{2\wt \sigma_t^2}\big)\cdot \Big(-\frac{x-m_ty}{\wt \sigma_t^2}\Big)\Big]\bigg\|\\
                &\leq\bigg\|\sum_{k\in \m K_2} \omega_k \cdot\bold{1}\Big({\rm dist}(x,\m M_k)\geq \big(\frac{c_4}{\sqrt{D}}-1\big)n^{-{1}/(2\alpha_k+d_k)} \Big) \cdot\mb{E}_{y\sim Q_k}\Big[\exp\big(-\frac{\|x-m_ty\|^2}{2\wt \sigma_t^2}\big)\cdot \Big(-\frac{x-m_ty}{\wt \sigma_t^2}\Big)\Big]\bigg\|\\
                                &\leq\bigg\|\sum_{k\in \m K_2} \omega_k \cdot\bold{1}\Big({\rm dist}(x,\m M_k)\geq \big(\frac{c_4}{\sqrt{D}}-1\big)\wt \sigma_{t_i}\sqrt{\log n}\Big) \cdot\mb{E}_{y\sim Q_k}\Big[\exp\big(-\frac{\|x-m_ty\|^2}{2\wt \sigma_t^2}\big)\cdot \Big(-\frac{x-m_ty}{\wt \sigma_t^2}\Big)\Big]\bigg\|\\
                &\leq n^{-c_3-1},
    \end{aligned}
\end{equation*}
and similarly
\begin{equation*}
    \begin{aligned}
                &\frac{\sqrt{\log n}}{\wt \sigma_{t_i}}|(B_2)-(B'_2)|\leq n^{-c_3-1}.
    \end{aligned}
\end{equation*}
Then since for any $k\in [K]$, $\m M_k$ is $\beta_k$-smooth, there exists a positive constant $r$ so that for any $y^*\in \m M_k$,
\begin{enumerate}
    \item The projection function ${\rm Proj}_{T_{y^*}\m M}(x-y^*)$ is a local diffeomorphism in $y^*$,  with the inverse $\Psi_{y^*}$  defined on $\mb B_{r}(\bold{0}_D)\cap T_{y^*}\m M$ and is $\beta$-smooth.
    \item $\mb B_r(y^*)\cap \m M\subset\Psi_{y^*}( B_{r}(\bold{0}_D)\cap T_{y^*}\m M)\subset \mb B_{8r/7}(y^*)\cap \m M$. 
\end{enumerate}
Then for any $k\in\m K_2$ and $j\in [\wt J_k]$,  let $V_{kj}$ be an arbitrary orthonormal basis for the tangent space $T_{\wt x^*_{kj}}\m M_k$ at $\wt x^*_{kj}$. Define a function $G^*_{kj}$ with domain $\mb B_{r}(0_{d_k})$ so that
\begin{equation}\label{defG}
    G^*_{kj}(z)=\Psi_{\wt x^*_{kj}}(V_{kj}z) 
 \end{equation}
 Then we can define the local inverse function of $G^*_{kj}$
\begin{equation}\label{defQ}
 Q^*_{kj}(y)=V_{kj}^T\cdot{\rm Proj}_{T_{\wt x^*_{kj}}\m M_k} (y-\wt x^*_{kj})=V_{kj}^T(y-\wt x^*_{kj}).
 \end{equation}
 For any $k\in \m K_2$, and $j\in [\wt J_k]$, consider the Taylor expansion of $G^*_{kj}$ at $0_{d_k}$, 
     \begin{equation*}
         G^*_{kj}(z)=G_{kj}(z)+O(\|z\|^{\beta_k}),
     \end{equation*}
where
     \begin{equation}\label{defG1}
G_{kj}(z)=\wt x_{kj}^*+\sum_{s\in \mb N_0^d\atop 1\leq |s|\leq \lfloor \beta_k\rfloor} G^*_{kj}{}^{(s)}(0_{d_k})\cdot z^s
     \end{equation}
is the polynomial approximation to $G^*_{kj}$. For any $z\in \mb B_r(0_{d_k})$, it holds that $\|G^*_{kj}(z)-G_{kj}(z)\|\lesssim \|z\|^{\beta_k}$ and  $\|\nabla G^*_{kj}(z)-\nabla G_{kj}(z)\|\lesssim \|z\|^{\beta_k-1}$, where $\nabla G(z)=(\nabla G_1(z),\nabla G_2(z),\cdots, \nabla G_D(z))^T$ is the Jacobian matrix of $G=(G_1,G_2,\cdots,G_D)$. Then we use the following lemma to approximate the projection function ${\rm Proj}_{\m M_k}(x)$ for $x$ near $x_{kj}^*$, the proof of which directly follows from Lemma C.8 of~\cite{TangYang1}.
\begin{lemma}\label{lemma2.4} (Lemma C.8 of~\cite{TangYang1})
For any $k\in \m K_2$ and $j\in [\wt J_k]$,  there exists a neural network $\phi^p_{kj}(x)\in \Phi(L'_{kj},W'_{kj},R'_{kj},B'_{kj})$ with $L'_{kj}=\Theta(\log^2 n)$, $\|W'_{kj}\|_{\infty}=\Theta(\log^3 n)$, $R'_{kj}=\Theta(\log^4 n)$  and $ B'_{kj}=\exp(\Theta(\log n))$ so that for any $x$ with $\|x-x_{kj}^*\|\leq 2c_4 n^{-\frac{1}{2\alpha_k+d_k}} $, 
 \begin{enumerate}
     \item $\Big\|\big\langle \nabla G_{kj}(\phi_{jk}^p(x)),x-G_{jk}(\phi_{kj}^p(x))\big\rangle\Big\|\lesssim   n^{-\frac{2\beta_k}{2\alpha_k+d_k}}$. 
      \item $\big\|\phi_{kj}^p(x)-Q^*_{kj}({\rm Proj}_{\m M_k}(x))\big\|\lesssim   n^{-\frac{\beta_k}{2\alpha_k+d_k}}$. 
 \end{enumerate}
\end{lemma}
\noindent Then we define 
 \begin{equation*}
 \begin{aligned}
         & (A_2^{''})=\sum_{k\in \m K_2} \frac{\omega_k}{\max\left(1,\sum_{j=1}^{\wt J_k} \rho\big(\frac{|x-x^*_{kj}|}{c_4n^{-{1}/(2\alpha_k+d_k)}}\big)\right)}\\
          &\qquad\qquad\qquad\cdot \sum_{j=1}^{\wt J_k}\int_{\{y=G^*_{jk}(z)\,:\,\|z-\phi_{jk}^p(x)\|_{\infty}\leq c_5\wt\sigma_{t_i}\sqrt{\log n}\}}\exp\big(-\frac{\|x-m_ty\|^2}{2\wt \sigma_t^2}\big)\cdot \Big(-\frac{x-m_ty}{\wt \sigma_t^2}\Big) q_k(y)\,\dd {\rm vol}_{\m M_k}(y)\\
          &\qquad\qquad\qquad\qquad\cdot \rho\big(\frac{|x-\wt x^*_{kj}|}{c_4 n^{-{1}/(2\alpha_k+d_k)}}\big) \rho\big(\frac{|x-G_{kj}(\phi^p_{kj}(x))|}{c_6 \wt \sigma_{t_i}\sqrt{\log n}}\big),
 \end{aligned}
 \end{equation*}
 and
 \begin{equation*}
  \begin{aligned}
    & (B_2^{''})=\sum_{k\in \m K_2} \frac{\omega_k}{\max\left(1,\sum_{j=1}^{\wt J_k} \rho\big(\frac{|x-\wt x^*_{kj}|}{c_4n^{-{1}/(2\alpha_k+d_k)} }\big)\right)}\\
          &\qquad\qquad\qquad\cdot \sum_{j=1}^{\wt J_k}\int_{\{y=G^*_{jk}(z)\,:\,\|z-\phi_{jk}^p(x)\|_{\infty}\leq c_5\wt\sigma_{t_i}\sqrt{\log n}\}}\exp\big(-\frac{\|x-m_ty\|^2}{2\wt \sigma_t^2}\big) q_k(y)\,\dd {\rm vol}_{\m M_k}(y)\\
          &\qquad\qquad\qquad\qquad\cdot \rho\big(\frac{|x-\wt x^*_{kj}|}{c_4 n^{-{1}/(2\alpha_k+d_k)} }\big) \rho\big(\frac{|x-G_{kj}(\phi^p_{kj}(x))|}{c_6 \wt \sigma_{t_i}\sqrt{\log n}}\big).
           \end{aligned}
 \end{equation*}
Then 
\begin{equation*}
    \begin{aligned}
       & \big\|(A_2')-(A_2^{''})\big\|\\
       &\leq \sum_{k\in \m K_2} \frac{\omega_k}{\max\left(1,\sum_{j=1}^{\wt J_k} \rho\big(\frac{|x-\wt x^*_{kj}|}{c_4n^{-{1}/(2\alpha_k+d_k)} }\big)\right)}\\
       &\cdot \sum_{j=1}^{\wt J_k} \mb{E}_{y\sim Q_k}\Big[\exp\big(-\frac{\|x-m_ty\|^2}{2\wt \sigma_t^2}\big)\cdot \Big\|\frac{x-m_ty}{\wt \sigma_t^2}\Big\|\Big]\\
       &\qquad\qquad\cdot \rho\big(\frac{|x-\wt x^*_{kj}|}{c_4 n^{-{1}/(2\alpha_k+d_k)} }\big)\cdot \Big|1-\rho\big(\frac{|x-G_{kj}(\phi^p_{kj}(x))|}{c_6 \wt \sigma_{t_i}\sqrt{\log n}}\big)\Big| \\
       &+\sum_{k\in \m K_2} \frac{\omega_k}{\max\left(1,\sum_{j=1}^{J_k} \rho\big(\frac{|x-x^*_{kj}|}{c_4n^{-{1}/(2\alpha_k+d_k)} }\big)\right)}\\
          &\qquad\cdot \sum_{j=1}^{\wt J_k}\int_{\m M_k\setminus\{y=G^*_{jk}(z)\,:\,\|z-\phi_{jk}^p(x)\|_{\infty}\leq c_5\wt\sigma_{t_i}\sqrt{\log n}\}}\exp\big(-\frac{\|x-m_ty\|^2}{2\wt \sigma_t^2}\big)\cdot \Big\|\frac{x-m_ty}{\wt \sigma_t^2}\Big\| q_k(y)\,\dd {\rm vol}_{\m M_k}(y)\\
          &\qquad\qquad\cdot \rho\big(\frac{|x-\wt x^*_{kj}|}{c_4 n^{-{1}/(2\alpha_k+d_k)} }\big)\rho\big(\frac{|x-G_{kj}(\phi^p_{kj}(x))|}{c_6 \wt \sigma_{t_i}\sqrt{\log n}}\big)  \\
          &\leq \underbrace{\sup_{k\in \m K_2}\sup_{j\in [\wt J_k]}\underset{x\, :\, |x-G_{kj}(\phi^p_{kj}(x))|\geq c_6\wt \sigma_{t_i}\sqrt{\log n}\atop |x-\wt x_{kj}^*|\leq 2c_4n^{-{1}/(2\alpha_k+d_k)} }{\sup} \mb{E}_{y\sim Q_k}\Big[\exp\big(-\frac{\|x-m_ty\|^2}{2\wt \sigma_t^2}\big)\cdot \Big\|\frac{x-m_ty}{\wt \sigma_t^2}\Big\|\Big]}_{(C)}\\
          &+\sup_{k\in \m K_2}\sup_{j\in [\wt J_k]}\underset{x\, :\, |x-G_{kj}(\phi^p_{kj}(x)|\leq 2c_6\wt \sigma_{t_i}\sqrt{\log n}\atop |x-\wt x_{kj}^*|\leq 2c_4n^{-{1}/(2\alpha_k+d_k)} }{\sup}\int_{\m M_k\setminus\{y=G^*_{jk}(z)\,:\,\|z-\phi_{jk}^p(x)\|_{\infty}\leq c_5\wt\sigma_{t_i}\sqrt{\log n}\}}\exp\big(-\frac{\|x-m_ty\|^2}{2\wt \sigma_t^2}\big)\\
          &\underbrace{\qquad\qquad\qquad\cdot \Big\|\frac{x-m_ty}{\wt \sigma_t^2}\Big\| q_k(y)\,\dd {\rm vol}_{\m M_k}(y).\qquad\qquad\qquad\qquad\qquad\qquad\qquad\qquad\qquad\qquad\qquad\qquad}_{(D)}
    \end{aligned}
\end{equation*}
We first bound term $(C)$.  Note that by Lemma~\ref{lemma2.4}, we have for any $x\in 
\mb R^D$ with  $|x-\wt x_{kj}^*|\leq 2c_4n^{-{1}/(2\alpha_k+d_k)}$, 
\begin{equation*}
    \|\phi^p_{kj}(x)\|\leq \|\phi^p_{kj}(x)-Q^*_{kj}({\rm Proj}_{\m M_k}x)\|+\|{\rm Proj}_{\m M_k}(x)-\wt x_{kj}^*\|\lesssim n^{-{1}/(2\alpha_k+d_k)},
\end{equation*}
and there exist constants $C,C_1$ so that
\begin{equation*}
    \begin{aligned}
        |x-G_{kj}(\phi^p_{kj}(x))|&\leq  |x-{\rm Proj}_{\m M_k}(x)|+|G^*_{jk}(Q^*_{jk}({\rm Proj}_{\m M_k}(x)))-G^*_{jk}(\phi^p_{kj}(x))|+|G^*_{jk}(\phi^p_{kj}(x))-G_{jk}(\phi^p_{kj}(x))|\\
        &\leq  |x-{\rm Proj}_{\m M_k}(x)|+C\,  n^{-\frac{\beta_k}{2\alpha_k+d_k}}\\
          &\leq  |x-{\rm Proj}_{\m M_k}(x)|+C_1\, \wt \sigma_{\tau}     \leq  |x-{\rm Proj}_{\m M_k}(x)|+C_1\, \wt \sigma_{t_i}.
    \end{aligned}
\end{equation*}
Hence when $c_6$ is large enough, 
\begin{equation*}
    \begin{aligned}
        (C)&\leq \sup_{k\in \m K_2}\sup_{j\in [\wt J_k]}\underset{x\, :\, |x-{\rm Proj}_{\m M_k}(x)|\geq (c_6-C_1)\wt \sigma_{t_i}\sqrt{\log n}}{\sup} \mb{E}_{y\sim Q_k}\Big[\exp\big(-\frac{\|x-m_ty\|^2}{2\wt \sigma_t^2}\big)\cdot \Big\|\frac{x-m_ty}{\wt \sigma_t^2}\Big\|\Big]\\
        &\leq \sup_{k\in \m K_2}\sup_{j\in [\wt J_k]}\underset{x\, :\, {\rm dist}(x,\m M_k)\geq \frac{(c_6-C_1)}{\sqrt{D}}\wt \sigma_{t_i}\sqrt{\log n}}{\sup} \mb{E}_{y\sim Q_k}\Big[2\exp\big(-\frac{\|x-y\|^2}{4\wt \sigma_t^2}\big)\cdot \Big\|1+\frac{x-y}{\wt \sigma_t^2}\Big\|\Big]\leq \frac{n^{-c_3-1}}{2}.
    \end{aligned}
\end{equation*}
Then we bound term $(D)$. For any $x\in \mb R^D$ with $ |x-G_{kj}(\phi^p_{kj}(x)|\leq 2c_6\wt \sigma_{t_i}\sqrt{\log n}$ and $ |x-\wt x_{kj}^*|\leq 2c_4n^{-{1}/(2\alpha_k+d_k)}$, we have
\begin{equation*}
\begin{aligned}
     |x-{\rm Proj}_{\m M_k}(x)|&\leq    |x-G_{kj}(\phi^p_{kj}(x))|+|G^*_{jk}(\phi^p_{kj}(x))-G_{jk}(\phi^p_{kj}(x))|+|G^*_{jk}(Q^*_{jk}({\rm Proj}_{\m M_k}x))-G^*_{jk}(\phi^p_{kj}(x))|\\
     &\leq  |x-G_{kj}(\phi^p_{kj}(x))|+C_1\wt\sigma_{t_i}\\ 
\end{aligned}
\end{equation*}
 \begin{equation*}
 \begin{aligned}
       \{y\in \m M_k\,:\, \|y-x\|\leq c\wt\sigma_{t_i}\sqrt{\log n}\}&\subset   \{y\in \m M\,:\, \|y-{\rm Proj}_{\m M_k}(x)\|\leq (c+2c_6+C_1)\wt\sigma_{t_i}\sqrt{\log n}\}\\
       &\subset  \{y=G^*_{jk}(z)\,:\, \|z-Q^*_{jk}({\rm Proj}_{\m M_k}(x))\|\leq (c+2c_6+C_1)\wt\sigma_{t_i}\sqrt{\log n}\}
 \end{aligned}
 \end{equation*}
Furthermore, using Lemma~\ref{lemma2.4} and $\wt \sigma_{t_i}\geq \wt \sigma_{\tau}\gtrsim \sqrt{\log n}\cdot n^{-\min_{k\in [K]}\frac{\alpha_k+1}{2\alpha_k+d_k}}\geq \sqrt{\log n}\cdot n^{-\min_{k\in [K]}\frac{\beta_k}{2\alpha_k+d_k}}$, we have
\begin{equation*}
    \begin{aligned}
       \|z-\phi^p_{jk}(x)\|&\leq \|z-Q^*_{jk}({\rm Proj}_{\m M_k}(x))\|+\|\phi_{jk}^p(x)Q^*_{jk}({\rm Proj}_{\m M_k}(x))\|\\
       &\leq \|z-Q^*({\rm Proj}_{\m M}(x))\|+C n^{-\frac{\beta_k}{2\alpha_k+d}}\leq   \|z-Q^*({\rm Proj}_{\m M}(x))\|+C_1\wt \sigma_{t_i},
    \end{aligned}
\end{equation*}
and thus 
\begin{equation*}
    \begin{aligned}
     &\{y\in \m M_k\,:\, \|y-x\|\leq c\wt\sigma_{t_i}\sqrt{\log n}\} \subset  \{y=G^*_{jk}(z)\,:\, \|z-Q^*_{jk}({\rm Proj}_{\m M_k}(x))\|\leq (c+2c_6+C_1)\wt\sigma_{t_i}\sqrt{\log n}\}\\
    & \subset  \{y=G^*_{jk}(z)\,:\, \|z-\phi^p_{jk}(x)\|\leq (c+2c_6+2C_1)\wt\sigma_{t_i}\sqrt{\log n}\}\\
    & \subset  \{y=G^*_{jk}(z)\,:\, \|z-\phi^p_{jk}(x)\|_{\infty}\leq (c+2c_6+2C_1)\wt\sigma_{t_i}\sqrt{\log n}\}.
    \end{aligned}
\end{equation*}
So when $c_5$ is sufficiently large, we have 
\begin{equation*}
    \begin{aligned}
      (D)&\leq  \sup_{k\in \m K_2}\sup_{j\in [\wt J_k]}\underset{x\, :\, |x-G_{kj}(\phi^p_{kj}(x)|\leq 2c_6\wt \sigma_{t_i}\sqrt{\log n}\atop |x-\wt x_{kj}^*|\leq 2c_4n^{-{1}/(2\alpha_k+d_k)} }{\sup}\int_{\m M_k\setminus \mb B_{\frac{c_5}{2} \wt \sigma_{t_i}\sqrt{\log n}}(x)}\exp\big(-\frac{\|x-m_ty\|^2}{2\wt \sigma_t^2}\big) \Big\|\frac{x-m_ty}{\wt \sigma_t^2}\Big\| q_k(y)\,\dd {\rm vol}_{\m M_k}(y)\\
      &\leq  \sup_{k\in \m K_2}\sup_{j\in [\wt J_k]}\underset{x\, :\, |x-G_{kj}(\phi^p_{kj}(x)|\leq 2c_6\wt \sigma_{t_i}\sqrt{\log n}\atop |x-\wt x_{kj}^*|\leq 2c_4n^{-{1}/(2\alpha_k+d_k)} }{\sup}\int_{\m M_k\setminus \mb B_{\frac{c_5}{2} \wt \sigma_{t_i}\sqrt{\log n}}(x)}\exp\big(-\frac{\|x-y\|^2}{4\wt \sigma_t^2}\big) \Big\|1+\frac{x-y}{\wt \sigma_t^2}\Big\| q_k(y)\,\dd {\rm vol}_{\m M_k}(y)\\
      &\leq \frac{n^{-c_3-1}}{2}.
    \end{aligned}
\end{equation*}
So $\|(A_2^{'})-(A_2^{''})\|\leq n^{-c_3-1}$. Similarly, we can show that $\frac{\sqrt{\log n}}{\wt \sigma_{t_i}} |(B_2^{'})-(B_2^{''})|\leq n^{-c_3-1}$. Then denote $v^*_{kj}(z)=q_k(G^*_{kj}(z))\sqrt{\opn{det}\big(\nabla G^*_{kj}(z)^T\nabla G^*_{kj}(z)\big)}$, we can rewrite term $(A_2^{''})$ as 
\begin{equation*}
    \begin{aligned}
      &  (A_2^{''})=\sum_{k\in \m K_2} \frac{\omega_k}{\max\left(1,\sum_{j=1}^{\wt J_k} \rho\big(\frac{|x-x^*_{kj}|}{c_4n^{-{1}/(2\alpha_k+d_k)}}\big)\right)}\\
          &\qquad\qquad\qquad\cdot \sum_{j=1}^{\wt J_k}\int_{\{z\in \mb R^{d_k}\,:\,\|z-\phi_{jk}^p(x)\|_{\infty}\leq c_5\wt\sigma_{t_i}\sqrt{\log n}\}}\exp\big(-\frac{\|x-m_tG^*_{jk}(z)\|^2}{2\wt \sigma_t^2}\big)\cdot \Big(-\frac{x-m_tG^*_{jk}(z)}{\wt \sigma_t^2}\Big) v^*_{kj}(z)\,\dd z\\
          &\qquad\qquad\qquad\qquad\cdot \rho\big(\frac{|x-\wt x^*_{kj}|}{c_4 n^{-{1}/(2\alpha_k+d_k)}}\big) \rho\big(\frac{|x-G_{kj}(\phi^p_{kj}(x))|}{c_6 \wt \sigma_{t_i}\sqrt{\log n}}\big)
    \end{aligned}
\end{equation*}
 For any $k\in \m K_2$, and $j\in [\wt J_k]$, consider the Taylor expansion of $v^*_{kj}$ at $0_{d_k}$, 
     \begin{equation*}
         v^*_{kj}(z)=v_{kj}(z)+O(\|z\|^{\alpha_k}),
     \end{equation*}
where
     \begin{equation}\label{defG1}
v_{kj}(z)= v^*_{kj}(0_{d_k})+\sum_{s\in \mb N_0^d\atop 1\leq |s|\leq \lfloor \alpha_k\rfloor} v*_{kj}{}^{(s)}(0_{d_k})\cdot z^s.
     \end{equation}
We define 
\begin{equation*}
    \begin{aligned}
 &(A_2^{'''})=\sum_{k\in \m K_2} \frac{\omega_k}{\max\left(1,\sum_{j=1}^{\wt J_k}\rho\big(\frac{|x-x^*_{kj}|}{c_4n^{-{1}/(2\alpha_k+d_k)}}\big)\right)}\\
          &\qquad\qquad\qquad\cdot \sum_{j=1}^{\wt J_k}\int_{\{z\in \mb R^{d_k}\,:\,\|z-\phi_{jk}^p(x)\|_{\infty}\leq c_5\wt\sigma_{t_i}\sqrt{\log n}\}}\exp\big(-\frac{\|x-m_tG_{jk}(z)\|^2}{2\wt \sigma_t^2}\big)\cdot \Big(-\frac{x-m_tG_{jk}(z)}{\wt \sigma_t^2}\Big) v_{kj}(z)\,\dd z\\
          &\qquad\qquad\qquad\qquad\cdot \rho\big(\frac{|x-\wt x^*_{kj}|}{c_4 n^{-{1}/(2\alpha_k+d_k)}}\big) \rho\big(\frac{|x-G_{kj}(\phi^p_{kj}(x))|}{c_6 \wt \sigma_{t_i}\sqrt{\log n}}\big)
    \end{aligned}
\end{equation*}
and
\begin{equation*}
    \begin{aligned}
 &(B_2^{'''})=\sum_{k\in \m K_2} \frac{\omega_k}{\max\left(1,\sum_{j=1}^{\wt J_k}\rho\big(\frac{|x-x^*_{kj}|}{c_4n^{-{1}/(2\alpha_k+d_k)}}\big)\right)}\\
          &\qquad\qquad\qquad\cdot \sum_{j=1}^{\wt J_k}\int_{\{z\in \mb R^{d_k}\,:\,\|z-\phi_{jk}^p(x)\|_{\infty}\leq c_5\wt\sigma_{t_i}\sqrt{\log n}\}}\exp\big(-\frac{\|x-m_tG_{jk}(z)\|^2}{2\wt \sigma_t^2}\big)\cdot   v_{kj}(z)\,\dd z\\
          &\qquad\qquad\qquad\qquad\cdot \rho\big(\frac{|x-\wt x^*_{kj}|}{c_4 n^{-{1}/(2\alpha_k+d_k)}}\big) \rho\big(\frac{|x-G_{kj}(\phi^p_{kj}(x))|}{c_6 \wt \sigma_{t_i}\sqrt{\log n}}\big).
    \end{aligned}
\end{equation*}
Then we have 
\begin{equation*}
    \begin{aligned}
      &\big\|  (A_2^{''})-(A_2^{'''})\big\|\leq\sum_{k\in \m K_2}\sum_{j=1}^{\wt J_k} \frac{\omega_k \rho\big(\frac{|x-\wt x^*_{kj}|}{c_4 n^{-{1}/(2\alpha_k+d_k)}}\big) \rho\big(\frac{|x-G_{kj}(\phi^p_{kj}(x))|}{c_6 \wt \sigma_{t_i}\sqrt{\log n}}\big)}{\max\left(1,\sum_{j=1}^{\wt J_k}\rho\big(\frac{|x-x^*_{kj}|}{c_4n^{-{1}/(2\alpha_k+d_k)}}\big)\right)}\cdot \exp\big(-\frac{\|x-G_{kj}(\phi_{kj}^p(x))\|^2}{2\wt \sigma_t^2}\big)\cdot (E),\\
      &\text{ where }(E)= \Bigg\| \int_{\{z\in \mb R^{d_k}\,:\,\|z-\phi_{jk}^p(x)\|_{\infty}\leq c_5\wt\sigma_{t_i}\sqrt{\log n}\}}\exp\big(-\frac{\|G_{kj}(\phi_{kj}^p(x))-m_tG_{jk}(z)\|^2}{2\wt \sigma_t^2}\big)\\
      &\qquad\cdot\exp\big(-\frac{\langle x-G_{kj}(\phi_{kj}^p(x)),G_{kj}(\phi_{kj}^p(x))-m_tG_{jk}(z)\rangle}{\wt \sigma_t^2}\big)\cdot \Big(-\frac{x-m_tG_{jk}(z)}{\wt \sigma_t^2}\Big) v_{kj}(z)\,\dd z\\
            &\qquad \qquad-\int_{\{z\in \mb R^{d_k}\,:\,\|z-\phi_{jk}^p(x)\|_{\infty}\leq c_5\wt\sigma_{t_i}\sqrt{\log n}\}}\exp\big(-\frac{\|G_{kj}(\phi_{kj}^p(x))-m_tG^*_{jk}(z)\|^2}{2\wt \sigma_t^2}\big)\\
      & \qquad\qquad\qquad\cdot\exp\big(-\frac{\langle x-G_{kj}(\phi_{kj}^p(x)),G_{kj}(\phi_{kj}^p(x))-m_tG^*_{jk}(z)\rangle}{\wt \sigma_t^2}\big)\cdot \Big(-\frac{x-m_tG^*_{jk}(z)}{\wt \sigma_t^2}\Big) v^*_{kj}(z)\,\dd z\Bigg\|.
    \end{aligned}
\end{equation*}
 To bound the term $(E)$,  consider an arbitrary $k\in \m K_2$ and $j\in [\wt J_k]$, we will consider and bound the following terms using Lemma~\ref{lemma2.4} for any $x\in \mb R^D$  with $|x-\wt x_{kj}^*|\leq 2c_4n^{\frac{-1}{2\alpha_k+d_k}}$ and $|x-G_{kj}(\phi_{kj}^p(x))|\leq 2c_6\wt \sigma_{t_i}\sqrt{\log n}$, and any $z\in \mb R^{d_k}$ satisfying $\|z-\phi_{jk}^p(x)\|_{\infty}\leq c_5\wt\sigma_{t_i}\sqrt{\log n}$:
\begin{equation*}
    (i). \,\left\|G_{kj}(\phi^p_{kj}(x))-m_t G_{jk}(z)\right\|\leq  \left\||G_{kj}(\phi^p_{kj}(x))-  G_{jk}(z)\right\|+(1-m_t)\|G_{jk}(z)\|\lesssim \wt\sigma_{t_i}\sqrt{\log n}.
    \end{equation*}
    \begin{equation*}
        \begin{aligned}
        (ii).\,  &\frac{1}{\wt\sigma_t^2}\cdot\Big|\|G_{kj}(\phi^p_{kj}(x))-m_t G_{jk}(z)\|^2-\|G_{kj}(\phi^p_{kj}(x))-m_t G^*_{jk}(z)\|^2\Big|\\
        &\lesssim\frac{\wt\sigma_{t_i}\sqrt{\log n}\cdot n^{-\frac{\beta_k}{2\alpha_k+d_k}}}{\wt\sigma_t^2}\asymp \frac{\sqrt{\log n}\cdot n^{-\frac{\beta_k}{2\alpha_k+d_k}}}{\wt\sigma_{t_i}},
        \end{aligned}
    \end{equation*}
    where we have used $\|\phi_{jk}^p(x)\|\lesssim n^{-\frac{1}{2\alpha_k+d_k}}$ and $\|z\|\leq \|z-\phi_{jk}^p(x)\|+\|\phi_{jk}^p(x)\|\lesssim n^{-\frac{1}{2\alpha_k+d_k}}$.
    \begin{equation*}
      (iii).\,\|x-m_t G_{jk}(z)\|\leq \|x-G_{kj}(\phi_{kj}^p(x))\|+\|G_{kj}(\phi_{kj}^p(x))-m_t G_{jk}(z)\|\lesssim \wt\sigma_{t_i}\sqrt{\log n}.
    \end{equation*}
    \begin{equation*}
        \begin{aligned}
           (iv).\, &\frac{1}{\wt\sigma_t^2}\cdot\Big|\big\langle x-G_{kj}(\phi_{kj}^p(x)),G_{kj}(\phi_{kj}^p(x))-m_tG_{jk}(z)\big\rangle-\big\langle x-G_{kj}(\phi_{kj}^p(x)),G_{kj}(\phi_{kj}^p(x))-m_tG^*_{jk}(z)\big\rangle\Big|\\
            &\lesssim\frac{\wt\sigma_{t_i}\sqrt{\log n}\cdot n^{-\frac{\beta_k}{2\alpha_k+d_k}}}{\wt\sigma_t^2}\asymp \frac{\sqrt{\log n}\cdot n^{-\frac{\beta_k}{2\alpha_k+d_k}}}{\wt\sigma_{t_i}}.
        \end{aligned}
    \end{equation*}

\begin{equation}\label{eqncase3ip}
\begin{aligned}
    &(v).\, \big|\langle x-G_{kj}(\phi_{kj}^p(x)),G_{kj}(\phi_{kj}^p(x))-m_tG_{jk}(z)\rangle\big|\\ &\leq\Big|\langle x-G_{kj}(\phi_{kj}^p(x)), \nabla G_{kj}(\phi_{kj}^p(x)) (\phi^p_{kj}(x)-z)\rangle\Big|\\
    &+\Big|\langle x-G_{kj}(\phi_{kj}^p(x)), G_{kj}(\phi_{kj}^p(x))-G_{kj}(z)-\nabla G_{kj}(\phi_{kj}^p(x))(\phi^p_{kj}(x)-z)\rangle\Big|\\
    &+\Big |\langle x-G_{kj}(\phi_{kj}^p(x)),G_{kj}(z)-m_tG_{kj}(z)\rangle\Big|\\
       &\lesssim n^{-\frac{2\beta_k}{2\alpha_k+d_k}}\wt\sigma_{t_i}\sqrt{\log n}+\wt\sigma_{t_i}^3(\log n)^{\frac{3}{2}}+\wt\sigma_{t_i}^3\sqrt{\log n}\\
       &\lesssim \wt\sigma_{t_i}^3(\log n)^{\frac{3}{2}}.
       \end{aligned}
\end{equation}
Combining all the pieces, we can obtain 
\begin{equation}\label{eqn:case3ub1}
\begin{aligned}
&(E)\lesssim  \int_{\{z\in \mb R^{d_k}\,:\,\|z-\phi_{jk}^p(x)\|_{\infty}\leq c_5\wt\sigma_{t_i}\sqrt{\log n}\}}\exp\big(-\frac{\|G_{kj}(\phi_{kj}^p(x))-m_tG_{jk}(z)\|^2}{2\wt \sigma_t^2}\big)\,\dd z
\\&\qquad\qquad\qquad\cdot  \Big(\frac{n^{-\frac{\beta_k}{2\alpha_k+d_k}}\log n}{\wt\sigma^2_{t_i}}+ \frac{n^{-\frac{\alpha_k}{2\alpha_k+d}}\sqrt{\log n}}{\wt \sigma_{t_i}}\Big).
\end{aligned}
\end{equation}
We then bound $\int_{\{z\in \mb R^{d_k}\,:\,\|z-\phi_{jk}^p(x)\|_{\infty}\leq c_5\wt\sigma_{t_i}\sqrt{\log n}\}}\exp\big(-\frac{\|G_{kj}(\phi_{kj}^p(x))-m_tG_{jk}(z)\|^2}{2\wt \sigma_t^2}\big)\,\dd z$.  Notice that 
\begin{equation*}
\begin{aligned}
    \|\phi^p_{jk}(x)- z\|  &\leq  \|G^*_{jk}(\phi^p_{jk}(x))- G^*_{jk}(z)\|\\
    &\leq \|G_{jk}(\phi^p_{jk}(x))- m_tG^*_{jk}(z)\|+(1-m_t)\|G^*_{jk}(z)\|+ \|G_{jk}(\phi^p_{jk}(x))-G^*_{jk}(\phi^p_{jk}(x))\|\\
    &\leq \|G_{jk}(\phi^p_{jk}(x))-m_t G^*_{jk}(z)\|+o(\wt\sigma_{t_i}),
\end{aligned}
\end{equation*}
we have 
\begin{equation*}
    \begin{aligned}
        &\int_{\{z\in \mb R^{d_k}\,:\,\|z-\phi_{jk}^p(x)\|_{\infty}\leq c_5\wt\sigma_{t_i}\sqrt{\log n}\}}\exp\big(-\frac{\|G_{kj}(\phi_{kj}^p(x))-m_tG_{jk}(z)\|^2}{2\wt \sigma_t^2}\big)\,\dd z\\
        &\lesssim \int_{\mb R^{d_k}}  \exp\left(-\frac{\|\phi^p_{jk}(x)-z\|^2}{2\wt \sigma_t^2}\right)\,\dd z\lesssim \wt\sigma_{t_i}^{d_k}.
    \end{aligned}
\end{equation*}
So we can get 
\begin{equation*}
    \begin{aligned}
        \big\|  (A_2^{''})-(A_2^{'''})\big\|\lesssim\sum_{k\in \m K_2}\sum_{j=1}^{\wt J_k} &\frac{\omega_k  \rho\big(\frac{|x-\wt x^*_{kj}|}{c_4 n^{-{1}/(2\alpha_k+d_k)}}\big) \rho\big(\frac{|x-G_{kj}(\phi^p_{kj}(x))|}{c_6 \wt \sigma_{t_i}\sqrt{\log n}}\big)}{\max\left(1,\sum_{j=1}^{\wt J_k}\rho\big(\frac{|x-x^*_{kj}|}{c_4n^{-{1}/(2\alpha_k+d_k)}}\big)\right)}\cdot \exp\big(-\frac{\|x-G_{kj}(\phi_{kj}^p(x))\|^2}{2\wt \sigma_t^2}\big)\\
        &\cdot \wt\sigma_{t_i}^{d_k}\cdot   \Big(\frac{n^{-\frac{\beta_k}{2\alpha_k+d_k}}\log n}{\wt\sigma^2_{t_i}}+ \frac{n^{-\frac{\alpha_k}{2\alpha_k+d}}\sqrt{\log n}}{\wt \sigma_{t_i}}\Big).
    \end{aligned}
\end{equation*}
Similarly, we can get 
\begin{equation*}
    \begin{aligned}
    \frac{\sqrt{\log n}}{\wt \sigma_{t_i}}\big|  (B_2^{''})-(B_2^{'''})\big|\lesssim\sum_{k\in \m K_2}\sum_{j=1}^{\wt J_k} &\frac{\omega_k  \rho\big(\frac{|x-\wt x^*_{kj}|}{c_4 n^{-{1}/(2\alpha_k+d_k)}}\big) \rho\big(\frac{|x-G_{kj}(\phi^p_{kj}(x))|}{c_6 \wt \sigma_{t_i}\sqrt{\log n}}\big)}{\max\left(1,\sum_{j=1}^{\wt J_k}\rho\big(\frac{|x-x^*_{kj}|}{c_4n^{-{1}/(2\alpha_k+d_k)}}\big)\right)}\cdot \exp\big(-\frac{\|x-G_{kj}(\phi_{kj}^p(x))\|^2}{2\wt \sigma_t^2}\big)\\
        &\cdot \wt\sigma_{t_i}^{d_k}\cdot   \Big(\frac{n^{-\frac{\beta_k}{2\alpha_k+d_k}}\log n}{\wt\sigma^2_{t_i}}+ \frac{n^{-\frac{\alpha_k}{2\alpha_k+d}}\sqrt{\log n}}{\wt \sigma_{t_i}}\Big).
    \end{aligned}
\end{equation*}
Then  note that 
\begin{equation*}
    \begin{aligned}
        &(B_2^{'''})=\sum_{k\in \m K_2} \frac{\omega_k}{\max\left(1,\sum_{j=1}^{\wt J_k}\rho\big(\frac{|x-x^*_{kj}|}{c_4n^{-{1}/(2\alpha_k+d_k)}}\big)\right)}\\
          &\qquad\qquad\qquad\cdot \sum_{j=1}^{\wt J_k}\int_{\{z\in \mb R^{d_k}\,:\,\|z-\phi_{jk}^p(x)\|_{\infty}\leq c_5\wt\sigma_{t_i}\sqrt{\log n}\}}\exp\big(-\frac{\|x-m_tG_{jk}(z)\|^2}{2\wt \sigma_t^2}\big)  v_{kj}(z)\,\dd z\\
          &\qquad\qquad\qquad\qquad\cdot \rho\big(\frac{|x-\wt x^*_{kj}|}{c_4 n^{-{1}/(2\alpha_k+d_k)}}\big) \rho\big(\frac{|x-G_{kj}(\phi^p_{kj}(x))|}{c_6 \wt \sigma_{t_i}\sqrt{\log n}}\big)\\
          &=\sum_{k\in \m K_2}\sum_{j=1}^{\wt J_k} \frac{\omega_k \rho\big(\frac{|x-\wt x^*_{kj}|}{c_4 n^{-{1}/(2\alpha_k+d_k)}}\big) \rho\big(\frac{|x-G_{kj}(\phi^p_{kj}(x))|}{c_6 \wt \sigma_{t_i}\sqrt{\log n}}\big)}{\max\left(1,\sum_{j=1}^{\wt J_k}\rho\big(\frac{|x-x^*_{kj}|}{c_4n^{-{1}/(2\alpha_k+d_k)}}\big)\right)}\cdot \exp\big(-\frac{\|x-G_{kj}(\phi_{kj}^p(x))\|^2}{2\wt \sigma_t^2}\big)\\
          &\cdot  \int_{\{z\in \mb R^{d_k}\,:\,\|z-\phi_{jk}^p(x)\|_{\infty}\leq c_5\wt\sigma_{t_i}\sqrt{\log n}\}}\exp\big(-\frac{\|G_{kj}(\phi_{kj}^p(x))-m_tG_{jk}(z)\|^2}{2\wt \sigma_t^2}\big)\\
      &\qquad\cdot\exp\big(-\frac{\langle x-G_{kj}(\phi_{kj}^p(x)),G_{kj}(\phi_{kj}^p(x))-m_tG_{jk}(z)\rangle}{\wt \sigma_t^2}\big)  v_{kj}(z)\,\dd z.
    \end{aligned}
\end{equation*}
Now use the fact that 
\begin{equation*}
\begin{aligned}
    &\|G_{jk}(\phi^p_{jk}(x))-m_t G_{jk}(z)\|\\
    &\leq  \|G_{jk}(\phi^p_{jk}(x))- G_{jk}(z)\|+(1-m_t)\|G^*_{jk}(z)\|\\
&\lesssim \|\phi^p_{jk}(x)-z\|+o(\wt\sigma_{t_i}),
\end{aligned}
\end{equation*}
we have 
\begin{equation*}
    \begin{aligned}
        &\int_{\{z\in \mb R^{d_k}\,:\,\|z-\phi_{jk}^p(x)\|_{\infty}\leq c_5\wt\sigma_{t_i}\sqrt{\log n}\}}\exp\big(-\frac{\|G_{kj}(\phi_{kj}^p(x))-m_tG_{jk}(z)\|^2}{2\wt \sigma_t^2}\big)\\
      &\qquad\cdot\exp\big(-\frac{\langle x-G_{kj}(\phi_{kj}^p(x)),G_{kj}(\phi_{kj}^p(x))-m_tG_{jk}(z)\rangle}{\wt \sigma_t^2}\big)  v_{kj}(z)\,\dd z\\
        &\geq   \int_{\{z\in \mb R^{d_k}\,:\,\|z-\phi_{jk}^p(x)\|\leq  \wt\sigma_{t_i}\}}\exp\big(-\frac{\|G_{kj}(\phi_{kj}^p(x))-m_tG_{jk}(z)\|^2}{2\wt \sigma_t^2}\big)\\
      &\qquad\cdot\exp\big(-\frac{\langle x-G_{kj}(\phi_{kj}^p(x)),G_{kj}(\phi_{kj}^p(x))-m_tG_{jk}(z)\rangle}{\wt \sigma_t^2}\big)  v_{kj}(z)\,\dd z\\
        &\gtrsim \int_{\{z\in \mb R^{d_k}\,:\,\|z-\phi_{jk}^p(x)\|\leq  \wt\sigma_{t_i}\}} \exp\big(-\frac{\langle x-G_{kj}(\phi_{kj}^p(x)),G_{kj}(\phi_{kj}^p(x))-m_tG_{jk}(z)\rangle}{\wt \sigma_t^2}\big) \,\dd z\\
        &\gtrsim (\wt\sigma_{t_i})^{d_k}.
    \end{aligned}
\end{equation*}
Hence
\begin{equation*}
    \begin{aligned}
                &(B_2^{'''})\gtrsim\sum_{k\in \m K_2}\sum_{j=1}^{\wt J_k} \frac{\omega_k \rho\big(\frac{|x-\wt x^*_{kj}|}{c_4 n^{-{1}/(2\alpha_k+d_k)}}\big) \rho\big(\frac{|x-G_{kj}(\phi^p_{kj}(x))|}{c_6 \wt \sigma_{t_i}\sqrt{\log n}}\big)}{\max\left(1,\sum_{j=1}^{\wt J_k}\rho\big(\frac{|x-x^*_{kj}|}{c_4n^{-{1}/(2\alpha_k+d_k)}}\big)\right)}\cdot \exp\big(-\frac{\|x-G_{kj}(\phi_{kj}^p(x))\|^2}{2\wt \sigma_t^2}\big) \big(\wt\sigma_{t_i}\big)^{d_k}.
    \end{aligned}
\end{equation*}
 So we have  
\begin{equation*}
    \begin{aligned}
      &\frac{\big|  (B_2^{''})-(B_2^{'''})\big|}{B_2^{'''}}\lesssim\bigg(\sum_{k\in \m K_2}\sum_{j=1}^{\wt J_k} \frac{\omega_k \rho\big(\frac{|x-\wt x^*_{kj}|}{c_4 n^{-{1}/(2\alpha_k+d_k)}}\big) \rho\big(\frac{|x-G_{kj}(\phi^p_{kj}(x))|}{c_6 \wt \sigma_{t_i}\sqrt{\log n}}\big)}{\max\left(1,\sum_{j=1}^{\wt J_k}\rho\big(\frac{|x-x^*_{kj}|}{c_4n^{-{1}/(2\alpha_k+d_k)}}\big)\right)}\cdot \exp\big(-\frac{\|x-G_{kj}(\phi_{kj}^p(x))\|^2}{2\wt \sigma_t^2}\big) \big(\wt\sigma_{t_i}\big)^{d_k}\bigg)^{-1}
      \\
      &\cdot
      \sum_{k\in \m K_2}\sum_{j=1}^{\wt J_k} \frac{\omega_k  \rho\big(\frac{|x-\wt x^*_{kj}|}{c_4 n^{-{1}/(2\alpha_k+d_k)}}\big) \rho\big(\frac{|x-G_{kj}(\phi^p_{kj}(x))|}{c_6 \wt \sigma_{t_i}\sqrt{\log n}}\big)}{\max\left(1,\sum_{j=1}^{\wt J_k}\rho\big(\frac{|x-x^*_{kj}|}{c_4n^{-{1}/(2\alpha_k+d_k)}}\big)\right)}\cdot \exp\big(-\frac{\|x-G_{kj}(\phi_{kj}^p(x))\|^2}{2\wt \sigma_t^2}\big) \cdot (\wt\sigma_{t_i})^{d_k}\\
      &\qquad\qquad\qquad\cdot \Big(\frac{n^{-\frac{\beta_k}{2\alpha_k+d_k}}\sqrt{\log n}}{\wt\sigma_{t_i}}+ n^{-\frac{\alpha_k}{2\alpha_k+d}} \Big)\\
      &\lesssim \underset{k\in \m K_2}{\sup}\frac{n^{-\frac{\beta_k}{2\alpha_k+d_k}}\sqrt{\log n}}{\wt\sigma_{t_i}}+ n^{-\frac{\alpha_k}{2\alpha_k+d}}.
    \end{aligned}
\end{equation*}
So when  $t_i\geq \tau$ with
\begin{equation*}
\tau=c_1\sqrt{\log n}\cdot\left\{
\begin{array}{cc}
n^{-c},& \sigma_*> c_1 n^{-\min_{k} \frac{\alpha_k+1}{2\alpha_k+d_k}}\sqrt{\log n}\\
n^{-\min_{k} \frac{2\alpha_k+2}{2\alpha_k+d_k}},& \sigma_*\leq c_1n^{-\min_{k} \frac{\alpha_k+1}{2\alpha_k+d_k}}\sqrt{\log n},\\
\end{array}
\right.
 \end{equation*}
for a sufficiently large $c_1$, and using $\beta_k\geq \alpha_k+1$, we have $(B_2^{'''})\geq \frac{1}{2}(B_2^{''})$. Moreover, 
\begin{equation*}
    \begin{aligned}
      &\frac{\big\|  (A_2^{''})-(A_2^{'''})\big\|}{(B_2^{'''})}\lesssim\bigg(\sum_{k\in \m K_2}\sum_{j=1}^{\wt J_k} \frac{\omega_k \rho\big(\frac{|x-\wt x^*_{kj}|}{c_4 n^{-{1}/(2\alpha_k+d_k)}}\big) \rho\big(\frac{|x-G_{kj}(\phi^p_{kj}(x))|}{c_6 \wt \sigma_{t_i}\sqrt{\log n}}\big)}{\max\left(1,\sum_{j=1}^{\wt J_k}\rho\big(\frac{|x-x^*_{kj}|}{c_4n^{-{1}/(2\alpha_k+d_k)}}\big)\right)}\cdot \exp\big(-\frac{\|x-G_{kj}(\phi_{kj}^p(x))\|^2}{2\wt \sigma_t^2}\big) \big(\wt\sigma_{t_i}\big)^{d_k}\bigg)^{-1}
      \\
      &\cdot
      \sum_{k\in \m K_2}\sum_{j=1}^{\wt J_k} \frac{\omega_k  \rho\big(\frac{|x-\wt x^*_{kj}|}{c_4 n^{-{1}/(2\alpha_k+d_k)}}\big) \rho\big(\frac{|x-G_{kj}(\phi^p_{kj}(x))|}{c_6 \wt \sigma_{t_i}\sqrt{\log n}}\big)}{\max\left(1,\sum_{j=1}^{\wt J_k}\rho\big(\frac{|x-x^*_{kj}|}{c_4n^{-{1}/(2\alpha_k+d_k)}}\big)\right)}\cdot \exp\big(-\frac{\|x-G_{kj}(\phi_{kj}^p(x))\|^2}{2\wt \sigma_t^2}\big) \cdot (\wt\sigma_{t_i})^{d_k}\\
      &\qquad\qquad\qquad\cdot    \Big(\frac{n^{-\frac{\beta_k}{2\alpha_k+d_k}}\log n}{\wt\sigma^2_{t_i}}+ \frac{n^{-\frac{\alpha_k}{2\alpha_k+d}}\sqrt{\log n}}{\wt \sigma_{t_i}}\Big)\\
      &\lesssim   \sup_{k\in \m K_2} \frac{n^{-\frac{\beta_k}{2\alpha_k+d_k}}\log n}{\wt\sigma^2_{t_i}}+ \frac{n^{-\frac{\alpha_k}{2\alpha_k+d}}\sqrt{\log n}}{\wt \sigma_{t_i}}.
    \end{aligned}
\end{equation*}
 Then using $(B_1)+(B_2)\geq n^{-c_3}$. If $(B_1)\geq \frac{n^{-c_3}}{2}$, 
 \begin{equation*}
     \phi^*_{(B_1)}(x,t)\geq (B_1)-n^{-c_3-1} \geq \frac{n^{-c_3}}{4}.
 \end{equation*}
 Otherwise, we have  $(B_2)\geq \frac{n^{-c_3}}{2}$, 
 \begin{equation*}
     (B_2^{'''})\geq \frac{1}{2}(B_2^{''})\geq \frac{1}{2}\cdot((B_2)-2 n^{-c_3-1})\geq \frac{n^{-c_3}}{4}.
 \end{equation*}
 So $ \phi^*_{(B_1)}(x,t)+ (B_2^{'''})\geq \frac{n^{-c_3}}{4}.$ Then we can bound
 \begin{equation*}
     \begin{aligned}
         &\Big\|\frac{(A_1)+(A_2)}{(B_1)+(B_2)}-\frac{ \phi^*_{(A_1)}(x,t)+(A_2^{'''})}{ \phi^*_{(B_1)}(x,t)+(B_2^{'''})}\Big\|\\
         &\leq \Big\|\frac{(A_1)+(A_2)}{(B_1)+(B_2)}\Big\|\cdot \frac{\Big|\phi^*_{(B_1)}(x,t)+(B_2^{'''})-(B_1)-(B_2)\Big|}{\phi^*_{(B_1)}(x,t)+(B_2^{'''})}+\frac{\Big\| \phi^*_{(A_1)}(x,t)+(A_2^{'''})-(A_1)-(A_2)\Big\|}{\phi^*_{(B_1)}(x,t)+(B_2^{'''})}\\
         &\leq n^{-1}+\frac{\sqrt{\log n}}{\wt \sigma_{t_i}} \cdot \frac{\Big|(B_2^{'''})-(B_2^{''})\Big|}{(B_2^{'''})}+\frac{\Big\| (A_2^{'''})-(A_2^{''})\Big\|}{(B_2^{'''})}\\
         &\lesssim  \sup_{k\in \m K_2}  \frac{n^{-\frac{\beta_k}{2\alpha_k+d_k}}\log n}{\wt\sigma^2_{t_i}}+ \frac{n^{-\frac{\alpha_k}{2\alpha_k+d}}\sqrt{\log n}}{\wt \sigma_{t_i}}.
     \end{aligned}
 \end{equation*}
  Let $\ms L_1=C_1\log n$ and $\ms L_2=C_2$ for sufficiently large constants $C_1$ and $C_2$. Then we define 
\begin{equation*}
    \begin{aligned}
   &(A_2^*)=    \sum_{k\in \m K_2}\sum_{j=1}^{\wt J_k} \frac{\omega_k \rho\big(\frac{|x-\wt x^*_{kj}|}{c_4 n^{-{1}/(2\alpha_k+d_k)}}\big) \rho\big(\frac{|x-G_{kj}(\phi^p_{kj}(x))|}{c_6 \wt \sigma_{t_i}\sqrt{\log n}}\big)}{\max\left(1,\sum_{j=1}^{\wt J_k}\rho\big(\frac{|x-x^*_{kj}|}{c_4n^{-{1}/(2\alpha_k+d_k)}}\big)\right)}\cdot \exp\big(-\frac{\|x-G_{kj}(\phi_{kj}^p(x))\|^2}{2\wt \sigma_t^2}\big)\\
          &\cdot  \int_{\{z\in \mb R^{d_k}\,:\,\|z-\phi_{jk}^p(x)\|_{\infty}\leq c_5\wt\sigma_{t_i}\sqrt{\log n}\}} \sum_{l_1=0}^{\ms L_1} (-1)^{l_1} \frac{\|G_{kj}(\phi_{kj}^p(x))-m_tG_{jk}(z)\|^{2l_1}}{l_1! 2^{l_1} \wt\sigma_{t}^{2l_1}}\\
          &\qquad\qquad\qquad\cdot \sum_{l_2=0}^{\ms L_2}  (-1)^{l_2} \frac{\langle x-G_{kj}(\phi_{kj}^p(x)),G_{kj}(\phi_{kj}^p(x))-m_tG_{jk}(z)\rangle^{l_2}}{l_2! \wt\sigma_{t}^{2l_2}} \cdot \Big(-\frac{x-m_tG_{jk}(z)}{\wt \sigma_t^2}\Big) v_{kj}(z) \,\dd z.
    \end{aligned}
\end{equation*}
and 
\begin{equation*}
    \begin{aligned}
   &(B_2^*)=    \sum_{k\in \m K_2}\sum_{j=1}^{\wt J_k} \frac{\omega_k \rho\big(\frac{|x-\wt x^*_{kj}|}{c_4 n^{-{1}/(2\alpha_k+d_k)}}\big) \rho\big(\frac{|x-G_{kj}(\phi^p_{kj}(x))|}{c_6 \wt \sigma_{t_i}\sqrt{\log n}}\big)}{\max\left(1,\sum_{j=1}^{\wt J_k}\rho\big(\frac{|x-x^*_{kj}|}{c_4n^{-{1}/(2\alpha_k+d_k)}}\big)\right)}\cdot \exp\big(-\frac{\|x-G_{kj}(\phi_{kj}^p(x))\|^2}{2\wt \sigma_t^2}\big)\\
          &\cdot  \int_{\{z\in \mb R^{d_k}\,:\,\|z-\phi_{jk}^p(x)\|_{\infty}\leq c_5\wt\sigma_{t_i}\sqrt{\log n}\}} \sum_{l_1=0}^{\ms L_1} (-1)^{l_1} \frac{\|G_{kj}(\phi_{kj}^p(x))-m_tG_{jk}(z)\|^{2l_1}}{l_1! 2^{l_1} \wt\sigma_{t}^{2l_1}}\\
          &\qquad\qquad\qquad\cdot \sum_{l_2=0}^{\ms L_2}  (-1)^{l_2} \frac{\langle x-G_{kj}(\phi_{kj}^p(x)),G_{kj}(\phi_{kj}^p(x))-m_tG_{jk}(z)\rangle^{l_2}}{l_2! \wt\sigma_{t}^{2l_2}}  v_{kj}(z) \,\dd z.
    \end{aligned}
\end{equation*}
Then note that for any $k\in \m K_2$ and $j\in [\wt J_k]$, it holds that  $\frac{\|G_{kj}(\phi^p_{kj}(x))-m_t G_{jk}(z)\|^2}{\wt \sigma_t^2}\lesssim  \log n$  and $\frac{|\langle x-G_{kj}(\phi_{kj}^p(x)),G_{kj}(\phi_{kj}^p(x))-m_tG_{jk}(z)\rangle|}{\wt \sigma_t^2}\lesssim \wt \sigma_{t_i}(\log n)^{\frac{3}{2}}\lesssim  n^{-\min_{k\in [K]} \frac{1}{2\alpha_k+d_k}}\cdot\log n$. So using the fact that for any $x\geq 0$, it holds that $|\exp(-x)-\sum_{l=0}^{\ms L} (-1)^l\frac{x^l}{l!}|\leq \frac{|x|^{\ms L+1}}{(\ms L+1)!}\leq (\frac{|x|e}{(\ms L+1)})^{\ms L+1}$, and for any $|x|\leq 1$, $|\exp(-x)-\sum_{l=0}^{\ms L} (-1)^l\frac{x^l}{l!}|\leq e\frac{|x|^{\ms L+1}}{(\ms L+1)!}\leq e(\frac{|x|e}{(\ms L+1)})^{\ms L+1}$.
When  $\ms L_1=C_1\log n$ and $\ms L_2=C_2$ with sufficiently large constants $C_1$ and $C_2$, it holds that 
\begin{equation*}
    \|(A_2^*)-(A_2^{'''})\|\lesssim n^{-c_3-1}, 
\end{equation*}
and 
\begin{equation*}
    \frac{\sqrt{\log n}}{\wt \sigma_{t_i}}|(B_2^*)-(B_2^{'''})|\lesssim n^{-c_3-1}.
\end{equation*}
Then we approximate $\frac{ \phi^*_{(A_1)}(x,t)+(A_2^{*})}{ \phi^*_{(B_1)}(x,t)+(B_2^{*})}$ using ReLU neural network.  Since $G_{kj}(z)$ and $v_{kj}(z)$ are polynomials with degree at most $\lfloor \beta_k\rfloor$  and $\lfloor \alpha_k\rfloor$ respectively, we can rewrite $(A_2^*)$ and $(B_2^*)$ as
  \begin{equation*}
      \begin{aligned}
      &(A_2^{*})=  \sum_{k\in \m K_2}\sum_{j=1}^{\wt J_k} \frac{\omega_k \rho\big(\frac{|x-\wt x^*_{kj}|}{c_4 n^{-{1}/(2\alpha_k+d_k)}}\big) \rho\big(\frac{|x-G_{kj}(\phi^p_{kj}(x))|}{c_6 \wt \sigma_{t_i}\sqrt{\log n}}\big)}{\max\left(1,\sum_{j=1}^{\wt J_k}\rho\big(\frac{|x-x^*_{kj}|}{c_4n^{-{1}/(2\alpha_k+d_k)}}\big)\right)}\cdot \exp\big(-\frac{\|x-G_{kj}(\phi_{kj}^p(x))\|^2}{2\wt \sigma_t^2}\big)\\
          &\cdot \sum_{l_1=0}^{\ms L_1}\sum_{l_2=0}^{\ms L_2} \big(\frac{1}{\wt\sigma_t}\big)^{2l_1+2l_2+2} \sum_{0\leq s_1\leq 2l_1+l_2+1} m_t^{s_1}
\sum_{s_2\in \mb N_0^{d_k}, |s_2|\leq (2l_1+2l_2+1)\lfloor\beta_k\rfloor+d_k+\lfloor\alpha_k\rfloor}   (\phi_{jk}^p(x))^{(s_2)}\\
&\qquad\qquad \cdot\sum_{s_3\in \mb N_0^D, |s_3|\leq l_2+1} a'_{k,j,l_1,l_2,s_1,s_2,s_3} \cdot x^{(s_3)},
      \end{aligned}
  \end{equation*}
  and 
    \begin{equation*}
      \begin{aligned}
      &(B_2^{*})=  \sum_{k\in \m K_2}\sum_{j=1}^{\wt J_k} \frac{\omega_k \rho\big(\frac{|x-\wt x^*_{kj}|}{c_4 n^{-{1}/(2\alpha_k+d_k)}}\big) \rho\big(\frac{|x-G_{kj}(\phi^p_{kj}(x))|}{c_6 \wt \sigma_{t_i}\sqrt{\log n}}\big)}{\max\left(1,\sum_{j=1}^{\wt J_k}\rho\big(\frac{|x-x^*_{kj}|}{c_4n^{-{1}/(2\alpha_k+d_k)}}\big)\right)}\cdot \exp\big(-\frac{\|x-G_{kj}(\phi_{kj}^p(x))\|^2}{2\wt \sigma_t^2}\big)\\
          &\cdot \sum_{l_1=0}^{\ms L_1}\sum_{l_2=0}^{\ms L_2} \big(\frac{1}{\wt\sigma_t}\big)^{2l_1+2l_2} \sum_{0\leq s_1\leq 2l_1+l_2} m_t^{s_1}
\sum_{s_2\in \mb N_0^{d_k}, |s_2|\leq (2l_1+2l_2)\lfloor\beta_k\rfloor+d_k+\lfloor\alpha_k\rfloor}   (\phi_{jk}^p(x))^{(s_2)}\\
&\qquad\qquad \cdot\sum_{s_3\in \mb N_0^D, |s_3|\leq l_2} b'_{k,j,l_1,l_2,s_1,s_2,s_3} \cdot x^{(s_3)},
      \end{aligned}
  \end{equation*}
where $a'_{k,j,l_1,l_2,s_1,s_2,s_3}\in \mb R^D$ and $b'_{k,j,l_1,l_2,s_1,s_2,s_3}\in \mb R$  are some constant coefficients satisfying $(\frac{1}{\wt \sigma_t})^{2l_1+2l_2+1}a'_{k,j,l_1,l_2,s_1,s_2,s_3}\lesssim\exp(\m O(\log^2 n))$ and $(\frac{1}{\wt \sigma_t})^{2l_1+2l_2}b'_{k,j,l_1,l_2,s_1,s_2,s_3}\lesssim\exp(\m O(\log^2 n))$. Then  we will use  Lemmas~\ref{lemma3.3},~\ref{LemmaF.6} and~\ref{lemmaF.7} for the approximation $m_t$, $\sigma_t$, monomial and reciprocal function.  Furthermore, we will use Lemma F.12 in~\cite{oko2023diffusion} for the approximation of exponential function.
\begin{lemma}\label{LemmaF.12}
     (Lemma F.12 in~\cite{oko2023diffusion}) Take $\varepsilon>0$ arbitrarily. There exists a neural network $\phi_{\exp} \in \Phi(L, W, R, B)$ such that
$$
\sup _{x, x^{\prime} \geq 0}\left|e^{-x^{\prime}}-\phi_{\exp }(x)\right| \leq \varepsilon+\left|x-x^{\prime}\right|
$$
holds, where $L=\mathcal{O}\left(\log ^2 \varepsilon^{-1}\right),\|W\|_{\infty}=\mathcal{O}\left(\log \varepsilon^{-1}\right), R=\mathcal{O}\left(\log ^2 \varepsilon^{-1}\right), B=\exp \left(\mathcal{O}\left(\log ^2 \varepsilon^{-1}\right)\right)$. Moreover, $\left|\phi_{\exp }(x)\right| \leq$ $\varepsilon$ for all $x \geq \log 3 \varepsilon^{-1}$.
 \end{lemma}
Then we consider $\phi_{m}(t),\phi_{\sigma}(t),\phi_{rec}(x),\phi^{[D]}_{vpower}(x;s),\phi_{power}(x;a),\phi_{mult}(x,y), \phi_{\frac{1}{\wt \sigma}}(t), \phi_{\rho}(x;x^*,a)$ defined in Case 1. Moreover, we approximate $\exp(-x)$ by $\phi_{\exp} \in \Phi(L', W', R', B')$ with $L'=\m O(\log^4 n)$, $\|W'\|_{\infty}=\m O(\log^2 n)$, $R'=\m O(\log^4 n)$ and $B'=\exp(\m O(\log^4 n))$. For $k\in \m K_2$, $z\in \mb R^{d_k}$ and $s\in \mb N^{d_k}$, we approximate $z^s$ by $\phi^{[d_k]}_{vpower}(z;s)\in \Phi(L'_1,W'_1,R'_1,B'_1)$ with $L'_1=\Theta(\log^2 n\cdot\log\log n)$, $\|W'_1\|_{\infty}=\Theta(\log n)$, $R'_1=\Theta(\log^3 n)$ and $B'_1=\exp(\Theta(\log n\cdot \log \log n))$. Furthermore, we define 
\begin{equation*}
    \phi_{G_{kj}}(z)=\wt x_{kj}^*+\sum_{s\in \mb N_0^{d_k}\atop 1\leq |s|\leq \lfloor \beta_k\rfloor} \phi_{mult}\Big(\phi^{[d_k]}_{vpower}(z;s),G^{*}_{kj}{}^{(s)}(0_{d_k})\Big).
\end{equation*}
We construct 
\begin{equation*}
\begin{aligned}
       & \phi^*_{(A)}(x,t)=\phi_{(A_1)}^*(x,t)+\sum_{k\in \m K_2}\phi_{mult}\Bigg(\phi_{rec}\Big(\max\big(1,\sum_{j=1}^{\wt J_k}\phi_\rho(x;x_{kj}^*,c_4 n^{-\frac{1}{2\alpha_k+d_k}}\big)\Big),\\
        & \sum_{j=1}^{\wt J_k} \sum_{l_1=0}^{\ms L_1}\sum_{l_2=0}^{\ms L_2}   \sum_{0\leq s_1\leq 2l_1+l_2+1}  \sum_{s_2\in \mb N_0^{d_k}, |s_2|\leq (2l_1+2l_2+1)\lfloor\beta_k\rfloor+d_k+\lfloor\alpha_k\rfloor} \sum_{s_3\in \mb N_0^D, |s_3|\leq l_2}  \omega_k \cdot a'_{k,j,l_1,l_2,s_1,s_2,s_3}\\
        & \cdot \phi_{mult}\Bigg(\phi_{mult}\bigg(\phi_{mult}\bigg(\phi_{mult}\Big(\phi_{mult}\big(\phi_{power}(\phi_{\frac{1}{\wt\sigma^2}}(t),l_1+l_2+1),\phi_{power}(\phi_{m}(t),s_1)\big), \\
        &\phi^{[d_k]}_{vpower}(\phi_{kj}^p(x);s_2)\Big),  \phi^{[D]}_{vpower}(x;s_3)\bigg), \phi_\rho(x;x_{kj}^*,c_4n^{-\frac{1}{2\alpha_k+d_k}})\bigg),\phi_\rho\big(x;\phi_{G_{kj}}(\phi_{kj}^p(x)),c_6\wt \sigma_{t_i}\sqrt{\log n}\big)\Bigg)
      \Bigg),
\end{aligned}
\end{equation*}
and 
\begin{equation*}
\begin{aligned}
        \phi^*_{(B)}(x,t)&=\max\Bigg(\frac{n^{-c_3}}{4},\phi_{(B_1)}^*(x,t)+\sum_{k\in \m K_2}\phi_{mult}\Bigg(\phi_{rec}\Big(\max\big(1,\sum_{j=1}^{\wt J_k}\phi_\rho(x;x_{kj}^*,c_4 n^{-\frac{1}{2\alpha_k+d_k}}\big)\Big),\\
        & \sum_{j=1}^{\wt J_k} \sum_{l_1=0}^{\ms L_1}\sum_{l_2=0}^{\ms L_2}   \sum_{0\leq s_1\leq 2l_1+l_2}  \sum_{s_2\in \mb N_0^{d_k}, |s_2|\leq (2l_1+2l_2)\lfloor\beta_k\rfloor+d_k+\lfloor\alpha_k\rfloor} \sum_{s_3\in \mb N_0^D, |s_3|\leq l_2}  \omega_k \cdot b'_{k,j,l_1,l_2,s_1,s_2,s_3}\\
        & \cdot \phi_{mult}\Bigg(\phi_{mult}\bigg(\phi_{mult}\bigg(\phi_{mult}\Big(\phi_{mult}\big(\phi_{power}(\phi_{\frac{1}{\wt\sigma^2}}(t),l_1+l_2),\phi_{power}(\phi_{m}(t),s_1)\big), \\
        &\phi^{[d_k]}_{vpower}(\phi_{kj}^p(x);s_2)\Big),  \phi^{[D]}_{vpower}(x;s_3)\bigg), \phi_\rho(x;x_{kj}^*,c_4n^{-\frac{1}{2\alpha_k+d_k}})\bigg),\phi_\rho\big(x;\phi_{G_{kj}}(\phi_{kj}^p(x)),c_6\wt \sigma_{t_i}\sqrt{\log n}\big)\Bigg)
      \Bigg)\Bigg).
\end{aligned}
\end{equation*}
Then using  Lemmas~\ref{lemma3.3},~\ref{LemmaF.6},~\ref{lemmaF.7} and~\ref{LemmaF.12}, as well as $ \phi^*_{(B_1)}(x,t)+ (B_2^{'''})\geq \frac{n^{-c_3}}{4}$,  we can get 
\begin{equation*}
    \| \phi^*_{(A)}(x,t)-\phi_{(A_1)}^*(x,t)-(A_2^{'''})\|\lesssim  n^{-c_3-1},
\end{equation*}
and 
\begin{equation*}
      \frac{\sqrt{\log n}}{\wt \sigma_{t_i}}|\phi^*_{(B)}(x,t)-\phi_{(B_1)}^*(x,t)-(B_2^{'''})|\lesssim    n^{-c_3-1}.
\end{equation*}
So
\begin{equation*}
\begin{aligned}
       &\Big\|\frac{\phi^*_{(A)}(x,t)}{\phi^*_{(B)}(x,t)}-\nabla \log p_t(x)\Big\|\\
       &\leq   \Big\|\frac{\phi^*_{(A)}(x,t)}{\phi^*_{(B)}(x,t)}-\frac{ \phi^*_{(A_1)}(x,t)+(A_2^{'''})}{ \phi^*_{(B_1)}(x,t)+(B_2^{'''})}\Big\|+\Big\|\frac{(A_1)+(A_2)}{(B_1)+(B_2)}-\frac{ \phi^*_{(A_1)}(x,t)+(A_2^{'''})}{ \phi^*_{(B_1)}(x,t)+(B_2^{'''})}\Big\|\\
       &\lesssim  \sup_{k\in \m K_2}  \frac{n^{-\frac{\beta_k}{2\alpha_k+d_k}}\log n}{\wt\sigma^2_{t_i}}+ \frac{n^{-\frac{\alpha_k}{2\alpha_k+d}}\sqrt{\log n}}{\wt \sigma_{t_i}}.
\end{aligned}
    \end{equation*}
    Finally, define
\begin{equation*}
    \phi^*(x,t)=\max\bigg(-c_2\frac{\sqrt{\log n}}{\wt \sigma_{t_i}},\min\bigg(c_2\frac{\sqrt{\log n}}{\wt \sigma_{t_i}},\phi_{mult}\Big(\phi_{rec}\big( \phi^*_{(B)}(x,t)\big), \phi^*_{(A)}(x,t)\Big)\bigg)\bigg).
\end{equation*}
Since $\|\nabla \log p_t(x)\|\leq c_2\frac{\sqrt{\log n}}{\wt \sigma_{t_i}}$, it holds for any $t\in [t_i,t_{i+1}]$ and $x\in\mb R^D$ with ${\rm dist}(x,\m M)\leq c_0\wt \sigma_{t_i}\sqrt{\log n}$,
\begin{equation*}
    \begin{aligned}
       & \big\| \phi^*(x,t)-\nabla \log p_t(x)\big\|\lesssim \sup_{k\in \m K_2}  \frac{n^{-\frac{\beta_k}{2\alpha_k+d_k}}\log n}{\wt\sigma^2_{t_i}}+ \frac{n^{-\frac{\alpha_k}{2\alpha_k+d}}\sqrt{\log n}}{\wt \sigma_{t_i}}.
    \end{aligned}
\end{equation*}
Furthermore, based on Lemmas F.1-F.3 in~\cite{oko2023diffusion} for the concatenation and parallelization of neural networks, there exists $L_i,W_i,R_i,B_i$ with $L_i=\Theta(\log^4 n)$, $\|W_i\|_{\infty}=\wt\Theta\big(\sum_{k=1}^K \wt J_k)$, $R_i=\wt \Theta\big(\sum_{k=1}^K \wt J_k)$, and $B_i=\exp(\Theta(\log^4 n))$ so that $\phi^*\in \Phi(L_i,W_i,R_i,B_i,c_2\frac{\sqrt{\log n}}{\wt \sigma_{t_i}})$. So we have 
 \begin{equation*}
     \begin{aligned}
& \mb E\left[\int_{t_i}^{t_{i+1}} \int_{\mathbb{R}^D}\left\|\wh S(x, t)-\nabla \log p_t(x)\right\|^2 p_t(x)\, \dd x \dd t\right] \\
& \lesssim (t_{i+1}- t_i)\cdot  \sup_{t\in [t_i,t_{i+1}]}\sup_{x\,:\, \operatorname{dist}(x, \m M) \leq c_0\wt\sigma_{t_i}\sqrt{\log n}}\left\|\phi^*(x, t)-\nabla \log p_t(x)\right\|^2+\frac{(\log n)^2}{n} R_i L_i \log \left(n  L_i  \left\|W_i\right\|_\infty  B_i\right)+\frac{1}{n^2}\\
&=\wt{\m O}\Big(\sum_{k\in \m K_2}  t_i\frac{n^{-\frac{2\beta_k}{2\alpha_k+d_k}}}{\wt\sigma_{t_i}^4}+ t_i\frac{n^{-\frac{2\alpha_k}{2\alpha_k+d_k}}}{\wt\sigma_{t_i}^2}\Big)=\wt{\m O}\Big(\sum_{k\in \m K_2} \frac{n^{-\frac{2\beta_k}{2\alpha_k+d_k}}}{\sigma_*^2+t_i}+  n^{-\frac{2\alpha_k}{2\alpha_k+d_k}}\Big).
\end{aligned}
\end{equation*}
Then note that $\wt \sigma_{t_i}\leq \min_{k\in \m K_2} \frac{n^{-\frac{1}{2\alpha_k+d_k}}}{\sqrt{\log n}}$, we have
\begin{equation*}
  \frac{(\sigma_*+\sqrt{1\wedge t_i})^{-(d_k\vee 2)}}{n}\gtrsim\frac{(\sigma_*+\sqrt{1\wedge t_i})^{-d_k}}{n}\gtrsim   \frac{(\wt \sigma_{t_i})^{-d_k}}{n}\gtrsim \max_{k\in \m K_2} n^{-\frac{2\alpha_k}{2\alpha_k+d_k}}(\log n)^{\frac{d_k}{2}},
\end{equation*}
and  if $d_k\geq 2$,
\begin{equation*}
    \begin{aligned}
        \frac{(\sigma_*+\sqrt{1\wedge t_i})^{-d_k}}{n}\gtrsim         \frac{(\sigma_*+\sqrt{1\wedge t_i})^{-d_k+2}}{n (\sigma_*^2+t_i)}\gtrsim \frac{\max_{k\in \m K_2} n^{-\frac{2\alpha_k+2}{2\alpha_k+d_k}}(\log n)^{\frac{d_k-2}{2}}}{\sigma_*^2+t_i} \gtrsim \frac{\max_{k\in \m K_2} n^{-\frac{2\beta_k}{2\alpha_k+d_k}}(\log n)^{\frac{d_k-2}{2}}}{\sigma_*^2+t_i},
    \end{aligned}
\end{equation*}
 if $d_k=1$,
 \begin{equation*}
    \begin{aligned}
        \frac{(\sigma_*+\sqrt{1\wedge t_i})^{-2}}{n}\gtrsim         \frac{1}{n (\sigma_*^2+t_i)} \gtrsim \frac{\max_{k\in \m K_2} n^{-\frac{2\beta_k}{2\alpha_k+1}}}{\sigma_*^2+t_i},
    \end{aligned}
\end{equation*}
  Hence, we have 
 \begin{equation*}
     \begin{aligned}
& \mb E\left[\int_{t_i}^{t_{i+1}} \int_{\mathbb{R}^D}\left\|\wh S(x, t)-\nabla \log p_t(x)\right\|^2 p_t(x)\, \dd x \dd t\right] \\
&=\wt{\m O}\Big(\sum_{k\in \m K_2} \frac{n^{-\frac{2\beta_k}{2\alpha_k+d_k}}}{\sigma_*^2+t_i}+  n^{-\frac{2\alpha_k}{2\alpha_k+d_k}}\Big)\\
&=\wt{\m O}\Big(\sum_{k\in \m K_2}\min(\frac{n^{-\frac{2\beta_k}{2\alpha_k+d_k}}}{\sigma_*^2+t_i}+n^{-\frac{2\alpha_k}{2\alpha_k+d_k}}, \frac{(\sigma_*+\sqrt{1\wedge t_i})^{-(d_k\vee 2)}}{n})\Big)\\
&=\wt{\m O}\Big(\sum_{k=1}^K \min(\frac{n^{-\frac{2\beta_k}{2\alpha_k+d_k}}}{\sigma_*^2+t_i}+n^{-\frac{2\alpha_k}{2\alpha_k+d_k}}, \frac{(\sigma_*+\sqrt{1\wedge t_i})^{-(d_k\vee 2)}}{n})\Big),
\end{aligned}
\end{equation*}
and using $\beta_k\geq \alpha_k+1$,
  \begin{equation*}
     \begin{aligned}
&\mb E\left[((t_i\log n)\wedge 1)\cdot\int_{t_i}^{t_{i+1}} \int_{\mathbb{R}^D}\left\|\wh S(x, t)-\nabla \log p_t(x)\right\|^2 p_t(x)\, \dd x \dd t\right] \\
&=\wt{\m O}\Big(\sum_{k\in \m K_2}\min(n^{-\frac{2\beta_k}{2\alpha_k+d_k}}+t_i\cdot n^{-\frac{2\alpha_k}{2\alpha_k+d_k}}, t_i\frac{(\sigma_*+\sqrt{1\wedge t_i})^{-(d_k\vee 2)}}{n})\Big)\\
&=\wt{\m O}\Big(\frac{1}{n}+\sum_{k=1}^K \min(n^{-\frac{2\alpha_k+2}{2\alpha_k+d_k}}, \frac{\sigma_*^{2-d_k}}{n})\Big).
\end{aligned}
\end{equation*}
\subsection{Summarizing the results}
Finally, by combining the results in case 1 and 2, we have for any $i\in [L]$, 
 \begin{equation*}
     \begin{aligned}
& \mb E\left[\int_{t_i}^{t_{i+1}} \int_{\mathbb{R}^D}\left\|\wh S(x, t)-\nabla \log p_t(x)\right\|^2 p_t(x)\, \dd x \dd t\right] \\
&=\wt{\m O}\Big(\sum_{k\in \m K_2} \frac{n^{-\frac{2\beta_k}{2\alpha_k+d_k}}}{\sigma_*^2+t_i}+  n^{-\frac{2\alpha_k}{2\alpha_k+d_k}}\Big)\\
&=\wt{\m O}\Big(\sum_{k\in \m K_2}\min(\frac{n^{-\frac{2\beta_k}{2\alpha_k+d_k}}}{\sigma_*^2+t_i}+n^{-\frac{2\alpha_k}{2\alpha_k+d_k}}, \frac{(\sigma_*+\sqrt{1\wedge t_i})^{-(d_k\vee 2)}}{n})\Big)\\
&=\wt{\m O}\Big(\sum_{k=1}^K \min(\frac{n^{-\frac{2\beta_k}{2\alpha_k+d_k}}}{\sigma_*^2+t_i}+n^{-\frac{2\alpha_k}{2\alpha_k+d_k}}, \frac{(\sigma_*+\sqrt{1\wedge t_i})^{-(d_k\vee 2)}}{n})\Big),
\end{aligned}
\end{equation*}
and using Lemma~\ref{lemma1},  if $\sigma_*>  n^{-\min_{k} \frac{\alpha_k+1}{2\alpha_k+d_k}}\sqrt{\log n}$, we have $\tau\leq \frac{1}{n}\log n$ and
  \begin{equation*}
  \begin{aligned}
\mb{E}[W_1\left(\wh p, P_*\right)]& \lesssim \frac{1}{n}+\tau^{\frac{1}{2}}+
\sum_{i=0}^{L-1} \sqrt{\mb{E}\left[\big((t_i \log n)\wedge 1 \big)\int_{t_i}^{t_{i+1}}\int_{\mb R^D}\left\|\wh S(x,t)-\nabla \log p_t(x)\right\|^2 p_t(x)\,\dd x\,\dd t\right]}  \\
&= \wt{\m O}\Big(\frac{1}{\sqrt{n}}+\sum_{k=1}^K \min(n^{-\frac{\alpha_k+1}{2\alpha_k+d_k}}, \frac{\sigma_*^{1-\frac{d_k}{2}}}{\sqrt{n}})\Big).
  \end{aligned}
    \end{equation*}
if $\sigma_*\leq  n^{-\min_{k} \frac{\alpha_k+1}{2\alpha_k+d_k}}\sqrt{\log n}$, we have $\tau\asymp n^{-\min_{k} \frac{2\alpha_k+2}{2\alpha_k+d_k}}\cdot\log n$ and
\begin{equation*}
    \frac{\sigma_*^{1-\frac{d_k}{2}}}{\sqrt{n}}+\frac{1}{\sqrt{n}}\geq \frac{1}{\sqrt{n}}+n^{-\min_k\frac{\alpha_k+1}{2\alpha_k+d_k}}.
\end{equation*}
Hence
  \begin{equation*}
  \begin{aligned}
\mb{E}[W_1\left(\wh p, P_*\right)]& \lesssim \frac{1}{n}+\tau^{\frac{1}{2}}+
\sum_{i=0}^{L-1} \sqrt{\mb{E}\left[\big((t_i \log n)\wedge 1 \big)\int_{t_i}^{t_{i+1}}\int_{\mb R^D}\left\|\wh S(x,t)-\nabla \log p_t(x)\right\|^2 p_t(x)\,\dd x\,\dd t\right]}  \\
&= \wt{\m O}\Big(\frac{1}{\sqrt{n}}+\sum_{k=1}^K n^{-\frac{\alpha_k+1}{2\alpha_k+d_k}}\Big)\\
&= \wt{\m O}\Big(\frac{1}{\sqrt{n}}+\sum_{k=1}^K \min(n^{-\frac{\alpha_k+1}{2\alpha_k+d_k}}, \frac{\sigma_*^{1-\frac{d_k}{2}}}{\sqrt{n}})\Big).
  \end{aligned}
    \end{equation*}

\subsection{Proof for Lemma~\ref{lemma2.1}}
Since $\m M\subset \mb B_1(0_D)$, for any $x\in \mb R^D$, we have
\begin{equation*}
    \begin{aligned}
        \|\nabla \log p_t(x)\|&=\left\|\frac{\nabla p_t(x)}{p_t(x)}\right\|\\
         &=\left\|\frac{\sum_{k=1}^K \omega_k\cdot\mb{E}_{y\sim Q_k}\Big[\exp\big(-\frac{\|x-m_ty\|^2}{2\wt \sigma_t^2}\big)\cdot \Big(-\frac{x-m_ty}{\wt \sigma_t^2}\Big)\Big]}{\sum_{k=1}^K \omega_k\cdot\mb{E}_{y\sim Q_k}\Big[\exp\big(-\frac{\|x-m_ty\|^2}{2\wt \sigma_t^2}\big)\Big]}\right\|\\
&\leq \frac{\|x\|+\sqrt{D}}{\wt\sigma_t^2}.
    \end{aligned}
\end{equation*}
Therefore,  for any constant $c_1>0$,
\begin{equation*}
    \begin{aligned}
        & \int_{\mb R^D} \|\nabla \log p_t(x)\|^2 p_t(x) \cdot 1\left(\operatorname{dist}(x, \m M) \geq c_0\wt\sigma_{t_i}\sqrt{\log n}\right)\,\dd x\\
        &\leq  \frac{1}{(2\pi\wt \sigma_t^2)^{\frac{D}{2}}}  \sum_{k=1}^K \omega_k\cdot\mb{E}_{y\sim Q_k}\Big[\int_{\mb R^D} \frac{\|x\|+\sqrt{D}}{\wt\sigma_t^2}  \exp\big(-\frac{\|x-m_ty\|^2}{2\wt \sigma_t^2}\big)\\
        &\qquad\cdot\left( 1\left(\operatorname{dist}(x, \m M) \geq c_0\wt\sigma_{t_i}\sqrt{\log n}, \|x\|\geq c_1\sqrt{\log n}\right)+ 1\left(\operatorname{dist}(x, \m M) \geq c_0\wt\sigma_{t_i}\sqrt{\log n}, \|x\|< c_1\sqrt{\log n}\right)\right)\,\dd x\Big].
         \end{aligned}
\end{equation*}
 Note that for large enough $c_1$, using $\m M\subset \mb B_1(0_D)$ and $\sigma_{\tau}\leq \sigma_{t_i}\leq \wt \sigma_t\leq \sqrt{\sigma_*^2+\sigma_t^2}\leq \sqrt{2}$, we have for any $t\in [t_i,t_{i+1}]$,
\begin{equation*}
    \begin{aligned}
    & \frac{1}{(2\pi\wt \sigma_t^2)^{\frac{D}{2}}}  \sum_{k=1}^K \omega_k\cdot\mb{E}_{y\sim Q_k}\Big[\int_{\mb R^D} \frac{\|x\|+\sqrt{D}}{\wt\sigma_t^2}  \exp\big(-\frac{\|x-m_ty\|^2}{2\wt \sigma_t^2}\big)\\
        &\qquad\cdot  1\left(\operatorname{dist}(x, \m M) \geq c_0\wt\sigma_{t_i}\sqrt{\log n}, \|x\|>c_1\sqrt{\log n}\right)\,\dd x\Big]\\
      & \leq \frac{1}{(2\pi\wt \sigma_t^2)^{\frac{D}{2}}}  \sum_{k=1}^K \omega_k\cdot\mb{E}_{y\sim Q_k}\Big[\int_{\mb R^D} \frac{\|x\|+\sqrt{D}}{\wt\sigma_t^2}  \exp\big(-\frac{\|x-m_ty\|^2}{2\wt \sigma_t^2}\big) \cdot  1\left(\|x\|> c_1\sqrt{\log n}\right)\,\dd x\Big]\\
        &\leq \frac{1}{n^2}.
    \end{aligned}
\end{equation*}
Moreover, for large enough $c_0$, we have 
\begin{equation*}
    \begin{aligned}
    & \frac{1}{(2\pi\wt \sigma_t^2)^{\frac{D}{2}}}  \sum_{k=1}^K \omega_k\cdot\mb{E}_{y\sim Q_k}\Big[\int_{\mb R^D} \frac{\|x\|+\sqrt{D}}{\wt\sigma_t^2}  \exp\big(-\frac{\|x-m_ty\|^2}{2\wt \sigma_t^2}\big)\\
        &\qquad\cdot  1\left(\operatorname{dist}(x, \m M) \geq c_0\wt\sigma_{t_i}\sqrt{\log n}, \|x\|\leq c_1\sqrt{\log n}\right)\,\dd x\Big]\\
         &\lesssim \frac{c_1\sqrt{\log n}+D}{\wt\sigma_t^2}\frac{1}{(2\pi\wt\sigma_t^2)^{\frac{D}{2}}}\cdot \exp\big(-\frac{c_0^2\wt\sigma_{t_i}^2}{4\wt\sigma_t^2}\log n) \int_{\mb R^{D}}  1\left( \|x\|\leq c_1\sqrt{\log n}\right)\, \dd x\\
         &\leq \frac{1}{n^2}.
    \end{aligned}
\end{equation*}
Therefore, for sufficiently large $c_0$, we have
\begin{equation*}
            \begin{aligned}
 \int\left\|\nabla \log p_t(x)\right\|^2 p_t(x) \cdot 1\left(\operatorname{dist}(x, \m M) \geq c_0\wt\sigma_{t_i}\sqrt{\log n}\right) \,\dd x \leq  c_1\frac{1}{n^2}.
            \end{aligned}
        \end{equation*}
Similarly, we can show 
\begin{equation*}
            \begin{aligned}
 &\int\left\| S(x,t)\right\|^2 p_t(x) \cdot 1\left(\operatorname{dist}(x, \m M) \geq c_0\wt\sigma_{t_i}\sqrt{\log n}\right) \,\dd x\\
 &\leq  \int c^2 \frac{\log n}{\wt\sigma_t^2} p_t(x) \cdot 1\left(\operatorname{dist}(x, \m M) \geq c_0\wt\sigma_{t_i}\sqrt{\log n}\right) \,\dd x\leq c^2 c_1\frac{1}{n^2}.
        \end{aligned}
        \end{equation*}
The first statement  is then proved. For the second statement, denote $\kappa: \mb R^{D_x}\to [K]$ so that $\kappa(x)\in\arg\min_{k\in [K]}{\rm dist}(x,\m M_k)$, and denote
${\rm Proj}_{\m M_k}(x)$ as any point inside ${\arg\min}_{y\in \m M_k}\|x-y\|$. Then for any $x\in \mb R^D$ with $\operatorname{dist}(x, \m M) \leq c_0\wt\sigma_{t_i}\sqrt{\log n}$, we have 
\begin{equation*}
    \begin{aligned}
        (2\pi\wt\sigma_t^2)^{\frac{D}{2}}p_t(x)&\geq
        \omega_{\kappa(x)}\cdot        \int_{y\in \mb B_{\wt\sigma_t}({\rm Proj}_{\m M_{\kappa(x)}}(x))\cap \m M_{\kappa(x)}}\exp\bigg(-\frac{\|x-m_ty\|^2}{2\wt \sigma_t^2}\bigg)  q_{\kappa(x)}(y)\,\dd y\\
        &\geq C \exp(-\frac{(c_0\wt \sigma_{t_i}\sqrt{\log n}+\wt\sigma_t+(1-m_t))^2}{2\wt\sigma_t^2})\wt\sigma_t^d\\
        &\geq n^{-c_2}.
    \end{aligned}
\end{equation*}
Therefore, there exists a constant $c_0'$ so that for any $x\in \mb R^D$ with $\operatorname{dist}(x, \m M) \leq c_0\sigma_{t_i}\sqrt{\log n}$,
\begin{equation*}
\begin{aligned}
    \|\nabla \log p_t(x)\|&\leq  \left\|\frac{ \sum_{k=1}^K \omega_k\cdot\mb{E}_{y\sim Q_k}\Big[\exp\big(-\frac{\|x-m_ty\|^2}{2\wt \sigma_t^2}\big)\cdot \Big(-\frac{x-m_ty}{\wt \sigma_t^2}\Big)\cdot \mathbf{1}\left(\|x-m_ty\|\leq c_0'\wt\sigma_t\sqrt{\log n}\right)\Big]
    }{ \sum_{k=1}^K \omega_k\cdot\mb{E}_{y\sim Q_k}\Big[\exp\big(-\frac{\|x-m_ty\|^2}{2\wt \sigma_t^2}\big)\cdot \mathbf{1}\left(\|x-m_ty\|\leq c_0'\wt\sigma_t\sqrt{\log n}\right)\Big]}\right\| +\frac{1}{n}\\
    &\lesssim \frac{\sqrt{\log n}}{\wt\sigma_t}\asymp \frac{\sqrt{\log n}}{\wt\sigma_{t_i}}.
    \end{aligned}
\end{equation*}
We can then get the desired statement by combining all pieces.

\section{Proofs for Section 5}

\subsection{Proof of Theorem~\ref{th:normalalign}}

Let $m_t=\exp(-\int_0^t\beta_sds)$ and $\sigma_t^2=1-m_t^2$ and note that $\sigma_t \to 0$ as $t\to 0^+$. The marginal density at a particular time $t >0$ can be written as
$$p_t(x) = \int_{\mathcal{S}} \mathcal{N}(x;m_tx_0,\sigma_t^2I)dP_*(x_0) = \int_{\m S} \phi_{\sigma_t}(x-m_tx_0)dP_*(x_0).$$
Hence, by differentiating under the integral,
$$\nabla p_t(x) = \int_{\m S} -\frac{x-m_tx_0}{\sigma_t^2} \phi_{\sigma_t}(x-m_tx_0)dP_*(x_0)$$
Thus 
\begin{align*}
    \nabla \log p_t(x) &= \frac{ \nabla p_t(x)}{p_t(x)}=\frac{\int -(\frac{x-m_tx_0}{\sigma_t^2})\phi_{\sigma_t}(x-m_tx_0)dP_*(x_0)}{\int \phi_{\sigma_t}(x-m_tx_0)dP_*(x_0)}\\
    &=-\frac{1}{\sigma_t^2}\Big(x-m_t\frac{\int x_0\phi_{\sigma_t}(x-m_tx_0)dP_*(x_0)}{\int \phi_{\sigma_t}(x-m_tx_0)dP_*(x_0)}\Big)
\end{align*}
First consider the case where $\m S=M_\ell$. Let $n_t(x)$ be the expression in parentheses above. By assumption $x$ is sufficiently close to $\m S$, so $x$ is in a tubular region of $M_\ell$ where $\pi_\ell$ is unique.  Recall Laplace's method, which states that if $f$ and $g$ are smooth functions and $f$ achieves a unique minimum at a point $y_*$ in the interior of a region $U \subset \mathbb{R}^d$, then
$$\int_U g(y)\exp(-Mf(y))\,dx \approx \Big(\frac{2\pi}{M}\Big)^{d/2} \frac{g(y_*)\exp(-Mf(y_*))}{|\text{Hess}_f(y_*)|^{1/2}} \quad \text{as} \quad M \to \infty.$$
One extends this to manifolds by using the intrinsic Hessian. We can hence express the integrals in our setting as $t\to 0$ by
$$\int_{M_\ell} g(y)q_\ell(y)(2\pi \sigma_t^2)^{-D/2}\exp(-f_t(y)/\sigma_t^2)d\text{vol}_{M_\ell}(y) \approx (2\pi\sigma_t^2)^{-(D-d_\ell)/2}\frac{g(y_{*,t})q_\ell(y_{*,t})\exp(-f_t(y_{*,t})/\sigma_t^2)}{\sqrt{\det(\text{Hess}_{M_\ell}f_t(y_{*,t}))}},$$
where $q_\ell$ is the density of $Q_\ell$ with respect to the volume measure on $M_\ell$, $f_t(y)=\frac{1}{2}\|x-m_ty\|^2$, and at the minimizer $y_{*,t}$, the intrinsic Hessian is positive definite on $T_{y_{*,t}}M$.

For simplicity in the general stratified case, we combine the geometric and local density terms in the expression above into the function
$$G_{\ell}(x,t) = (2\pi\sigma_t^2)^{-(D-d_\ell)/2}\frac{q_\ell(y_{*,t})\exp(-f_t(y_{*,t})/\sigma_t^2)}{\sqrt{\det(\text{Hess}_{M_\ell}f_t(y_{*,t}))}}$$

Note that if $y_*=\arg \min_y \|x-y\|^2$, then $$y_{*,t} = \arg\min_y \|x-m_ty\|^2= \arg \min_y \left\|x/m_t-y\right\|^2 = \pi_\ell(x/m_t).$$ For all sufficiently small $t$, $x/m_t$ lies in a tubular neighborhood of $M_\ell$ where $\pi_\ell$ is uniquely defined. Hence, there is a unique minimizer of $m_tM_\ell$ given by $m_ty_{*,t}=\pi_{\ell,t}(x):=m_t\pi_\ell(x/m_t)$. The integrals in the expression for $n_t(x)$ only differ in the amplitude $g$, so, after applying Laplace's method to the numerator and denominator and canceling terms,
$$n_t(x) \approx x-m_ty_{*,t} = x-\pi_{\ell,t}(x) \quad \text{ as } t\to 0^+.$$

Since $m_t \to 1$ as $t\to 0^+$ and the projection map is continuous in a tubular neighborhood of $M_\ell$, we have that $m_ty_{*,t} \to y_*$. Thus, as $t \to 0^+$,
$$\frac{\nabla \log p_t(x)}{\|\nabla \log p_t(x)\|} = \frac{-n_t(x)}{\|n_t(x)\|} \to \frac{\pi_\ell(x)-x}{\|\pi_\ell(x)-x\|}.$$

\paragraph{General case.} Now recall that for a stratified space we obtain a mixture of the form
$$p_t(x) = \int_{\m S} \phi_{\sigma_t}(x-m_tx_0)dP_*(x_0) = \sum_{k=1}^K \omega_k\int_{M_k} \phi_{\sigma_t}(x-m_tx_0)dQ_k(x_0),$$
and
$$\nabla p_t(x) = \sum_{k=1}^K \omega_k \int_{M_k} -\frac{x-m_tx_0}{\sigma_t^2}\phi_{\sigma_t}(x-m_tx_0)dQ_k(x_0).$$
The score is then realized as
\begin{align*}
    \nabla \log p_t(x) &= \frac{\nabla p_t(x)}{p_t(x)}= \sum_{k=1}^K \omega_k \frac{\int_{M_k} -(\frac{x-m_tx_0}{\sigma_t^2})\phi_{\sigma_t}(x-m_tx_0)dQ_k(x_0)}{\sum_{\ell=1}^K \omega_\ell\int_{M_\ell}\phi_{\sigma_t}(x-m_tx_0)dQ_\ell(x_0)}\\
    &=-\frac{1}{\sigma_t^2}\left(x-m_t\sum_{k=1}^K \omega_k\frac{\int_{M_k} x_0\phi_{\sigma_t}(x-m_tx_0)dQ_k(x_0)}{\sum_{\ell=1}^K \omega_\ell\int_{M_\ell}\phi_{\sigma_t}(x-m_tx_0)dQ_\ell(x_0)}\right)
\end{align*}

Note that we need only consider the indices in the set $$L(x)=\{k\in \{1,...,K\}: \operatorname{dist}(x,M_k) = \operatorname{dist}(x,\m S)\}.$$ Indeed, if $k \not\in L(x)$, then $\Delta=\text{dist}(x,M_k)^2-\text{dist}(x,\m S)^2>0$. By compactness of $M_k$ and $\m S$ and the fact that $m_t \to 1$ as $t\to 0$, for sufficiently small $t>0$ we have $$\text{dist}(x,m_tM_k)^2 \geq \text{dist}(x,m_t\m S)^2 + \frac{\Delta}{2}.$$
Define $$I_k(t)=\int_{M_k} \phi_{\sigma_t}(x-m_tx_0)dQ_k(x_0),\quad J_k(t)=\int_{M_k}x_0\phi_{\sigma_t}(x-m_tx_0)dQ_k(x_0).$$
Then,
$$I_k(t)\lesssim \sigma_t^{-D}\exp\left( -\frac{\operatorname{dist}(x,m_tM_k)^2}{2\sigma_t^2}\right) \leq \sigma_t^{-D}\exp\left(-\frac{\operatorname{dist}(x,m_t\m S)^2}{2\sigma_t^2}\right)\exp\left(-\frac{\Delta}{4\sigma_t^2}\right).$$
On the other hand, for $\ell \in L(x)$, we have by Laplace's method that, for small $t$,
$$I_\ell(t) \gtrsim \sigma_t^{-(D-d_\ell)}\exp\left(-\frac{\operatorname{dist}(x,m_t\m S)^2}{2\sigma_t^2}\right).$$
Hence $$\frac{I_k}{I_\ell}\lesssim \sigma_t^{-d_\ell}\exp\left(-\frac{\Delta}{4\sigma_t^2}\right) \to 0.$$
This holds similarly for $J_k$ in place of $I_k$ above since $\|x_0\|$ is bounded uniformly by compactness, so it only contributes to the constant.

Applying Laplace's method to every integral for indices in $L(x)$, we have that
\begin{align*}
    n_t(x) &\approx x-m_t\sum_{k \in L(x)} \omega_k\frac{G_k(x,t)y_{k,*,t}}{\sum_{\ell \in L(x)} \omega_\ell G_\ell(x,t)}\\
    &=x-\sum_{k\in L(x)} \gamma_{k,t}\cdot m_ty_{k,*,t}\\
    &=\sum_{k\in L(x)} \gamma_{k,t}(x-m_ty_{k,*,t}),
\end{align*}
where $$\gamma_{k,t}=\frac{\omega_kG_k(x,t)}{\sum_{\ell \in L(x)} \omega_\ell G_\ell(x,t)}.$$
Then observe that as $t\to 0$,
$$\frac{G_k(x,t)}{G_\ell(x,t)} \asymp  \sigma_t^{d_k - d_\ell}\exp\left(-\frac{\operatorname{dist}(x,m_tM_k)^2-\operatorname{dist}(x,m_tM_\ell)^2}{2\sigma_t^2}\right).$$
Note that for $k,\ell \in L(x)$, we have that $\operatorname{dist}(x,M_k)=\operatorname{dist}(x,M_\ell)=\operatorname{dist}(x,\m S)$. Hence, by a first order expansion, $$\operatorname{dist}(x,m_tM_k)^2=\operatorname{dist}(x,\m S)^2 + O(1-m_t),$$
$$\operatorname{dist}(x,m_tM_\ell)^2=\operatorname{dist}(x,\m S)^2 + O(1-m_t);$$
and so $$\operatorname{dist}(x,m_tM_k)^2-\operatorname{dist}(x,m_tM_\ell)^2 = O(1-m_t)=O(\sigma_t^2).$$
This implies that the exponential term in the quotient above is bounded, so that $G_k/G_\ell \to 0$ as $\sigma_t^{d_k-d_\ell} \to 0$, which is the case if $d_k > d_\ell$. Hence as $t\to 0$, assuming that stratum $k_*$ has unique minimal dimension,
$$\gamma_{k_*,t}(x)\to 1,\quad \gamma_{k,t}(x) \to 0\quad \text{for $k \neq k_*$}.$$ 
Similar to the manifold case, letting $t \to 0$ implies that
$$\frac{\nabla \log p_t(x)}{\|\nabla \log p_t(x)\|} \to \frac{\pi_{k_*}(x)-x}{\|\pi_{k_*}(x)-x\|}.$$ $\hfill{\Box}$

\subsection{Proof of Lemma~\ref{lm:spectral_gap}}
Let $Y$ be a $d$-dimensional manifold and $x_0 \in Y$. Diffused points are of the form $x_t = m_tx_0 + \sigma_tz$, $z \sim \mathcal{N}(0,I_D)$. \cite{liu2025improving} expand the score function of a VE-SDE into a normal part with bounded remainder for points off the manifold within a small tubular neighborhood. In our VP setting,
$$p_t(x)=\int_Y \phi_{\sigma_t}(x-m_ty)p_*(y)d\operatorname{vol}_Y(y)=\int_{m_tY} \phi_{\sigma_t}(x-z)m_t^{-d}p_*(z/m_t)d\operatorname{vol}_{m_tY}(z),$$
where the last equality follows from reparameterization. Hence, $p_t$ is a Gaussian smoothing of the scaled manifold $m_tY$ and the VP-SDE analogue of Theorem 3.1 in \cite{liu2025improving} is:
$$\nabla \log p_t(x)=-\frac{x-\pi_t(x)}{\sigma_t^2} +O(1),$$
where $\pi_t(x)= m_t\pi(x/m_t)$ is the projection onto $m_tY$ and $x \notin m_tY$ is in a small tubular neighborhood of $m_tY$. For $x_t=m_tx_0+\sigma_tz$ in a tubular neighborhood of $m_tY$, we have that $$x_t-\pi_t(x_t)= (m_tx_0+\sigma_t z)-m_t\pi(x_0+(\sigma_t/m_t)z).$$
By applying a Taylor expansion to $\pi$, we have that
$$\pi(x_0+(\sigma_t/m_t)z)=x_0 + (\sigma_t/m_t)P_Tz + O(\sigma_t^2\|z\|^2),$$
where $P_T$ is projection onto the tangent space of of $Y$ at $x_0$. Substitution yields
$$x_t-\pi_t(x_t)= \sigma_t z - \sigma_tP_Tz + O(\sigma_t^2\|z\|^2)=\sigma_tP_Nz + O(\sigma_t^2\|z\|^2),$$
where $P_N=I-P_T$ is the projection onto the normal space of $Y$ at $x_0$.
Hence, we deduce that
$$\nabla \log p_t(x_t)=-\frac{1}{\sigma_t}P_Nz + R_t, \quad \text{where}\quad \sup_{t \in [t_{\operatorname{start}},t_{\operatorname{end}}]} \mb{E}[\|R_t\|^2 \mid t]\leq C$$
for $x_t \not\in m_tY$ in a sufficiently small tubular neighborhood of $m_tY$. Then
$$\mb{E}[\nabla \log p_t(x)\nabla \log p_t(x)^\top \mid t] = \frac{1}{\sigma_t^2}P_N-\frac{1}{\sigma_t}\mb{E}[P_NzR_t^\top + R_tz^\top P_N \mid t] + \mb{E}[R_t R_t^\top \mid t].$$
By Cauchy-Schwarz, 
$$\|\mb{E}[P_NzR_t^\top \mid t]\|_{\operatorname{op}} \leq \mb{E}[\|P_Nz\|^2]^{1/2}\mb{E}[\|R_t\|^2 \mid t]^{1/2}\leq \mb{E}[\|P_Nz\|^2]^{1/2}\sqrt{C} \leq \sqrt{D-d}\sqrt{C}=O(1),$$
where we used the fact that $\mb{E}[\|P_Nz\|^2]=\operatorname{trace}(P_N)=D-d$. We also have that
$$\|\mb{E}[R_t R_t^\top \mid t]\|_{\operatorname{op}} \leq \mb{E}[\|R_t\|^2|t]\leq C.$$

Thus,
$$\mb{E}[\nabla \log p_t(X) \nabla \log p_t(X)^\top|t]=\frac{1}{\sigma_t^2}P_N + O(\sigma_t^{-1}).$$
Now if we average over $t \in [t_{\text{start}},t_{\text{end}}]$, we have that
$$\Sigma(x_0)=\mb{E}\left[\frac{1}{\sigma_t^2}\right]P_N + O(\mb{E}[\sigma_t^{-1}]).$$
Since $Y$ is $d$-dimensional, the eigenvalues of $P_N$ consist of $D-d$ ones and $d$ zeros. Arranging them in descending order implies that the eigenvalues $\mu_k$ of $\Sigma(x_0)$ satisfy:
$$\mu_k = \mathbb{E}[1/\sigma_t^2] + O(\mb{E}[\sigma_t^{-1}]), \quad k \leq D-d;$$
$$\mu_k = O(\mb{E}[\sigma_t^{-1}]), \quad k >D-d.$$

Let $g=\mb{E}[1/\sigma_t^2]$ and $h=\mb{E}[\sigma_t^{-1}]$. Then we can write the eigenvalues as
$$\mu_k=\begin{cases}
    g+O(h),& k\leq D-d;\\
    O(h),& k > D-d.
\end{cases}$$
For $k < D-d$,
$$\mu_k - \mu_{k+1}=(g+O(h))-(g+O(h))=O(h).$$
For $k>D-d$, we have
$$\mu_k-\mu_{k+1}=O(h)$$
Lastly, for $k=D-d$,
$$\mu_{D-d}-\mu_{D-d+1}=(g+O(h))-O(h)= \Theta(g)$$
because $h=o(g)$ by Jensen's inequality. Since $g=\mb{E}[1/\sigma_t^2]\geq 1/\sigma_{t_{\text{end}}}^2\to \infty$ as $t_{\text{end}} \to 0$, the gap between indices $D-d$ and $D-d+1$ will dominate all other gaps for sufficiently small $ t_{\text{end}}$.

$\hfill{\Box}$

\subsection{Proof of Theorem~\ref{th:consistency}}
Let $x_0 \in M_\ell \subseteq S \setminus S_{\text{sing}}$, $\delta = d(x_0,S_{\text{sing}})>0$. Note that $t_2^{(n)} \to 0$ implies that $\sigma_{t_2^{(n)}}\lesssim \delta/2$ for sufficiently large $n$, making other distant strata have negligible effect on the score, allowing us to localize on $M_\ell$ and apply Lemma~\ref{lm:spectral_gap} (see proof of Lemma~\ref{th:normalalign} for justification). Let $n \in \mathbb{N}$ and suppose that 
$$\mb{E}_{X_{1:n}\sim P_{*}^{\otimes n}}\bigg[ \frac{1}{\Delta t_n}  \int_{t_{1}^{(n)}}^{t^{(n)}_{2}} \int_{\mb R^{D}}   \|\wh S_n(x,t)-\nabla \log p_t(x)\|^2 p_t(x)\,\dd x\dd t\bigg] \leq \epsilon_n,$$
where $\Delta t_n = t_2^{(n)}-t_1^{(n)}$. Note that $\wh{S}_n$ is trained on an independent sample $X_{1:n} \sim P_*^{\otimes n}$, while the test points $\{(x_i,t_i)\}_{i=1}^{N_n}$ are drawn independently from the training data.

Let $R_n\geq 1$. For any $t>0$, consider the sets $B_t=\{x:||x-m_tx_0||\leq R_n\sigma_t\}$ and $U_t = B(x_0,\frac{1}{R_n}\sigma_t)\cap M_\ell$. We get a lower bound on the marginal distribution $p_t$ via concentration on $U_t$:
$$p_t(x) = \int_{\mathbb{R}^D} p_t(x|x')P_*(dx') \geq \int_{U_t} p_t(x|x')P_*(dx') $$
Recall that $p_t(x|y)=\m N(m_ty,\sigma_t^2I_D)$. Hence, for $x \in B_t$ and $x' \in U_t$, 
\begin{align*}
    \frac{p_t(x|x_0)}{p_t(x|x')} &=\exp\Big(\frac{\|x-m_tx'\|^2-\|x-m_tx_0\|^2}{2\sigma_t^2}\Big)\\
    &= \exp \Big(\frac{m_t^2\|x'-x_0\|^2 - 2m_t\langle x-m_tx_0,x'-x_0 \rangle}{2\sigma_t^2}\Big)\\
    & \leq \exp\Big( \frac{m_t^2R_n^{-2}+2m_tR_nR_n^{-1}}{2}\Big)\\
    & \leq \exp \left(\frac{R_n^{-2}+2}{2}\right)
\end{align*}
This implies that $$p_t(x|x_0) \leq c'p_t(x|x')$$
for some constant $c'\leq e^{3/2}$ since $R_n \geq 1$. Then, since $P_*$ is bounded from below on $M_\ell$, we have that $P_*(U_t) \gtrsim R_n^{-d_\ell}\sigma_t^{d_\ell}$. Putting everything together and combining constants into $\tilde{C}$ with the fact that $\sigma_t$ is monotonic in $t$ gives that over the interval $[t_1^{(n)},t_2^{(n)}]$,
$$p_t(x|x_0) \leq \tilde{C}R_n^{d_\ell}\sigma_t^{-d_\ell}p_t(x) \leq \tilde{C}R_n^{d_\ell}\sigma_{t^{(n)}_{1}}^{-d_\ell}p_t(x)$$
On $B_t$, we can use the upper bound on score approximation to get 
$$\mb{E}_{X_{1:n}\sim P_{*}^{\otimes n}}\Big[\frac{1}{\Delta t_n}\int_{t_{1}^{(n)}}^{t_{2}^{(n)}}\int_{B_t} ||\hat{S}(x,t)-\nabla \log p_t(x)||^2p_t(x|x_0)dxdt\Big] \leq \tilde{C}R_n^{d_\ell}\sigma_{t_1^{(n)}}^{-d_\ell}\epsilon_n.$$

\noindent For a bound on $B_t^c$ we use the following lemma.
\begin{lemma}\label{lm:tail_bound}
    Suppose $\|\wh{S}_n\|_\infty \leq V_n$ and $\operatorname{supp } P_*\subseteq B(0,L)$. Then for $X \sim p_t(\cdot|x_0)$,
    $$\mathbb{E}[\|\hat{S}_n(x,t)-\nabla \log p_t(x)\|^2\mathbf{1}_{B_t^c}] \lesssim \Big(V_n^2 + \frac{L^2}{\sigma_t^4} + \frac{R^2+D}{\sigma_t^2}\Big)e^{-(R-\sqrt{D})^2/2} $$
\end{lemma}

Now let $T_n(R_n)$ denote the expectation bound from the lemma with $R=R_n$. Then
$$\mb{E}\left[ \frac{1}{\Delta t_n} \int_{t_{1}^{(n)}}^{t_{2}^{(n)}}\int_{\mb{R}^D} \|\hat{S}(x,t)-\nabla \log p_t(x)\|^2p_t(x|x_0)dxdt\right] \leq \tilde{C}R_n^{d_\ell}\sigma_{t_{1}^{(n)}}^{-d_\ell}\epsilon_n + T_n(R_n).$$
Let $E_n$ be the term inside the expectation and let $\rho_n=\tilde{C}R_n^{d_\ell}\sigma_{t_{1}^{(n)}}^{-d_\ell}\epsilon_n + T_n(R_n)$. Choose $R_n \to \infty$ such that $T_n(R_n) \leq \eta_n \to 0$ for some decreasing sequence $\eta_n$. Note that this also affects the size of $B_t$, but one can take $R_n$ to grow slowly, such as $R_n = O(\sqrt{\log n})$ assuming mild conditions such as polynomial growth of $V_n$, $\sigma_{t_1^{(n)}}^{-1}$ and $1/\eta_n$. This means that $R_n^{d_\ell}=O((\log n)^{d_\ell/2})=o(n^\gamma)$ for any $\gamma >0$. We then have that $T_n(R_n)\lesssim \eta_n \to 0$, and by our assumption that $n^\eta\sigma_{t_{1}^{(n)}}^{-d_\ell}\epsilon_n \to 0$, we have $\rho_n\to 0$ as $n \to \infty$ by choosing any $\gamma<\eta$.  \\

The term $E_n$ depends on $X_{1:n}$, so we define the event $\m A_n=\{E_n(X_{1:n}) \leq \rho_n^{1/2}\}$. Then by Markov's inequality, $\mb{P}(\m A_n^c) \leq \rho_n^{1/2}\to 0$ as $n \to \infty$. Let $\Sigma_n=\Sigma(x_0)$ as in Lemma~\ref{lm:spectral_gap}, i.e. $$\Sigma_n=\mb{E}_{t \sim \operatorname{Uniform}[t_1^{(n)},t_2^{(n)}], X \sim p_t(\cdot|x_0)}[\nabla \log p_t(X)\nabla \log p_t(X)^\top]$$
and define
$$\hat{\Sigma}_n=\frac{1}{N_n}\sum_{i=1}^{N_n} \hat{S}_n(x_i,t_i)\hat{S}_n(x_i,t_i)^\top,$$
Letting $s_i=\nabla \log p_{t_i}(x_i)$ and $\hat{S}_i=\hat{S}_n(x_i,t_i)$, we have $$\hat{\Sigma}_n-\Sigma_n=\underbrace{\frac{1}{N_n}\sum_{i=1}^{N_n}\left[ \hat{S}_n(x_i,t_i)\hat{S}_n(x_i,t_i)^\top-s_is_i^\top\right]}_{(A)} + \underbrace{\frac{1}{N_n}\sum_{i=1}^{N_n}s_is_i^\top - \Sigma_n}_{(B)}.$$

\textbf{Bounding $(A)$:} Using the relation $\hat{S}_i\hat{S}_i^\top-s_is_i^\top=(\hat{S}_i-s_i)(\hat{S}_i-s_i)^\top + (\hat{S}_i-s_i)s_i^\top + s_i(\hat{S}_i-s_i)^\top$ implies that
$$\|\hat{S}_i\hat{S}_i^\top-s_is_i^\top\|_{\operatorname{op}}\leq \|\hat{S}_i-s_i\|^2 + 2\|\hat{S}_i-s_i\|\|s_i\|.$$
Taking expectation with respect to $t\sim \operatorname{Uniform}[t_1^{(n)},t_2^{(n)}]$ and $x_i \sim p_t(\cdot|x_0)$ and applying Cauchy-Schwarz,
$$\mb{E}\left[\left\|(A)\right\|_{\operatorname{op}}| X_{1:n}\right]\leq \mb{E}[\|\hat{S}_i-s_i\|^2|X_{1:n}] + 2\mb{E}[\|\hat{S}_i-s_i\|^2| X_{1:n}]^{1/2}\mb{E}[\|s_i\|^2]^{1/2}.$$
We have that $\mb{E}[\|\hat{S}_i-s_i\|^2\mid X_{1:n}] \leq \rho_n^{1/2}$ on $\m A_n$. For the second term, note that, by Jensen's inequality,
$$\mb{E}[\|s_i\|^2] \leq \mb{E}[\|s_i\|^4]^{1/2}.$$

\begin{lemma}\label{lm:fourth_moment}
    Let $x_0 \in M_\ell$ satisfy $\operatorname{dist}(x_0,\m S_{\operatorname{sing}})>\delta$. Then for sufficiently large $n$, $$\mb{E}[\|\nabla \log p_T(X)\|^4]\lesssim \sigma_{t_1^{(n)}}^{-4},$$
    where $T \sim \operatorname{Uniform}[t_1^{(n)},t_2^{(n)}]$ and $X|T=t\sim p_t(\cdot \mid x_0)$.
\end{lemma}

By Lemma~\ref{lm:fourth_moment}, we hence have that $\mb{E}[\|s_i\|^2]^{1/2} \lesssim \sigma_{t_1^{(n)}}^{-1}$. Thus
$$2\mb{E}[\|\hat{S}_i-s_i\|^2\mid X_{1:n}]^{1/2}\mb{E}[\|s_i\|^2]^{1/2} \leq 2\left(\rho_n^{1/2}\right)^{1/2}\cdot O(\sigma_{t_1^{(n)}}^{-1})=O\left( \rho_n^{1/4}\sigma_{t_1^{(n)}}^{-1}\right).$$
Let $\zeta_n$ be the upper bound on the total expectation above and note that since $\rho_n \to 0$, we have that $\zeta_n=O(\rho_n^{1/2}+\rho_n^{1/4}\sigma_{t_1^{(n)}}^{-1})=o(\sigma_{t_1^{(n)}}^{-2})$. For any fixed $c >0$, we have by Markov that
$$\mb{P}\left(\left\|(A)\right\|_{\operatorname{op}} >c\sigma_{t_1^{(n)}}^{-2} \mid X_{1:n}\right) \leq\frac{\zeta_n}{c\sigma_{t_1^{(n)}}^{-2}} \to 0.$$

\textbf{Bounding $(B)$:}
Note that $\|(B)\|_{\operatorname{op}}\leq \|(B)\|_F$ and denote $Y_{n,i}=s_is_i^\top-\Sigma_n$. Then, conditioned on the sample $X_{1:n}$ used to train $\hat{S}_n$, $Y_{n,i}$ are i.i.d., and hence
$$\mb{E}[\|(B)\|_{\operatorname{op}}^2\mid X_{1:n}] \leq \mb{E}[\|(B)\|_F^2\mid X_{1:n}]=\frac{1}{N_n}\mb{E}[\|Y_{n,1}\|_F^2]\leq \frac{1}{N_n}\mb{E}[\|s_i\|^4].$$

From Lemma~\ref{lm:fourth_moment}, we have that
$\mb E[\|s_i\|^4] \lesssim \sigma_{t_1^{(n)}}^{-4}$, so $$\mb E[\|(B)\|^2_{\operatorname{op}}\mid X_{1:n}] \lesssim \frac{1}{N_n}\sigma_{t_1^{(n)}}^{-4}.$$
Hence, by Chebyshev,
$$\mb P(\|(B)\|_{\operatorname{op}}>C\sigma_{t_1^{(n)}}^{-2}) \lesssim \frac{1}{C^2N_n} \to 0$$
as $N_n \to \infty$.

Let $\hat{\mu}_1 \geq \cdots \geq \hat{\mu}_D$ be the eigenvalues of $\hat{\Sigma}_n$ and $\mu_1 \geq \cdots \geq \mu_D$ be the eigenvalues of $\Sigma_n$. By Weyl's inequality,
$$|\hat{\mu}_k-\mu_k|\leq \|\hat{\Sigma}_n-\Sigma_n\|_{\operatorname{op}}.$$
By Lemma~\ref{lm:spectral_gap}, for sufficiently small $t_2^{(n)}$, the largest spectral gap $\gamma_n \asymp \sigma_{t_1^{(n)}}^{-2}$ of $\Sigma_n$ occurs between indices $D-d_\ell$ and $D-d_\ell +1$, and the rest of the consecutive spectral gaps are of order $\mb{E}[\sigma_t^{-1}]$.

On the event $\{\|\hat{\Sigma}_n-\Sigma_n\|_{\operatorname{op}} \leq \gamma_n/8\}$, we thus have for sufficiently large $n$,
$$\hat{\mu}_k-\hat{\mu}_{k+1}\geq \mu_k-\mu_{k+1}-2\|\hat{\Sigma}_n-\Sigma_n\|_{\operatorname{op}},$$
so $$\hat{\mu}_{D-d_\ell}-\hat{\mu}_{D-d_\ell+1} \geq \gamma_n-2\|\hat{\Sigma}_n-\Sigma_n\|_{\operatorname{op}} >3\gamma_n/4.$$
For any other $k \neq D-d_\ell$,
$$\hat{\mu}_k-\hat{\mu}_{k+1}\leq \mu_k-\mu_{k+1}+2\|\hat{\Sigma}_n-\Sigma_n\|_{\operatorname{op}}< \sigma_{t_1^{(n)}}^{-1}+2\|\hat{\Sigma}_n-\Sigma_n\|_{\operatorname{op}}<\gamma_n/2$$
for sufficiently large $n$ as $\gamma_n \asymp \sigma_t^{-2} \to \infty$, which means that $\hat{d}_n(x_0)=d_\ell$.
Now, we have that
$$\mb{P}(\hat{d}_n(x_0) \neq d_\ell) \leq \mb{P}(\m A_n^c)+\mb{P}(\{\|(A)\|_{\operatorname{op}} > \gamma_n/8\} \cap \m A_n) +\mb{P}(\|(B)\|_{\operatorname{op}}>\gamma_n/8)\to 0\quad  \text{as } n \to \infty.$$ $\hfill{\Box}$

\subsection{Proof of Lemma~\ref{lm:tail_bound}}
\textit{Proof:} We first note that 
$$\mathbb{E}[\|\hat{S}(x,t)-\nabla \log p_t(x)\|^2\mathbf{1}_{B_t^c}] \leq 2\mathbb{E}[\|\hat{S}(\cdot,t)\|_\infty^2\mathbf{1}_{B_t^c}] + 2\mathbb{E}[\|\nabla \log p_t(X)\|^2\mathbf{1}_{B_t^c}],$$
where $X \sim \mathcal{N}(m_tx_0,\sigma_t^2I_D)$. Now note that $x = m_tx_0 + \sigma_tz$ for some $z \sim \mathcal{N}(0,I_D)$, and, by definition of our network class, $\|\hat{S}\|_\infty \leq V_n$. For $R \gtrsim \sqrt{D}$, first term on the right side of the inequality can thus be bounded as:
$$2\mathbb{E}[\|\hat{S}(\cdot,t)\|_\infty^2\mathbf{1}_{B_t^c}] \lesssim 2V_n^2 e^{-(R-\sqrt{D})^2/2}.$$
As for the second term, we have by Tweedie's formula that
$$\nabla \log p_t(x) = -\frac{1}{\sigma_t^2}\Big(x-m_t\mathbb{E}[x_0|X_t=x]\Big).$$
Thus $$\|\nabla \log p_t(x)\| \leq \frac{1}{\sigma_t^2}(\|x\| + m_t\|\mathbb{E}[x_0|X_t=x]\|).$$
Since the support of the true distribution $P_*$ is compact, we have that $\|\mathbb{E}[x_0|X_t=x]\|<L$ for some constant $L>0$. We also have that
$$\|x\| = \|m_tx_0+\sigma_tz\| \leq \|m_tx_0\| + \sigma_t\|z\|,$$
which implies that on $B_t^c$ (where $\|z\| > R$):
\begin{align*}
    \|\nabla \log p_t(x) \|^2\mathbf{1}_{\{\|z\|>R\}} &\leq \frac{1}{\sigma_t^4}\Big(2\|x\|^2 + 2m_t^2\|\mathbb{E}[x_0|X_t=x]\|^2\Big)\\
    &\leq\frac{1}{\sigma_t^4}(4\|m_tx_0\|^2+4\sigma_t^2\|z\|^2 + 2m_t^2L^2)\\
    &= \frac{2}{\sigma_t^4}(3m_t^2L^2 + 2\sigma_t^2\|z\|^2)
\end{align*}
where the last inequality holds since $x_0 \in \operatorname{supp} P_*$. Now by taking expectation, we get
\begin{align*}
    \mathbb{E}[\|\nabla \log p_t(x)\|^2\mathbf{1}_{B_t^c}] &\lesssim \frac{L^2}{\sigma_t^4}\mathbb{P}(\|z\|>R) + \frac{1}{\sigma_t^2}\mathbb{E}[\|z\|^2\mathbf{1}_{\{\|z\|>R\}}]\\
    & \lesssim \frac{L^2}{\sigma_t^4}e^{-(R-\sqrt{D})^2/2} + \frac{1}{\sigma_t^2}(R^2+D)e^{-(R-\sqrt{D})^2/2}
\end{align*}
Combining the two bounds together gives us the lemma. $\hfill{\Box}$\\

\subsection{Proof of Lemma~\ref{lm:fourth_moment}}
Recall that a diffusion point from $x_0$ takes the form $X_t=m_tx_0+\sigma_tZ$, where $Z \sim \m N(0,I_D)$. From the proof of Lemma~\ref{lm:spectral_gap}, we see that for a small tubular region around $M_\ell$, $$\nabla \log p_t(X_t) = -\frac{X_t-\pi_t(X_t)}{\sigma_t^2} + R_t(X_t),$$
where $\|R_t(X_t)\|\leq C$ uniformly in $t \in [t_1^{(n)},t_2^{(n)}]$. Let $\m A_t$ be the event that this decomposition holds. Then $\m A_t^c \subseteq \{\|Z\|>r/\sigma_t\}$ for some radius $r>0$. On the event $\m A_t$, since $\pi_t(x)$ is the minimizer of the distance to $m_tM_\ell$, we have that
$$\|X_t-\pi_t(X_t)\|\leq \|X_t-m_tx_0\|\leq \sigma_t\|Z\|,$$
and so
$$\|\nabla \log p_t(X_t)\| \leq \frac{\|Z\|}{\sigma_t} + \|R_t(X_t)\|.$$
From the inequality $(a+b)^4\leq 8(a^4+b^4)$, we have that
$$\|\nabla \log p_t(X_t)\|^4 \lesssim \frac{\|Z\|^4}{\sigma_t^4}+1$$
Then
$$\mb{E}[\|\nabla \log p_t(X_t)\|^4\mathbf{1}_{\m A_t}\mid t] \lesssim \sigma_t^{-4}\mb{E}[\|Z\|^4]+1 \lesssim \sigma_{t^{(n)}_1}^{-4},$$
where we used the fact that $\sigma_t$ is monotonically increasing and $\|Z\|^2 \sim \chi_D^2$, and we have for a $\chi_D^2$ random variable that $\mb{E}[(\|Z\|^2)^2]=D^2+2D$.
On the event $\m A_t^c$, we note that the manifold $M_\ell$ is compact, so that $\|M_\ell\|\leq L <\infty$, and
$$\nabla \log p_t(x)=\frac{m_t\mb{E}[X_0|X_t=x]-x}{\sigma_t^2}.$$
Thus,
$$\|\nabla \log p_t(x)\|\leq \frac{m_tL + \|x\|}{\sigma_t^2}.$$
Plugging in $X_t=m_tx_0+\sigma_tZ$ implies that
$$\|\nabla \log p_t(X_t)\|\leq \frac{L}{\sigma_t^2}+\frac{\|m_tx_0 + \sigma_tZ\|}{\sigma_t^2} \lesssim \frac{1}{\sigma_t^2} + \frac{\|Z\|}{\sigma_t},$$
and so
$$\|\nabla \log p_t(X_t)\|^4 \lesssim \sigma_t^{-8}+\sigma_t^{-4}\|Z\|^4.$$
Now, by Gaussian tails, we have that $\mb{P}(\m A_t^c\mid t)\lesssim e^{-a/\sigma_t^2}$ for some constant $a>0$, and $\mb{E}[\|Z\|^4\mathbf{1}_{\m A_t^c}\mid t]$ similarly decays like $e^{-a/\sigma_t^2}$ up to a polynomial factor. Hence
$$\mb{E}[\|\nabla \log p_t(X_t)\|^4\mathbf{1}_{\m A_t^c}\mid t] \lesssim \sigma_t^{-8}e^{-a/\sigma_t^2} + \sigma_t^{-4}\mb{E}[\|Z\|^4\mathbf{1}_{\m A_t^c}]=o(\sigma_t^{-4}).$$ 
Putting everything together gives us the claim. $\hfill{\Box}$

\end{document}